\documentclass{article}

\PassOptionsToPackage{semicolon}{natbib}

\usepackage[preprint]{neurips_2020}

\usepackage[utf8]{inputenc} 
\usepackage[T1]{fontenc}    
\usepackage{hyperref}       
\usepackage{url}            
\usepackage{booktabs}       
\usepackage{amsfonts}       
\usepackage{nicefrac}       
\usepackage{microtype}      
\usepackage[ruled,lined,boxed]{algorithm2e}

\usepackage{graphicx}
\usepackage{caption}
\usepackage{float}
\usepackage{subcaption}

\usepackage{amsmath,amsthm,amssymb}
\usepackage{xifthen}
\usepackage{multirow}
\usepackage{comment}
\usepackage{wrapfig}

\allowdisplaybreaks

\makeatletter
\newcommand\footnoteref[1]{\protected@xdef\@thefnmark{\ref{#1}}\@footnotemark}
\makeatother

\usepackage{array}
\newcolumntype{L}[1]{>{\raggedright\let\newline\\\arraybackslash\hspace{0pt}}m{#1}}
\newcolumntype{C}[1]{>{\centering\let\newline\\\arraybackslash\hspace{0pt}}m{#1}}
\newcolumntype{R}[1]{>{\raggedleft\let\newline\\\arraybackslash\hspace{0pt}}m{#1}}

\mathchardef\mhyphen="2D

\newcommand{\Pb}{\mathbf{P}}
\newcommand{\Eb}{\mathbf{E}}

\def\tRhelp#1#2\relax{L_{\csname dom#1\endcsname#2}} 

\def\eRhelp#1#2\relax{\hat{L}_{\csname set#1\endcsname#2}}
\newcommand{\loss}[1]{l\ifthenelse{\isempty{#1}{}}{}{\left(#1\right)}}

\DeclareMathOperator*{\argmax}{arg\,max}
\DeclareMathOperator*{\expect}{\mathbb{E}}

\DeclareMathOperator*{\var}{Var}

\newcommand{\mydefv}[1]{\expandafter\newcommand\csname v#1\endcsname{\mathbf{#1}}}
\newcommand{\mydefallv}[1]{\ifx#1\mydefallv\else\mydefv{#1}\expandafter\mydefallv\fi}
\mydefallv abekmsuvwxyz\mydefallv 

\newcommand{\mydefvsym}[1]{\expandafter\newcommand\csname v#1\endcsname{\boldsymbol{\csname #1\endcsname}}}
\newcommand{\mydefallvsym}[1]{\ifx#1\mydefallvsym\else\mydefvsym{#1}\expandafter\mydefallvsym\fi}
\mydefallvsym {sigma}{alpha}{gamma}{mu}\mydefallvsym 

\newcommand{\mydefm}[1]{\expandafter\newcommand\csname m#1\endcsname{\mathbf{#1}}}
\newcommand{\mydefallm}[1]{\ifx#1\mydefallm\else\mydefm{#1}\expandafter\mydefallm\fi}
\mydefallm ABCIKLMOSTVWXZ\mydefallm 

\newcommand{\mydefmsym}[1]{\expandafter\newcommand\csname m#1\endcsname{\boldsymbol{\csname #1\endcsname}}}
\newcommand{\mydefallmsym}[1]{\ifx#1\mydefallmsym\else\mydefmsym{#1}\expandafter\mydefallmsym\fi}
\mydefallmsym {Sigma}{Gamma}{Phi}{gamma}\mydefallmsym 

\newcommand{\mydefalg}[1]{\expandafter\newcommand\csname alg#1\endcsname{\mathcal{#1}}}
\newcommand{\mydefallalg}[1]{\ifx#1\mydefallalg\else\mydefalg{#1}\expandafter\mydefallalg\fi}
\mydefallalg A\mydefallalg 

\newcommand{\mydefdom}[1]{\expandafter\newcommand\csname dom#1\endcsname{\mathcal{#1}}}
\newcommand{\mydefalldom}[1]{\ifx#1\mydefalldom\else\mydefdom{#1}\expandafter\mydefalldom\fi}
\mydefalldom HSTVX\mydefalldom 

\newcommand{\mydefset}[1]{\expandafter\newcommand\csname set#1\endcsname{\mathcal{#1}}}
\newcommand{\mydefallset}[1]{\ifx#1\mydefallset\else\mydefset{#1}\expandafter\mydefallset\fi}
\mydefallset BCGHPQSTVX\mydefallset 

\newcommand{\mydefdistr}[1]{\expandafter\newcommand\csname distr#1\endcsname{\mathcal{D}_{\csname dom#1\endcsname}}}
\newcommand{\mydefalldistr}[1]{\ifx#1\mydefalldistr\else\mydefdistr{#1}\expandafter\mydefalldistr\fi}
\mydefalldistr DHSTVX\mydefalldistr 

\newcommand{\mydefspace}[1]{\expandafter\newcommand\csname space#1\endcsname{\mathcal{#1}}}
\newcommand{\mydefallspace}[1]{\ifx#1\mydefallspace\else\mydefspace{#1}\expandafter\mydefallspace\fi}
\mydefallspace DFGHKLMPRUVXYZ\mydefallspace 

\newcommand{\mydeff}[1]{\expandafter\newcommand\csname f#1\endcsname[2][]{#1##1\ifthenelse{\equal{##2}{}}{}{\!\left(##2\right)}}}
\newcommand{\mydefallf}[1]{\ifx#1\mydefallf\else\mydeff{#1}\expandafter\mydefallf\fi}
\mydefallf cdfghkwCFGHMRT{PV}\mydefallf 

\newcommand{\mydeffsym}[1]{\expandafter\newcommand\csname f#1\endcsname[2][]{\csname #1\endcsname##1\ifthenelse{\equal{##2}{}}{}{\!\left(##2\right)}}}
\newcommand{\mydefallfsym}[1]{\ifx#1\mydefallfsym\else\mydeffsym{#1}\expandafter\mydefallfsym\fi}
\mydefallfsym {phi}{epsilon}{eta}\mydefallfsym 

\newcommand{\mydefnset}[1]{\expandafter\newcommand\csname nset#1\endcsname{\mathbb{#1}}}
\newcommand{\mydefallnset}[1]{\ifx#1\mydefallnset\else\mydefnset{#1}\expandafter\mydefallnset\fi}
\mydefallnset CNRSZ\mydefallnset 

\newcommand{\normTwo}[1]{\left\|#1\right\|_2}

\newcommand{\norm}[1]{\left\|#1\right\|}

\newcommand{\bigO}[1]{\mathcal{O}\left( #1 \right)}
\newcommand{\bigOmega}[1]{\Omega\left( #1 \right)}

\newtheorem{myth}{Theorem}
\newtheorem*{myth*}{Theorem}
\newtheorem{mylem}{Lemma}
\newtheorem*{mylem*}{Lemma}

\newtheorem{myprop}{Proposition}

\theoremstyle{definition}
\newtheorem{myrem}{Remark}
\newtheorem*{myrem*}{Remark}

\newtheorem*{myex*}{Example}

\newcommand{\uniform}{\text{Unif}}
\newcommand{\betadist}{\text{Beta}}

\newcommand{\ind}[1]{\mathbb{I}_{\left\{ #1 \right\}}}
\newcommand{\tr}[1]{\textup{trace}\left(#1\right)}
\newcommand{\1}{\mathbf{1}}
\newcommand{\St}{\widetilde{S}}
\newcommand{\jj}{{\setC_j\setC_j}}
\newcommand{\jl}{{\setC_j\setC_\ell}}
\renewcommand{\j}{{\setC_j}}
\renewcommand{\l}{{\setC_\ell}}
\newcommand{\nj}{{n_j}}
\renewcommand{\nl}{{n_\ell}}
\newcommand{\setW}{\mathcal{W}}
\renewcommand{\var}[1]{\textup{Var}\left(#1\right)}

\title{Near-Optimal Comparison Based Clustering} 

\author{%
  Micha{\"e}l Perrot\\
  Univ Lyon, UJM-Saint-Etienne, CNRS, IOGS,\\
  LabHC UMR 5516, F-42023, SAINT-ETIENNE, France\\
  \texttt{michael.perrot@univ-st-etienne.fr} \\
  \And
  Pascal Mattia Esser, \hskip3ex Debarghya Ghoshdastidar \\
  Department of Informatics \\
  Technical University of Munich \\
  \texttt{\{esser,ghoshdas\}@in.tum.de} \\
}

\begin{document}

\maketitle

\begin{abstract}
The goal of clustering is to group similar objects into meaningful partitions. This process is well understood when an explicit similarity measure between the objects is given. However, far less is known when this information is not readily available and, instead, one only observes ordinal comparisons such as \emph{``object i is more similar to j than to k.''} In this paper, we tackle this problem using a two-step procedure: we estimate a pairwise similarity matrix from the comparisons before using a clustering method based on semi-definite programming (SDP). We theoretically show that our approach can exactly recover a planted clustering using a near-optimal number of passive comparisons. We empirically validate our theoretical findings and demonstrate the good behaviour of our method on real data.
\end{abstract}

\section{Introduction}

In clustering, the objective is to group together objects that share the same semantic meaning, that are similar to each other, into $k$ disjoint partitions. This problem has been extensively studied in the literature when a measure of similarity between the objects is readily available, for example when the examples have a Euclidean representation or a graph structure \citep{shi2000normalized,arthur2007kmeans,vonluxburg2007tutorial}. However, it has attracted less attention when the objects are difficult to represent in a standard way, for example cars or food. A recent trend to tackle this problem is to use comparison based learning \citep{ukkonen2017crowdsourced,emamjomehzadeh2018adaptive} where, instead of similarities, one only observes comparisons between the examples:
\\ \textbf{Triplet comparison:} Object $x_i$ is  more similar to object $x_j$ than to object $x_k$;
\\ \textbf{Quadruplet comparison:} Objects $x_i$ and $x_j$ are more similar to each other than objects $x_k$ and $x_l$.
\\ There are two ways to obtain these comparisons. On the one hand, one can adaptively query them from an oracle, for example a crowd. This is the active setting. On the other hand, they can be directly given, with no way to make new queries. This is the passive setting. In this paper, we study comparison based learning for clustering using passively obtained triplets and quadruplets.

Comparison based learning mainly stems from the psychometric and crowdsourcing literature \citep{shepard1962theanalysis,young1987multidimensional,stewart2005absolute} where the importance and robustness of collecting ordinal information from human subjects has been widely discussed. In recent years, this framework has attracted an increasing amount of attention in the machine learning community and three main learning paradigms have emerged. The first one consists in obtaining an Euclidean embedding of the data that respects the comparisons as much as possible and then applying standard learning techniques \citep{borg2005modern,agarwal2007generalized,jamieson2011low,tamuz2011adaptively,van2012stochastic,terada2014local,zhang2015jointly,amid2015multiview,arias2017some}. The second paradigm is to directly solve a specific task from the ordinal comparisons, such as data dimension or density estimation \citep{kleindessner2015dimensionality,ukkonen2015crowdsourced}, 
classification and regression \citep{haghiri2018comparison}, or
clustering \citep{vikram2016interactive,ukkonen2017crowdsourced,ghoshdastidar2019foundations}. Finally, the third paradigm is an intermediate solution where the idea is to learn a similarity or distance function, as in embedding approaches, but, instead of satisfying the comparisons, the objective is to solve one or several standard problems such as classification or clustering \citep{kleindessner2017kernel}. In this paper, we focus on this third paradigm and propose two new similarities based on triplet and quadruplet comparisons respectively. While these new similarities can be used to solve any machine learning problem, we show that they are provably good for clustering under a well known planted partitioning framework  \citep{abbe2017community,yan2018provable,xu2020optimal}.

\textbf{Motivation of this work.}
A key bottleneck in comparison based learning is the overall number of available comparisons: given $n$ examples, there exist $\bigO{n^3}$ different triplets and $\bigO{n^4}$ different quadruplets. In practice, it means that, in most applications, obtaining all the comparisons is not realistic. Instead, most approaches try to use as few comparisons as possible. This problem is relatively easy when the comparisons can be actively queried and it is known that $\bigOmega{n\ln n}$ adaptively selected comparisons are sufficient for various learning problems \citep{haghiri2017comparison,emamjomehzadeh2018adaptive,ghoshdastidar2019foundations}. On the other hand, this problem becomes harder when the comparisons are passively obtained. The general conclusion in most theoretical results on learning from passive ordinal comparisons is that, in the worst case,  almost all the $\bigO{n^3}$ or $\bigO{n^4}$ comparisons should be observed \citep{jamieson2011low,emamjomehzadeh2018adaptive}. The focus of this work is to show that, by carefully handling the passively obtained comparisons, it is possible to design comparison based approaches that use almost as few comparisons as active approaches for planted clustering problems.

\textbf{Near-optimal guarantees for clustering with passive comparisons.} 
In hierarchical clustering, \citet{emamjomehzadeh2018adaptive} showed that constructing a hierarchy that satisfies all comparisons in a top-down fashion requires $\bigOmega{n^3}$ passively obtained triplets in the worst case.
Similarly, \citet{ghoshdastidar2019foundations} considered a planted model and showed that $\bigOmega{n^{3.5}\ln n}$ passive quadruplets suffice to recover the true hierarchy in the data using a bottom-up approach.
Since the main difficulty lies in recovering the small clusters at the bottom of the tree, we believe that this latter result also holds for standard clustering.
In this paper, we consider a planted model for standard clustering and we show that, when the number of clusters $k$ is constant, $\bigOmega{n(\ln n)^2}$ passive triplets or quadruplets are sufficient for exact recovery. This result is comparable to the sufficient number of active comparisons in most problems, that is $\bigOmega{n\ln n}$ \citep{haghiri2017comparison,emamjomehzadeh2018adaptive}. Furthermore, it is near-optimal as to cluster $n$ objects it is necessary to observe all the examples at least once and thus have access to at least $\bigOmega{n}$ comparisons. Finally, to obtain these results, we study a semi-definite programming (SDP) based clustering method and our analysis could be of significant interest beyond the comparison based framework.

\textbf{General noise model for comparison based learning.}
In comparison based learning, there are two main sources of noise.
First, the observed comparisons can be noisy, that is the observed triplets and quadruplets are not in line with the underlying similarities.
This noise stems, for example, from the randomness of the answers gathered from a crowd. It is typically modelled by assuming that each observed comparison is randomly (and independently) flipped \citep{jain2016finite,emamjomehzadeh2018adaptive}.
This is mitigated in the active setting by repeatedly querying each comparison, but may have a significant impact in the passive setting where a single instance of each comparison is often observed.
Apart from the aforementioned observation errors, the underlying similarities may also have intrinsic noise.
For instance, the food data set by \citet{wilber2014hcomp} contains triplet comparisons in terms of which items taste more similar, and it is possible that the taste of a dessert is closer to a main dish than to another dessert.
This noise has been considered in \citet{ghoshdastidar2019foundations} by assuming that every pair of items possesses a latent random similarity, which affects the responses to comparisons.
In this paper, we propose, to the best of our knowledge, the first analysis that considers and shows the impact of both types of noise on the number of passive comparisons.

\textbf{Scalable comparison based similarity functions.}
Several similarity and kernel functions have been proposed in the literature \citep{kleindessner2017kernel,ghoshdastidar2019foundations}.
However, computing these similarities is usually expensive as they require up to $\bigO{n}$ passes over the set of available comparisons.
In this paper, we propose new similarity functions whose construction is much more efficient than previous kernels.
Indeed, they can be obtained with a single pass over the set of available comparisons.
It means that our similarity functions can be computed in an online fashion where the comparisons are obtained one at a time from a stream.
The main drawback compared to existing approaches is that we lose the positive semi-definiteness of the similarity matrix, but our theoretical results show that this is not an issue in the context of clustering.
We also demonstrate this empirically as our similarities obtain results that are comparable with state of the art methods.

\section{Background and theoretical framework}
\label{sec_background}

In this section, we present the comparison based framework and our planted clustering model, under which we later show that a small number of passive comparisons suffices for learning. 
%
We consider the following setup.
There are $n$ items, denoted by $[n] = \{1,2,\ldots,n\}$, and we assume that, for every pair of distinct items $i,j \in [n]$, there is an implicit real-valued similarity $w_{ij}$ that we cannot directly observe. 
Instead, we have access to
\begin{equation}
\label{eqn_comparisons}
\begin{aligned} 
&\text{Triplets: }
&& \setT = \left\{ (i,j,r) \in [n]^3 ~:~ w_{ij} > w_{ir}, ~i,j,r \text{ distinct} \right\},
\qquad \text{or}
\\
&\text{Quadruplets: }
&& \setQ = \left\{ (i,j,r,s) \in [n]^4 ~:~ w_{ij} > w_{rs}, ~i\neq j,r\neq s, (i,j)\neq (r,s) \right\}.
\end{aligned}
\end{equation}
%
%

There are $\bigO{n^4}$ possible quadruplets and $\bigO{n^3}$ possible triplets, and it is expensive to collect such a large number of comparisons via crowdsourcing.
In practice, $\setT$ or $\setQ$ only contain a small fraction of all possible comparisons.
We note that if a triple $i,j,r\in[n]$ is observed with $i$ as reference item, then either $(i,j,r) \in \setT$ or $(i,r,j)\in \setT$ depending on whether $i$ is more similar to $j$ or to $r$. Similarly, when tuples $(i,j)$ and $(r,s)$ are compared, we have either $(i,j,r,s)\in\setQ$ or $(r,s,i,j) \in \setQ$.

\textbf{Sampling and noise in comparisons.}
This paper focuses on passive observation of comparisons. 
To model this, we assume that the comparisons are obtained via uniform sampling, and every comparison is equally likely to be observed.
Let $p \in(0,1]$ denote a sampling rate that depends on $n$. 
We assume that every comparison (triplet or quadruplet) is independently observed with probability $p$.
In expectation, $|\setQ | = \bigO{pn^4}$ and $|\setT| = \bigO{pn^3}$, and we can control the sampling rate $p$ to study the effect of the number of observations, $|\setQ|$ or $|\setT|$, on the performance of an algorithm.

As noted in the introduction, the observed comparisons are typically noisy due to random flipping of answers by the crowd workers and inherent noise in the similarities. 
To model the external (crowd) noise we follow the work of \citet{jain2016finite} and, given a parameter $\epsilon \in (0,1]$, we assume that any observed comparison is correct with probability $\frac12(1+\epsilon)$ and flipped with probability $\frac12(1-\epsilon)$.
To be precise, for observed triple $i,j,r\in [n]$ such that $w_{ij} > w_{ir}$,
\begin{align}
    \Pb\big( (i,j,r) \in \setT ~|~ w_{ij}>w_{ir} \big) = \frac{1+\epsilon}{2}, 
    \quad \text{whereas} \quad
    \Pb\big( (i,r,j) \in \setT ~|~ w_{ij}>w_{ir} \big) = \frac{1-\epsilon}{2}. 
    \label{eqn_Errcond}
\end{align}
The probabilities for flipping quadruplets can be similarly expressed.
We model the inherent noise by assuming $w_{ij}$ to be random, and present a model for the similarities under planted clustering.

\textbf{Planted clustering model.}
We now present a theoretical model for the inherent noise in the similarities that reflects a clustered structure of the items.
The following model is a variant of the popular stochastic block model, studied in the context of graph clustering \citep{abbe2017community}, and is related to the non-parametric weighted stochastic block model \citep{xu2020optimal}.

We assume that the item set $[n]$ is partitioned into $k$ clusters $\setC_1,\ldots,\setC_k$ of sizes $n_1,\ldots,n_k$, respectively, but
\textbf{the number of clusters $k$ as well as the clusters $\setC_1,\ldots,\setC_k$ are unknown to the algorithm.}
Let $F_{in}$ and $F_{out}$ be two distributions defined on $\mathbb{R}$.
We assume that the  inherent (and unobserved) similarities $\{w_{ij} : i <j\}$ are random and mutually independent, and
\begin{align*}
w_{ij} \sim F_{in} \quad \text{if } i,j \in C_\ell \text{ for some } \ell, \qquad \text{and} \qquad w_{ij} \sim F_{out} \quad \text{otherwise.}
\end{align*}
We further assume that $w_{ii}$ is undefined,  $w_{ji} = w_{ij}$, and that for $w,w'$ independent,
\begin{equation}
\begin{aligned}
&\Pb_{w,w' \sim F_{in}} (w>w') = \Pb_{w,w' \sim F_{out}} (w>w') = 1/2, \quad \text{and}
\\
&\Pb_{w\sim F_{in},w'\sim F_{out}}(w>w') = (1+\delta)/2 \quad \text{for some } \delta \in (0,1].
\end{aligned}
\label{eqn_Fcondn}
\end{equation}
The first condition in \eqref{eqn_Fcondn} requires that $F_{in},F_{out}$ do not have point masses, and is assumed for analytical convenience. 
The second condition ensures that within cluster similarities are larger than inter-cluster similarities---a natural requirement. 
\citet{ghoshdastidar2019foundations} used a special case of the above model, where $F_{in},F_{out}$ are assumed to be Gaussian with identical variances $\sigma^2$, and means satisfy $\mu_{in} > \mu_{out}$ .
In this case, $\delta = 2\Phi\big(({\mu_{in} - \mu_{out})}/{\sqrt{2}\sigma}\big)-1$ where $\Phi$ is the cumulative distribution function of the standard normal distribution.



%


\textbf{The goal of this paper is to obtain bounds on the number of passively obtained triplets/quadruplets that are sufficient to recover the aforementioned planted clusters with zero error.}
To this end, we propose two similarity functions respectively computed from triplet and quadruplet comparisons, and show that a similarity based clustering approach using semi-definite programming (SDP) can exactly recover clusters planted in the data using few passive comparisons.

\section{A theoretical analysis of similarity based clustering}

Before presenting our new comparison based similarity functions, we describe the SDP approach for clustering from similarity matrices that we use throughout the paper \citep{yan2018provable,chen2020diffusion}. 
In addition, we prove a generic theoretical guarantee for this approach that holds for any similarity matrix and, thus, that could be of interest even beyond the comparison based setting.

%
Similarity based clustering is widely used in machine learning, and there exist a range of popular approaches including spectral methods \citep{vonluxburg2007tutorial}, semi-definite relaxations \citep{yan2016robustness}, or linkage algorithms \citep{dasgupta2016cost} among others.
We consider the following SDP for similarity based clustering.
Let $S \in \nsetR^{n\times n}$ be a symmetric similarity matrix among $n$ items, and $Z \in \{0,1\}^{n\times k}$ be the cluster assignment matrix that we wish to estimate.
For unknown number of clusters $k$, it is difficult to directly determine $Z$, and hence, we estimate the \emph{normalised clustering matrix} $X \in \nsetR^{n\times n}$ such that
$X_{ij} = \frac{1}{|\setC|}$ if $i,j$ co-occur in estimated cluster $\setC$, and $X_{ij} = 0$ otherwise.
Note that $\tr{X}=k$. 
The following SDP was proposed and analysed by \citet{yan2018provable} under the stochastic block model for graphs, and can also be applied in the more general context of data clustering \citep{chen2020diffusion}. This SDP is agnostic to the number of clusters, but penalises large values of $\tr{X}$ to restrict the number of estimated clusters:
 \begin{equation}
 \begin{aligned}
 \max_X~& \tr{SX} - \lambda\, \tr{X} \\
 \text{s.t.}~& X \geq 0, \quad X \succeq 0, \quad X\1 = \1.
\end{aligned}
\label{eqn_sdp}\tag{SDP-$\lambda$}
\end{equation}

Here, $\lambda$ is a tuning parameter and $\1$ denotes the vector of all ones. The constraints $X\geq0$ and $X\succeq0$ restricts the optimisation to non-negative, positive semi-definite matrices.

We first present a general theoretical result for \ref{eqn_sdp}.
Assume that the data has an implicit partition into $k$ clusters $\setC_1,\ldots,\setC_k$ of sizes $n_1,\ldots,n_k$ and with cluster assignment matrix $Z$, and suppose that the similarity $S$ is close to an \emph{ideal similarity matrix} $\St$ that has a $k\times k$ block structure $\St = Z\Sigma Z^T$.
The matrix $\Sigma \in \nsetR^{k\times k}$ is such that $\Sigma_{\ell\ell'}$ represents the ideal pairwise similarity between items from clusters $\setC_\ell$ and $\setC_{\ell'}$.
Typically, under a random planted model, $\St$ is the same as $\Eb[S]$ up to possible differences in the diagonal terms.
For $S=\St$ and certain values of $\lambda$, the unique optimal solution of \ref{eqn_sdp} is a block diagonal matrix $X^* = ZN^{-1}Z^T$, where $N \in \nsetR^{k\times k}$ is diagonal with entries $n_1,\ldots,n_k$ (see Appendix \ref{app_prop_sdp}).
Thus, in the \emph{ideal case}, solving the SDP provides the desired normalised clustering matrix from which one can recover the partition $\setC_1,\ldots,\setC_k$.
The following result shows that $X^*$ is also the unique optimal solution of \ref{eqn_sdp} if $S$ is sufficiently close to $\St$.

\begin{myprop}[\textbf{Recovery of planted clusters using SDP-$\lambda$}]
\label{prop_sdp}
Let $Z\in\{0,1\}^{n\times k}$ be the assignments for a planted $k$-way clustering, $\St = Z\Sigma Z^T$, and $X^* = ZN^{-1}Z^T$ as defined above.
Define
\begin{align*}
\Delta_1 = \min_{\ell\neq \ell'} \left(\frac{\Sigma_{\ell\ell} + \Sigma_{\ell'\ell'}}{2} - \Sigma_{\ell\ell'}\right),
\quad \text{and} \quad
\Delta_2 = \max_{i\in [n]} \max_{\ell \in [k]} \left| \frac{1}{|\setC_\ell|} \sum_{j \in \setC_\ell} \left(S_{ij} - \St_{ij}\right) \right|.
\end{align*}
$X^*$ is the unique optimal solution of \ref{eqn_sdp} for any choice of $\lambda$ in the interval
\begin{align*}
\left\Vert S - \St \right\Vert_2 < \lambda < \min_\ell n_\ell \cdot \min\left\{ \frac{\Delta_1}{2}\,, \,  \Delta_1 - 6\Delta_2 \right\} \,.
\end{align*}
\end{myprop}

The proof of Proposition~\ref{prop_sdp}, given in Appendix~\ref{app_prop_sdp}, is adapted from \citet{yan2018provable} although uniqueness was not proved in this previous work.
The term $\Delta_1$ quantifies the separation between the ideal within and inter-cluster similarities, and is similar in spirit to the weak assortativity criterion for stochastic block models \citep{yan2018provable}. 
On the other hand, the matrix spectral norm $\Vert S - \St \Vert_2$ and the term $\Delta_2$ both quantify the deviation of the similarities $S$ from their ideal values $\St$.
Note that the number of clusters can be computed as $k = \tr{X}$ and cluster assignment $Z$ is obtained by clustering the rows of $X^*$ using $k$-means or spectral clustering for example.
In the experiments (Section~\ref{sec:experiments}), we present a data-dependent approach to tune $\lambda$ and find $k$.

We conclude this section by noting that most of the previous analyses of SDP clustering either assume sub-Gaussian data \citep{yan2016robustness} or consider similarity matrices with independence assumptions \citep{chen2014statistical,yan2018provable}
that might not hold in general, and do not hold for our \ref{eqn_adis3} and \ref{eqn_adis4} similarities described in the next section.
In contrast, the deterministic criteria stated in Proposition~\ref{prop_sdp} make the result applicable in more general settings.

\section{Similarities from passive comparisons}

We present two new similarity functions computed from passive comparisons (\ref{eqn_adis3} and \ref{eqn_adis4}) and guarantees for recovering planted clusters using \ref{eqn_sdp} in conjunction with these similarities. 
\citet{kleindessner2017kernel} introduced pairwise similarities computed from triplets.
A quadruplets variant was proposed by \citet{ghoshdastidar2019foundations}.
These similarities, detailed in Appendix~\ref{app_kernels}, are positive-definite kernels and have multiplicative forms.
In contrast, we compute the similarity between items $i,j$ by simply adding binary responses to comparisons involving $i$ and $j$.

\textbf{Similarity from quadruplets.} 
We construct the additive similarity for quadruplets, referred to as \ref{eqn_adis4}, in the following way. 
Recall the definition of $\setQ$ in Equation~\eqref{eqn_comparisons} and for every $i\neq j$, define
\begin{equation}
S_{ij} = \sum_{r\neq s}  \big(\ind{(i,j,r,s)\in\setQ} - \ind{(r,s,i,j)\in\setQ}\big),
\label{eqn_adis4}\tag{AddS\mbox{-}4}
\end{equation}
where $\ind{\cdot}$ is the indicator function. 
The intuition is that if $i,j$ are similar ($w_{ij}$ is large), then for every observed tuple $i,j,r,s$, $w_{ij}>w_{rs}$ is more likely to be observed.
Thus, $(i,j,r,s)$ appears in $\setQ$ more often than $(r,s,i,j)$, and $S_{ij}$ is a (possibly large) positive term. On the other hand, smaller $w_{ij}$ leads to a negative value of $S_{ij}$.
Under the aforementioned planted model with clusters of size $n_1,\ldots,n_k$, one can verify that $S_{ij}$ indeed reveals the planted clusters in expectation since
if $i,j$ belong to the same planted cluster, then $\Eb[S_{ij}] = \displaystyle p\epsilon\delta \sum_{\ell \in [k]} \frac{n_\ell (n-n_\ell)}{2}$, and $\Eb[S_{ij}] =\displaystyle - p\epsilon\delta \sum_{\ell \in [k]} \binom{n_\ell}{2}$ otherwise. 
%
Thus, in expectation, the within cluster similarity exceeds the inter-cluster similarity by $p\epsilon\delta\binom{n}{2}$.

\textbf{Similarity from triplets.}
The additive similarity based on passive triplets \ref{eqn_adis3} is given by
\begin{align}
S_{ij} &= \sum_{r\neq i,j}  \big(\ind{(i,j,r)\in\setT} - \ind{(i,r,j)\in\setT}\big) + \big(\ind{(j,i,r)\in\setT} - \ind{(j,r,i)\in\setT}\big)
\label{eqn_adis3}\tag{AddS\mbox{-}3}
\end{align}
for every $i\neq j$.
The \ref{eqn_adis3} similarity $S_{ij}$ aggregates all the comparisons that involve both $i$ and $j$, with either $i$ or $j$ as the reference item.
Similar to the case of \ref{eqn_adis4}, $S_{ij}$ tends to be positive when $w_{ij}$ is large, and negative for small $w_{ij}$.
One can also verify that, under a planted model, the expected within cluster \ref{eqn_adis3} similarity exceeds the inter-cluster similarity by $p\epsilon\delta( n - 2)$.

A significant advantage of \ref{eqn_adis3} and \ref{eqn_adis4} over existing similarities is in terms of computational time for constructing $S$.
Unlike existing kernels, both similarities can be computed from a single pass over $\setT$ or $\setQ$.
In addition, the following result shows that 
the proposed similarities can exactly recover planted clusters using only a few (near optimal) number of passive comparisons.

\begin{myth}[\textbf{Cluster recovery using AddS-3 and AddS-4}]
\label{thm_adis}
Let $X^*$ denote the normalised clustering matrix corresponding to the true partition, and $n_{\min}$ be the size of the smallest planted cluster.
Given the triplet or the quadruplet setting, there exist absolute constants $c_1,c_2,c_3,c_4>0$ such that, with probability at least $1-\frac1n$, $X^*$ is the unique optimal solution of \ref{eqn_sdp} if $\delta$ satisfies 
$\displaystyle     c_1 \frac{\sqrt{n \ln n}}{n_{\min}} < \delta \leq 1\;$,
and one of the following two conditions hold:
    
$\bullet$~\textbf{(triplet setting)} 
 $S$ is given by \ref{eqn_adis3}, and the number of triplets $|\setT|$ and the parameter $\lambda$ satisfy
   \begin{align*}
     |\setT| > c_2 \displaystyle \frac{n^3 (\ln n)^2}{\epsilon^2 \delta^2 n_{\min}^2}
    \quad\text{and}\quad
     c_3  \max\left\{\sqrt{ |\setT| \frac{ \ln n}{n}}\,, |\setT| \epsilon \sqrt{\frac{\ln n}{n^3}}\,, (\ln n)^2\right\} < \lambda < c_4 |\setT| \frac{\epsilon \delta n_{\min}}{n^2} \;;
    \end{align*}
$\bullet$~\textbf{(quadruplet setting)} 
 $S$ is given by \ref{eqn_adis4}, and the number of quadruplets $|\setQ|$ and $\lambda$ satisfy
     \begin{align*}
     |\setQ| > c_2 \displaystyle \frac{n^3 (\ln n)^2}{\epsilon^2 \delta^2 n_{\min}^2} 
     \quad\text{and}\quad
    c_3 \max\left\{\sqrt{ |\setQ| \frac{ \ln n}{n}}\,, |\setQ| \epsilon \sqrt{\frac{\ln n}{n^3}}\,, (\ln n)^2\right\} < \lambda < c_4 |\setQ| \frac{\epsilon \delta n_{\min}}{n^2} \;.
    \end{align*}
The condition on $\delta$ and the number of comparisons ensure that the interval for $\lambda$ is non-empty.
\end{myth}

Theorem~\ref{thm_adis} is proved in Appendix~\ref{app_thm_adis}. This result shows that given a sufficient number of comparisons, one can exactly recover the planted clusters using \ref{eqn_sdp} with an appropriate choice of $\lambda$.
In particular, if there are $k$ planted clusters of similar sizes and $\delta$ satisfies the stated condition, then recovery of the planted clusters with zero error is possible with only $\bigOmega{\frac{k^2}{\epsilon^2 \delta^2} n (\ln n)^2}$ passively obtained triplets or quadruplets.
We make a few important remarks about the sufficient conditions stated in Theorem~\ref{thm_adis}.

\begin{myrem}[\textbf{Comparison with existing results}]
\label{rem_theory_comparison}
For fixed $k$ and fixed $\epsilon,\delta \in(0,1]$, Theorem~\ref{thm_adis} states that $\bigOmega{n(\ln n)^2}$ passive comparisons (triplets or quadruplets) suffice to exactly recover the clusters.
This significantly improves over the $\bigOmega{n^{3.5}\ln n}$ passive quadruplets needed by \citet{ghoshdastidar2019foundations} in a planted setting, and the fact that $\bigOmega{n^3}$ triplets are necessary in the worst case \citep{emamjomehzadeh2018adaptive}.
\end{myrem}

\begin{myrem}[\textbf{Dependence of the number of comparisons on the noise levels $\epsilon,\delta$}]
When one can actively obtain comparisons, \citet{emamjomehzadeh2018adaptive} showed that it suffices to query  $\bigOmega{n\ln\left(\frac{n}{\epsilon}\right)}$ triplets.
Compared to the $\ln\left(\frac1\epsilon\right)$ dependence in the active setting, the sufficient number of passive comparisons in Theorem~\ref{thm_adis} has a stronger dependence of $\frac{1}{\epsilon^2}$ on the crowd noise level $\epsilon$.
While we do not know whether this dependence is optimal, the stronger criterion is intuitive since, unlike the active setting, the passive setting does not provide repeated observations of the same comparisons that can easily nullify the crowd noise. The number of comparisons also depends as $\frac{1}{\delta^2}$ on the inherent noise level, which is similar to the conditions in \citet{ghoshdastidar2019foundations}.
\end{myrem}

Theorem~\ref{thm_adis} states that exact recovery primarily depends on two sufficient conditions, one on $\delta$ and the other on the number of passive comparisons $(|\setT| \text{ or } |\setQ|)$.
The following two remarks show that both conditions are necessary, up to possible differences in logarithmic factors.

\begin{myrem}[\textbf{Necessity of the condition on $\delta$}]
\label{rem_optimal_delta}
The condition on $\delta$ imposes the condition of $n_{\min} = \bigOmega{\sqrt{n\ln n}}$. 
This requirement on $n_{\min}$ appears naturally in planted problems.
Indeed, assuming that all $k$ clusters are of similar sizes, the above condition is equivalent to a requirement of $k = \bigO{\sqrt{\frac{n}{\ln n}}}$ and it is believed that polynomial time algorithms cannot recover $k \gg \sqrt{n}$ planted clusters \citep[Conjecture 1]{chen2014statistical}.
\end{myrem}

\begin{myrem}[\textbf{Near-optimal number of comparisons}]
\label{rem_optimal_comparison}
To cluster $n$ items, one needs to observe each example at least once. Hence, one trivially needs at least $\bigOmega{n}$ comparisons (active or passive).
Similarly, existing works on actively obtained comparisons show that $\bigOmega{n\ln n}$ comparisons are sufficient for learning in supervised or unsupervised problems \citep{haghiri2017comparison,emamjomehzadeh2018adaptive,ghoshdastidar2019foundations}.
We observe that, in the setting of Remark~\ref{rem_theory_comparison}, it suffices to have $\bigOmega{n(\ln n)^2}$ passive comparisons which matches the necessary conditions up to logarithmic factors. 
However, the sufficient condition on the number of comparisons becomes $\bigOmega{k^2 n (\ln n)^2}$ if $k$ grows with $n$ while $\epsilon$ and $\delta$ are fixed.
It means that the worst case of $k = \bigO{\sqrt{\frac{n}{\ln n}}}$, stated in Remark~\ref{rem_optimal_delta}, can only be tackled using at least $\bigOmega{n^2 \ln n}$ passive comparisons.
\end{myrem}

\begin{myrem}[\textbf{No new information beyond $\bigOmega{n^2/\epsilon^2}$ comparisons}]
Theorem \ref{thm_adis} shows that for large $n$ and $\bigOmega{n^2/\epsilon^2}$ number of comparisons,
the condition for exact recovery of the clusters is only governed by the condition on $\delta$ as the interval for $\lambda$ is always non empty. It means that, beyond a quadratic number of comparisons, no new information is gained by observing more comparisons.
This explains why significantly fewer passive comparisons suffice in practice than the known worst-case requirements of $\bigOmega{n^3}$ passive triplets or $\bigOmega{n^4}$ passive quadruplets.
\end{myrem}

We conclude our theoretical discussion with a remark about recovering planted clusters when the pairwise similarities $w_{ij}$ are observed.
Our methods are near optimal even in this setting.

\begin{myrem}[\textbf{Recovering planted clusters for non-parametric $F_{in},F_{out}$}]
Theoretical studies in the classic setting of clustering with observed pairwise similarities $\{w_{ij}:i<j\}$ typically assume that the distributions $F_{in}$ and $F_{out}$ for the pairwise similarities are Bernoulli (in unweighted graphs), or take finitely many values (labelled graphs), or belong to exponential families
\citep{chen2014statistical,aicher2015learning,yun2016optimal}.
Hence, the applicability of such results are restrictive.
Recently, \citet{xu2020optimal} considered non-parametric distributions for $F_{in},F_{out}$, and presented a near-optimal approach based on discretisation of the similarities into finitely many bins.
Our work suggests an alternative approach: compute ordinal comparisons from the original similarities and use clustering on \ref{eqn_adis3} or \ref{eqn_adis4}.
Theorem~\ref{thm_adis} then guarantees, for any non-parametric and continuous $F_{in}$ and $F_{out}$, exact recovery of the planted clusters under a near-optimal condition on $\delta$. 
\end{myrem}

\section{Experiments}
\label{sec:experiments}

The goal of this section is three-fold: present a strategy to tune $\lambda$ in \ref{eqn_sdp}; empirically validate our theoretical findings; and demonstrate the performance of the proposed approaches on real datasets. 

\textbf{Choosing $\lambda$ and estimating the number of clusters based on Theorem \ref{thm_adis}.}
Given a similarity matrix $S$, the main difficulty involved in using \ref{eqn_sdp} is tuning the parameter $\lambda$.
\citet{yan2018provable} proposed the algorithm SPUR to select the best $\lambda$ as $\lambda^* = \argmax_{0\leq\lambda \leq \lambda_\text{max}}  \frac{\sum_{i\leq k_{\lambda}} \sigma_i(X_{\lambda})}{\tr{X_\lambda}}$ where $X_\lambda$ is the solution of \ref{eqn_sdp}, $k_{\lambda}$ is the integer approximation of $\tr{X_\lambda}$ and an estimate of the number of clusters, $\sigma_i(X_\lambda)$ is the $i$-th largest eigenvalue of $X_\lambda$, and $\lambda_\text{max}$ is a theoretically well-founded upper bound on $\lambda$.
The maximum of the above objective is 1, achieved when $X_\lambda$ has the same structure as $X^*$ in Proposition \ref{prop_sdp}.
In our setting, Theorem~\ref{thm_adis} gives an upper bound on $\lambda$ that depends on $\epsilon$, $\delta$ and $n_{\min}$ which are not known in practice.
Furthermore, it is computationally beneficial to use the theoretical lower bound for $\lambda$ instead of using $\lambda\geq0$ as suggested in SPUR. 

We propose to modify SPUR based on the fact that the estimated number of clusters $k$ monotonically decreases with $\lambda$ (details in Appendix \ref{app: sec: Algorithm}).
Given Theorem \ref{thm_adis}, we choose $\lambda_{\text{min}} = \sqrt{ c (\ln n)/n}$ and $\lambda_{\text{max}} = c/n$, where $c = |\setQ|$ or $|\setT|$. The trace of the \ref{eqn_sdp} solution then gives two estimates of the number of clusters, $k_{\lambda_{\text{min}}}$ and $k_{\lambda_{\text{max}}}$, and we search over $k \in [k_{\lambda_{\max}},k_{\lambda_{\min}}]$ instead of searching over $\lambda$---in practice, it helps to search over the values $\max\{2,k_{\lambda_{\max}}\} \leq k \leq k_{\lambda_{\min}}+2$.
We select $k$ that maximises the above SPUR objective, where $X$ is computed using a simpler SDP \citep{yan2018provable}:
\begin{align}
 \textstyle{\max_X}~& \langle S,X \rangle &\text{ s.t.~ } X \geq 0, \quad X \succeq 0, \quad X\1 = \1, \quad \tr{X} = k.
\label{eqn_sdp_PW}\tag{SDP\mbox{-}$k$}
\end{align}

\textbf{Clustering with \ref{eqn_adis3} and \ref{eqn_adis4}.}
For the proposed similarity matrices \ref{eqn_adis3} and \ref{eqn_adis4}, the above strategy provides the optimal number of clusters $k$ and a corresponding solution $X_k$ of \ref{eqn_sdp_PW}.
The partition is obtained by clustering the rows of $X_k$ using $k$-means. Alternative approaches, such as spectral clustering, lead to similar performances (see Appendix \ref{app: sec: planted}).

%











%
\textbf{Evaluation function.} We use the Adjusted Rand Index (ARI) \citep{hubert1985ARI} between the ground truth and the predictions. The ARI takes values in $[-1, 1]$ and measures the agreement between two partitions: $1$ implies identical partitions, whereas $0$ implies that the predicted clustering is random. In all the experiments, we report the mean and standard deviation over 10 repetitions.

\begin{figure*}
     \centering
     \begin{subfigure}[b]{0.32\textwidth}
         \centering
        \includegraphics[height=0.65\textwidth,clip=true,trim=2mm 2mm 185mm 3mm]{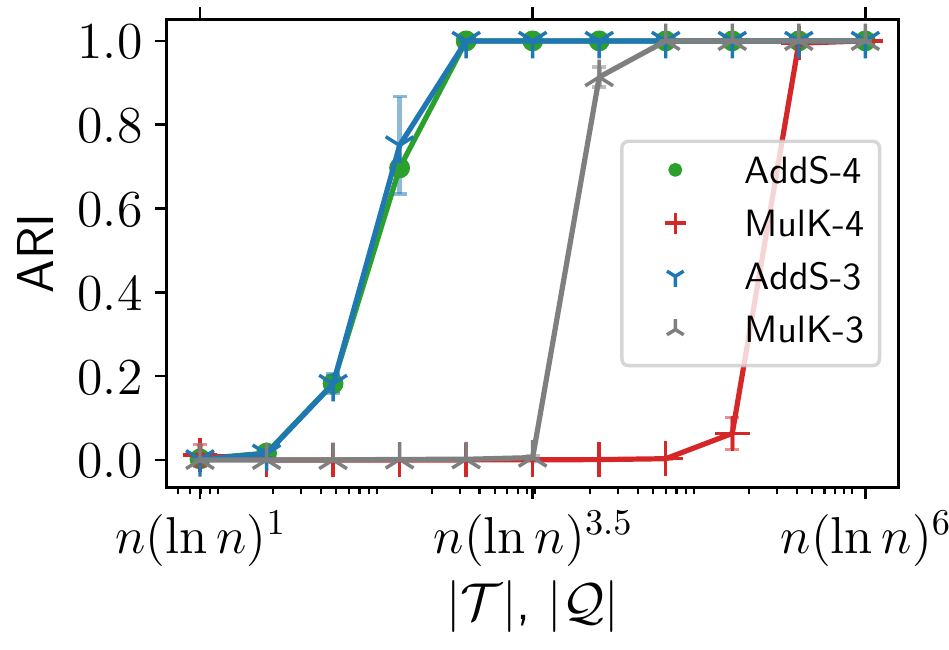}
         \caption{Vary the number of comparisons}
         \label{fig: planted sample number}
     \end{subfigure}
     \hfill
     \begin{subfigure}[b]{0.32\textwidth}
         \centering
        \includegraphics[height=0.65\textwidth,clip=true,trim=100mm 2mm 90mm 3mm]{figures/planted.pdf}
         \caption{Vary the external noise level, $\epsilon$}
         \label{fig: planted epsilon}
     \end{subfigure}
     \hfill
     \begin{subfigure}[b]{0.32\textwidth}
         \centering
        \includegraphics[height=0.65\textwidth,clip=true,trim=195mm 2mm 3mm 3mm]{figures/planted.pdf}
         \caption{Vary the distributions $F_{in}$, $F_{out}$}
         \label{fig: planted distributions}
     \end{subfigure}
        \caption{ARI of various methods on the planted model (higher is better). We vary: \eqref{fig: planted sample number} the number of comparisons $|\mathcal T|$ and $|\mathcal Q|$; \eqref{fig: planted epsilon} the crowd noise level $\epsilon$; \eqref{fig: planted distributions} the distributions $F_{in}$ and $F_{out}$.
        }
        \label{fig: planted cluster}
\end{figure*}

\textbf{Simulated data with planted clusters.} We generate data using the planted model from Section~\ref{sec_background} and verify that the learned clusters are similar to the planted ones. As default parameters we use $n=1000$, $k=4$, $\epsilon=0.75$, $|\mathcal{T}| = |\mathcal{Q}|= n(\ln n)^4$ and $F_{in}=\mathcal{N}\left(\sqrt{2}\sigma\Phi^{-1}\left( \frac{1+\delta}{2} \right),\sigma^2\right),F_{out}=\mathcal{N}\left(0,\sigma^2\right)$ with  $\sigma = 0.1$ and $\delta = 0.5$. In each experiment, we investigate the sensitivity of our method by varying one of the parameters while keeping the others fixed. We use SPUR to estimate the number of clusters. As baselines, we use \ref{eqn_sdp_PW} (using the number of clusters estimated by our approaches) followed by $k$-means with two comparison based multiplicative kernels: \ref{eqn_3k} for triplets \citep{kleindessner2017kernel} and \ref{eqn_4k} for quadruplets \citep{ghoshdastidar2019foundations}.

We present some significant results in Figure~\ref{fig: planted cluster} and defer the others to Appendix~\ref{app: sec: planted}. In Figure~\ref{fig: planted sample number}, we vary the number of sampled comparisons. Unsurprisingly, our approaches are able to exactly recover the planted clusters using as few as $n(\ln n)^3$ comparisons---extra $\ln n$ factor compared to Theorem~\ref{thm_adis} accounts for $\epsilon,\delta$ and constants. \ref{eqn_3k} and \ref{eqn_4k} respectively need $n(\ln n)^{4.5}$ and $n(\ln n)^{5.5}$ comparisons (both values exceed $n^2$ for $n=1000$). In all our experiments, \ref{eqn_adis3} and \ref{eqn_adis4} have comparable performance while \ref{eqn_3k} is significantly better than \ref{eqn_4k}. Thus we focus on triplets in the subsequent experiments for the sake of readability. In Figure~\ref{fig: planted epsilon}, we vary the external noise level $\epsilon$. Given $n(\ln n)^4$ comparisons, \ref{eqn_adis3} exactly recovers the planted clusters for $\epsilon$ as small as $0.25$ (high crowd noise) while, given the same number of comparisons, \ref{eqn_3k} only recovers the planted clusters for $\epsilon > 0.9$. 
Figure~\ref{fig: planted distributions} shows that \ref{eqn_adis3} outperforms \ref{eqn_3k} even when different distributions for $F_\text{in}$ and $F_\text{out}$ are considered (Uniform\,+\,Beta or Uniform\,+\,Normal; details in Appendix~\ref{app: sec: planted}).
It also shows that the distributions affect the performances, which is not evident from Theorem~\ref{thm_adis}, indicating the possibility of a refined analysis under distributional assumptions.




\textbf{MNIST clustering with comparisons.} We consider two datasets which are subsets of the MNIST test data \citep{lecun2010mnisthandwrittendigit}: (i) a subset of 2163 examples containing all $1$ and $7$ (\textit{MNIST 1vs.7}), two digits that are visually very similar, and (ii) a randomly selected subsets of 2000 examples from all 10 classes (\textit{MNIST 10}). To generate the comparisons, we use the Gaussian similarity on a 2-dimensional embedding of the entire MNIST test data constructed with t-SNE \citep{van2014accelerating} and normalized so that each example lies in $[-1,1]^2$.
We focus on the triplet setting and consider additional baselines.
First, we use t-STE \citep{van2012stochastic}, an ordinal embedding approach, to embed the examples in 2 dimensions, and then cluster them using $k$-means on the embedded data. Second, we directly use $k$-means on the normalized data obtained with t-SNE.
The latter is a baseline with access to Euclidean data instead of triplet comparisons.

%
For \textit{MNIST 1vs.7} (Figure~\ref{fig: mnist 1v7}), $|\mathcal{T}| = n(\ln n)^2$ is sufficient for \ref{eqn_adis3} to reach the performance of $k$-means and t-STE while \ref{eqn_3k} requires $n(\ln n)^3$ triplets. Furthermore, note that \ref{eqn_adis3} with known number of clusters performs similarly to \ref{eqn_adis3} using SPUR, indicating that SPUR estimates the number of clusters correctly.
If we consider \textit{MNIST 10} (\autoref{fig: mnist 10}) and $|\mathcal{T}| = n(\ln n)^2$,  \ref{eqn_adis3} with known $k$ outperforms  \ref{eqn_adis3} using SPUR, suggesting that the number of comparisons here is not sufficient to estimate the number of clusters accurately. Moreover, \ref{eqn_adis3} with known $k$ outperforms \ref{eqn_3k} while being close to the performance of t-STE. Finally for $n(\ln n)^4$ triplets, all ordinal methods converge to the baseline of $k$-means with access to original data. The ARI of \ref{eqn_adis3} SPUR improves when the number of comparisons increases due to better estimations of the number of clusters---estimated $k$ increases from 3 for $|\mathcal{T}| = n(\ln n)^2$ up to 9 for $|\mathcal{T}| = n(\ln n)^4$.

\textbf{Real comparison based data.}
We consider the Food dataset \citep{wilber2014hcomp} in Appendix~\ref{app: sec: realdata} and the Car dataset \citep{kleindessner2016cardataset} here. It contains 60 examples grouped into 3 classes (SUV, city cars, sport cars) with 4 outliers, and exhibits 12112 triplet comparisons.
For this dataset, \ref{eqn_adis3} SPUR estimates $k=2$ instead of the correct 3 clusters.
Figure~\ref{fig: car} considers all ordinal methods with $k=2$ and $k=3$, and shows the pairwise agreement (ARI) between different methods and also with the true labels.
While \ref{eqn_3k} with $k=3$ agrees the most with the true labels, all the clustering methods agree well for $k=2$ (top-left $3\times3$ block). 
Hence, the data may have another natural clustering with two clusters, suggesting possible discrepancies in how different people judge the similarities between cars (for instance, color or brand instead of the specified classes).

\begin{figure}
     \centering
     \begin{subfigure}[b]{0.32\textwidth}
         \centering
        \includegraphics[height=0.65\textwidth,clip=true,trim=1mm 1mm 90mm 3mm]{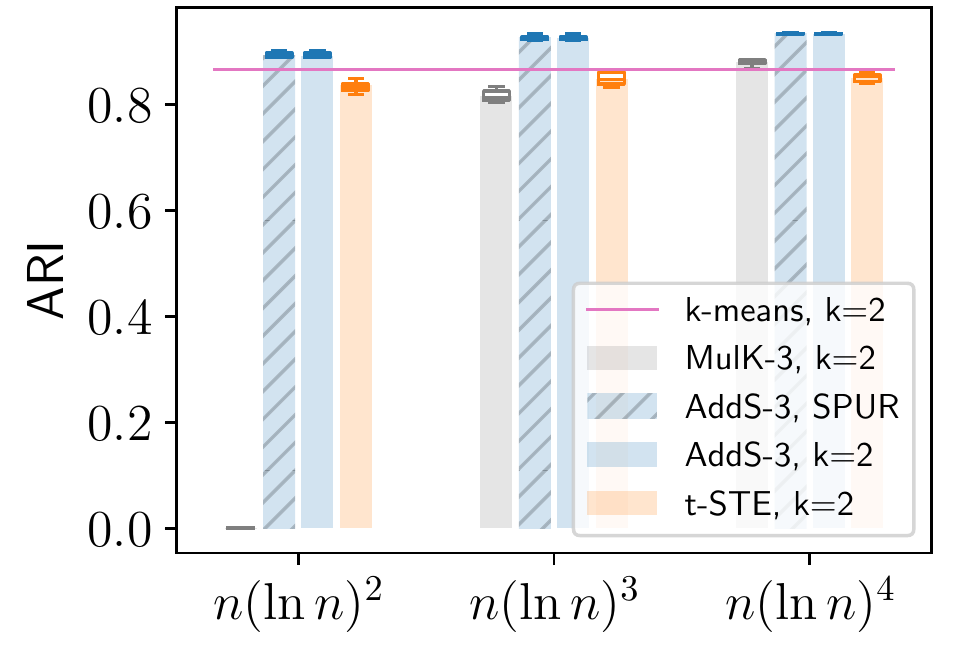}
         \caption{MNIST 1vs.7, $n=2163$}
         \label{fig: mnist 1v7}
     \end{subfigure}
     \hfill
     \begin{subfigure}[b]{0.32\textwidth}
         \centering
        \includegraphics[height=0.65\textwidth,clip=true,trim=97mm 1mm 1mm 3mm]{figures/real_data_mnist.pdf}
         \caption{MNIST 10, $n=2000$}
         \label{fig: mnist 10}
     \end{subfigure}
     \hfill
     \begin{subfigure}[b]{0.32\textwidth}
         \centering
        \includegraphics[height=0.65\textwidth,clip=true,trim=0mm 1mm 3mm 3mm]{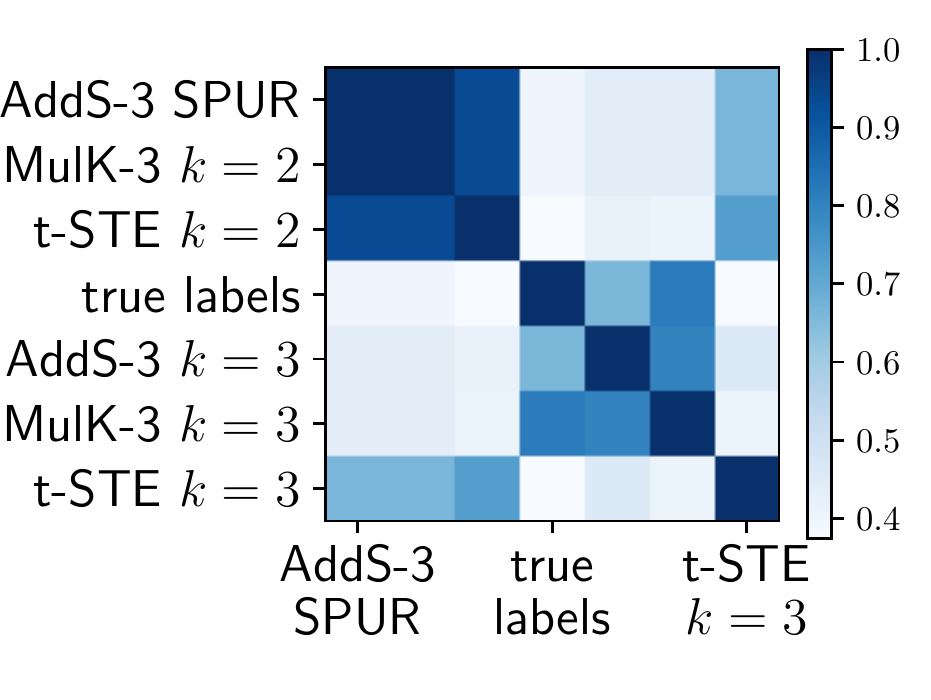}
         \caption{Car dataset, $|\setT| = 12112$}
         \label{fig: car}
     \end{subfigure}
        \caption{Experiments on real datasets. \eqref{fig: mnist 1v7}--\eqref{fig: mnist 10} ARI on MNIST; \eqref{fig: car} ARI similarity matrix comparing the clusters obtained by the different methods on car (darker means more agreement).
        }
        \label{fig: real cluster}
\end{figure}



\section{Conclusion}

It is generally believed that a large number of passive comparisons is necessary in comparison based learning. Existing results on clustering require at least $\bigOmega{n^3}$ passive comparisons in the worst-case or under a planted framework. We show that, in fact,  $\bigOmega{n(\ln n)^2}$ passive comparisons suffice for accurately recovering planted clusters.
This number of comparisons is near-optimal, and almost matches the number of active comparisons typically needed for learning. Our theoretical findings are based on two simple approaches for constructing pairwise similarity matrices from passive comparisons. 
While we studied the merits of \ref{eqn_adis3} and \ref{eqn_adis4} in the context of clustering, they could be used for other problems such as semi-supervised learning, data embedding, or classification.

\newpage
\section*{Broader Impact}

This work primarily has applications in the fields of psychophysics and crowdsourcing, and more generally, in learning from human responses.
Such data and learning problems could be affected by implicit biases in human responses. However, this latter issue is beyond the scope of this work and, thus, was not formally analysed.

\begin{ack}
The work of DG is partly supported by the Baden-W{\"u}rttemberg Stiftung through the BW Eliteprogramm for postdocs. The work of MP has been supported by the ACADEMICS grant of the IDEXLYON, project of the Université de Lyon, PIA operated by ANR-16-IDEX-0005.
\end{ack}

\bibliographystyle{plainnat}
\bibliography{bibliography}

\appendix

\section{Existing comparison based similarities / kernel functions}
\label{app_kernels}

The literature on ordinal embedding from triplet comparisons is extensive \citep{jamieson2011low,arias2017some}.
In contrast, the idea of directly constructing similarity or kernel matrices from the comparisons, without embedding the data in an Euclidean space, is rather new.
Such an approach is known to be significantly faster than embedding methods, and provides similar or sometimes better performances in certain learning tasks. 
To the best of our knowledge, there are only two works that learn kernel functions from comparisons \citep{kleindessner2017kernel,ghoshdastidar2019foundations}, while the works of \citet{jain2016finite} and \citet{mason2017learning} estimate a Gram (or kernel) matrix from the triplets, which is then further used for data embedding.
In this section, we describe the aforementioned approaches for constructing pairwise similarities from comparisons.
Through this discussion, we illustrate the fundamental difference between the proposed additive similarities, \ref{eqn_adis3} and \ref{eqn_adis4}, and the existing kernels that are of multiplicative nature \citep{kleindessner2017kernel,ghoshdastidar2019foundations}.

Kernels from ordinal data were introduced by \citet{kleindessner2017kernel}, who proposed two kernel functions (named $k_1$ and $k_2$) based on observed triplets.
The kernels originated from the notion of Kendall's $\tau$ correlation between two rankings, and $k_1$ was empirically observed to perform slightly better.
We mention this kernel function, which we refer to as a multiplicative triplet kernel \eqref{eqn_3k}.
For any distinct $i,j\in[n]$, the \ref{eqn_3k} similarity is computed as
\begin{align}
S_{ij} &= \frac{\sum\limits_{r<s} \big(\ind{(i,r,s)\in\setT} - \ind{(i,s,r)\in\setT}\big)\big(\ind{(j,r,s)\in\setT} - \ind{(j,s,r)\in\setT}\big)}{\sqrt{|\{(\ell,r,s) \in \setT ~:~ \ell = i \}|} \sqrt{|\{(\ell,r,s) \in \setT ~:~ \ell = j\}|}}
\label{eqn_3k}\tag{MulK\mbox{-}3}
\end{align}
where $\setT$ is the set of observed triplets. 
Note that this kernel does not consider comparisons involving $w_{ij}$ but, instead, uses multiplicative terms indicating how $i$ and $j$ behave with respect to every pair $r,s$.
For uniform sampling with rate $p \gg \frac{\ln n}{n^2}$, the denominators in \ref{eqn_3k} are approximately $p\binom{n}{2}$ for every $i\neq j$.
Hence, it suffices to focus only on the numerator.
\citet{ghoshdastidar2019foundations} proposed a kernel similar to \ref{eqn_3k} for the case of quadruplets, which is referred to as multiplicative quadruplet kernel \eqref{eqn_4k}.  For $i\neq j$, it is given by
\begin{align}
S_{ij} &= \sum\limits_{\ell \neq i,j} \sum\limits_{r<s}  \big(\ind{(i,\ell,r,s)\in\setQ} - \ind{(r,s,i,\ell)\in\setQ}\big) \big(\ind{(j,\ell,r,s)\in\setQ} - \ind{(r,s,j,\ell)\in\setQ}\big).
\label{eqn_4k}\tag{MulK\mbox{-}4}
\end{align}

\citet{ghoshdastidar2019foundations} studied \ref{eqn_4k} in the context of hierarchical clustering, and showed that it requires $\bigO{n^{3.5}\ln n}$ passive quadruplet comparisons to exactly recover a planted hierarchical structure in the data.
Combining their concentration results with Proposition \ref{prop_sdp} shows that the same number of passive quadruplets suffices to recover the planted clusters considered in this work.
Note that both \ref{eqn_3k} and \ref{eqn_4k} kernel functions have a multiplicative nature since each entry is an aggregate of products. This is essential for their positive semi-definite property.
In contrast, the proposed \ref{eqn_adis3} and \ref{eqn_adis4} similarities simply aggregate comparisons involving the pairwise similarity $w_{ij}$, and hence, are not positive semi-definite kernels.

We also mention the work on fast ordinal triplet embedding (FORTE) \citep{mason2017learning}, which learns a metric from the given triplet comparisons.
One can easily adapt the formulation to that of learning a kernel matrix $K \in \nsetR^{n\times n}$ from triplets.
Consider the squared distance in the corresponding reproducing kernel Hilbert space (RKHS), $d_K^2(i,j) = K_{ii} - 2K_{ij} + K_{jj}$. 
Assuming that the triplets adhere to the distance relation in the  RKHS, it is easy to see that when a comparison of $t = \{i,r,s\}$ with $i$ as pivot is available, then
\begin{align*}
 y_t := \ind{(i,r,s)\in\setT} - \ind{(i,s,r)\in\setT}
 &= \textup{sign}\left(d_K^2(i,r) - d_K^2(i,s)\right)
 \\&= \textup{sign}\left(K_{rr} - 2K_{ir} - K_{ss} + 2K_{is} \right),
\end{align*}
which is the sign of a linear map of $K$, which we can denote as
sign$(\langle M_t,K\rangle)$ for some $M_t \in \nsetR^{n\times n}$. 
One can learn the optimal kernel matrix, that satisfies most triplet comparisons, by minimising the empirical loss $\displaystyle \frac{1}{|\setT|}\sum_{t\in \setT} \ell (y_t \langle M_t,K\rangle)$ with positive definiteness constraints for $K$, where $\ell$ is a loss function (log loss is suggested by \citet{jain2016finite}).

\section{Proof of Proposition \ref{prop_sdp}}
\label{app_prop_sdp}

In this section, we first provide a proof of Proposition \ref{prop_sdp} which is split into two parts: the proof of optimality of $X^*$, and the proof of uniqueness of the optimal solution.
In addition, we provide a derivation for the claim that $X^*$ is the unique optimal solution for \ref{eqn_sdp} when $S=\St$ and $0 < \lambda < n_{\min}\Delta_1$. The derivation, given at the end of the section, follows from simplifying some of the computations in the proof of Proposition \ref{prop_sdp}.

\subsection{Optimality of $X^*$ when $S$ is close to $\St$}

The proof is adapted from \citet{yan2018provable}. 
We first state the Karush-Kahn-Tucker (KKT) conditions for \ref{eqn_sdp}.
Let $\Gamma,\Lambda\in\nsetR^{n\times n}$ be the Lagrange parameters for  the non-negativity constraint $(X\geq0)$ and the positive semi-definiteness constraint $(X \succeq 0)$, respectively.
Let $\alpha\in\nsetR^n$ be the Lagrange parameter for the row sum constraints.
The tuple $(X,\Lambda,\Gamma,\alpha)$ is a primal-dual optimal solution for \ref{eqn_sdp}  if and only if it satisfies the following KKT conditions:
\begin{align*}
&\text{Stationarity :}
&& S - \lambda I + \Lambda + \Gamma - \1\alpha^T - \alpha\1^T  = 0
\\
&\text{Primal feasibility :}
&&X\geq0 \quad; \quad X\succeq0 \quad;\quad X\1 = \1
\\
&\text{Dual feasibility :}
&&\Lambda \succeq 0 \quad;\quad \Gamma\geq0
\\
&\text{Complementary slackness :}
&&\langle\Lambda,X\rangle = 0 \quad;\quad \Gamma_{ij} X_{ij} = 0 ~\forall~i,j
\end{align*}
where we use $\langle A, B\rangle$ to denote $\tr{AB}$ for symmetric matrices $A,B$. The above derivation is straightforward.
The term $\1\alpha^T + \alpha\1^T$ in the stationarity condition arises due to the symmetry of $X$, that is, since row-sum and column-sum are identical.
We construct a primal-dual witness to show that $X^*$ is the optimal solution of \ref{eqn_sdp} under the stated conditions on $\lambda$.
We use the following notations. For any vector $u\in\nsetR^n$ and $\setC\subset [n]$, we let $u_{\setC}\in\nsetR^{|\setC|}$ be the projection of $u$ on the indices contained in $\setC$. Similarly, for a matrix $A\in\nsetR^{n\times n}$, $A_{\setC\setC'}$ is the sub-matrix corresponding to row indices in $\setC$ and column indices in $\setC'$. We also define $\1_{m}$ and $I_{m}$ the constant vector of ones and the identity matrix of size $m$, respectively.
We use $\setC_1,\ldots,\setC_k$ to denote the planted clusters of size $n_1,\ldots,n_k$.
We consider the following primal-dual construction that is similar to~\citet{yan2018provable}, where $X = X^*$. 
For every $j,\ell\in\{1,\ldots, k\}$ and $\ell\neq j$, we define
\begin{align}
\alpha: \quad & \hskip4.3ex
\alpha_\j = \frac{1}{n_j} S_\jj \1_\nj - \left(\frac{\lambda}{2n_j} + \frac{1}{2n_j^2} \1_\nj^TS_\jj \1_\nj\right) \1_\nj
\label{eqn_pf_sdp1}
\\
\Lambda : \quad &\left\{
\begin{aligned}
\Lambda_\jj &= -S_\jj +  \alpha_\j\1_\nj^T + \1_\nj\alpha_\j^T + \lambda I_\nj
\\
\Lambda_\jl &= - \left(I_\nj -\frac{1}{n_j}\1_\nj\1_\nj^T\right) S_\jl \left(I_\nl -\frac{1}{\nl}\1_\nl\1_\nl^T\right)
\end{aligned} \right.
\label{eqn_pf_sdp2}
\\
\Gamma : \quad &\left\{
\begin{aligned}
\Gamma_\jj &= 0 \\
\Gamma_\jl &=  - S_\jl - \Lambda_\jl +  \alpha_\j\1_\nl^T + \1_\nj\alpha_\l^T \;.
\end{aligned} \right.
\label{eqn_pf_sdp3}
\end{align}  

The proof of Proposition \ref{prop_sdp} is based on verifying the KKT conditions for the above $\Lambda, \Gamma, \alpha$ and $X=X^*$.
To this end, note that $X^*_\jj = \frac{1}{n_j} \1_\nj\1_\nj^T$ and $X^*_\jl = 0$ for $\ell\neq j$.
Primal feasibility is obviously satisfied by $X^*$, and it is easy to see that the choice of $\Lambda_\jj$ and $\Gamma_\jl$ ensures that stationarity holds.
Hence, we only need to verify dual feasibility and complementary slackness.

The complementary slackness condition for $\Gamma$ holds since $\Gamma_\jj =0$ and $X^*_\jl =0$ for $j\neq \ell$.
To verify $\langle \Lambda, X^*\rangle = 0$, observe that
\begin{align*}
\langle \Lambda, X^*\rangle = \sum_{j,\ell} \langle \Lambda_\jl, X_\jl^*\rangle
= \sum_j \langle \Lambda_\jj, X_\jj^*\rangle
&= \sum_j \frac{1}{n_j} \1_\nj^T \Lambda_\jj \1_\nj
\\&= \sum_j - \frac{1}{n_j} \1_\nj^T S_\jj \1_\nj + 2\1_\nj^T\alpha_\j + \lambda,
\end{align*}
where the last step follows by substituting $\Lambda_\jj$ from \eqref{eqn_pf_sdp2} and noting that $\1_\nj^T\1_\nj = n_j$.
Substituting the value of $\alpha_\j$ above shows that each term in the sum is zero, and hence, $\langle \Lambda, X^*\rangle=0$.

We now verify the dual feasibility and first prove that $\Gamma \geq0$, in particular, $\Gamma_\jl\geq0$ for $j\neq \ell$.
We substitute $\Lambda_\jl$ and $\alpha_\j$ in \eqref{eqn_pf_sdp3} to obtain
\begin{align*}
\Gamma_\jl &= -S_\jl  +  \left(I_\nj -\frac{1}{n_j}\1_\nj\1_\nj^T\right) S_\jl \left(I_\nl -\frac{1}{\nl}\1_\nl\1_\nl^T\right) 
\\&\qquad\qquad~ + \frac{1}{n_j} S_\jj \1_\nj\1_\nl^T - \left( \frac{\lambda}{2n_j} +\frac{1}{2n_j^2} \1_\nj^TS_\jj \1_\nj\right) \1_\nj\1_\nl^T
\\&\qquad\qquad~ + \frac{1}{\nl} \1_\nj\1_\nl^T S_{\setC_\ell\setC_\ell} - \left( \frac{\lambda}{2\nl} +\frac{1}{2n_\ell^2} \1_\nl^TS_{\setC_\ell\setC_\ell} \1_\nl\right) \1_\nj\1_\nl^T
\\
&= -\frac{1}{\nj} \1_\nj \1_\nj^TS_\jl - \frac{1}{\nl} S_\jl\1_\nl \1_\nl^T + \frac{1}{n_j} S_\jj \1_\nj\1_\nl^T +  \frac{1}{\nl} \1_\nj\1_\nl^T S_{\setC_\ell\setC_\ell}
\\&\qquad\qquad~ + \left(\frac{\1_\nj^T S_\jl \1_\nl}{\nj\nl} -  \frac{\lambda}{2n_j} - \frac{\1_\nj^TS_\jj \1_\nj}{2n_j^2}  -  \frac{\lambda}{2\nl} - \frac{ \1_\nl^TS_{\setC_\ell\setC_\ell} \1_\nl}{2n_\ell^2}\right) \1_\nj\1_\nl^T \;.
\end{align*}

Consider $i\in\setC_j$ and $r\in\setC_\ell$. From above, we can compute $\Gamma_{ir}$ as
\begin{align*}
\Gamma_{ir} &= -\frac{1}{\nj} \1_\nj^TS_{\setC_j r} - \frac{1}{\nl} S_{i\setC_\ell}\1_\nl + \frac{1}{n_j} S_{i\setC_j} \1_\nj +  \frac{1}{\nl} \1_\nl^T S_{\setC_\ell r}
\\&\qquad\qquad~  + \frac{\1_\nj^T S_\jl \1_\nl}{\nj\nl} - \frac{\1_\nj^TS_\jj \1_\nj}{2n_j^2}  - \frac{ \1_\nl^TS_{\setC_\ell\setC_\ell} \1_\nl}{2n_\ell^2} -  \frac{\lambda}{2n_j} -  \frac{\lambda}{2n_\ell} 
\\&= - \frac{1}{\nj} \sum_{i'\in\setC_j} S_{i'r}  - \frac{1}{\nl} \sum_{r'\in\setC_\ell} S_{ir'} +  \frac{1}{\nj} \sum_{i'\in\setC_j} S_{ii'} +   \frac{1}{\nl} \sum_{r'\in\setC_\ell} S_{rr'}
\\&\qquad\qquad~   + \frac{1}{\nj\nl} \sum_{i'\in\setC_j,r'\in\setC_\ell} \hskip-2ex S_{i'r'} - \frac{1}{2n_j^2} \sum_{i,i'\in\setC_j} S_{ii'}  - \frac{1}{2n_\ell^2}  \sum_{r,r'\in\setC_\ell} S_{rr'} -  \frac{\lambda}{2n_j} -  \frac{\lambda}{2n_\ell} \;.
\end{align*}
Our goal is to derive a lower bound for $\Gamma_{ir}$ and show that, for suitable values of $\lambda$, $\Gamma_{ir}\geq0$ for all $i\in\setC_j,r\in\setC_\ell$.
We bound each of the terms from below. For the last two terms involving $\lambda$, we note that both terms are at least $-\frac{\lambda}{2n_{\min}}$, where $n_{\min} = \min_\ell n_\ell$.
For each of the other terms, we rewrite the summations in terms of the ideal similarity matrix $\St$ and bound the deviation in terms of $\Delta_2 = \max\limits_{i\in [n]} \max\limits_{\ell \in [k]} \left| \frac{1}{n_\ell} \sum\limits_{r \in \setC_\ell} \left(S_{ir} - \St_{ir}\right) \right|$.
For the first term, we have
\begin{align*}
 - \frac{1}{\nj} \sum_{i'\in\setC_j} S_{i'r} &=  - \frac{1}{\nj} \sum_{i'\in\setC_j} \St_{i'r}  - \frac{1}{\nj} \sum_{i'\in\setC_j} \left(S_{i'r} - \St_{i'r} \right) 
 \\&= - \Sigma_{j\ell} - \frac{1}{\nj} \sum_{i'\in\setC_j} \left(S_{i'r} - \St_{i'r} \right) 
 \\&\geq  - \Sigma_{j\ell} - \Delta_2.
 \end{align*}
For the second inequality, we use the structure of $\St$ to note that $\St_{ir} = \Sigma_{j\ell}$ for every $i\in\setC_i,r\in\setC_\ell$, and finally the deviation term is bounded by $\Delta_2$.
Similarly, one can bound the second, third and fourth terms from below by $( - \Sigma_{j\ell} - \Delta_2)$, $(\Sigma_{jj} - \Delta_2)$ and  $(\Sigma_{\ell\ell} - \Delta_2)$, respectively.
For the fifth term, we write
\begin{align*}
  \frac{1}{\nj\nl} \sum_{i'\in\setC_j,r'\in\setC_\ell} \hskip-2ex S_{i'r'} &=   \frac{1}{\nj\nl} \sum_{i'\in\setC_j,r'\in\setC_\ell} \hskip-2ex \St_{i'r'} +  \frac{1}{\nj\nl} \sum_{i'\in\setC_j,r'\in\setC_\ell} \hskip-2ex \left(S_{i'r'} - \St_{i'r'} \right)
 \\&= \Sigma_{j\ell} + \frac{1}{\nj} \sum_{i'\in\setC_j} \left( \frac{1}{\nl} \sum_{r'\in\setC_\ell} \left(S_{i'r} - \St_{i'r} \right) \right)
 \\&\geq \Sigma_{j\ell} - \Delta_2,
 \end{align*}
since each term in the outer summation is at least $-\Delta_2$.
Similarly, one can bound the sixth and seventh terms from below by $\frac12(\Sigma_{jj} - \Delta_2)$ and  $\frac12(\Sigma_{\ell\ell} - \Delta_2)$, respectively. 
Combining the above lower bounds, we have
\begin{align*}
\Gamma_{ir} ~\geq~ \frac12 \Sigma_{jj} + \frac12\Sigma_{\ell\ell} - \Sigma_{j\ell} - 6\Delta_2 - \frac{\lambda}{n_{\min}}
~\geq~ (\Delta_1 - 6\Delta_2) - \frac{\lambda}{n_{\min}} \;,
\end{align*}
where we recall that $\Delta_1 = \min\limits_{\ell\neq \ell'} \left(\frac{\Sigma_{\ell\ell} + \Sigma_{\ell'\ell'}}{2} - \Sigma_{\ell\ell'}\right)$.
Hence, for $\lambda \leq n_{\min} (\Delta_1 - 6\Delta_2)$, as stated in Proposition \ref{prop_sdp}, $\Gamma_{ir} \geq 0$, and more generally, $\Gamma$ is non-negative. 

We finally derive the positive semi-definiteness of $\Lambda$.
Define the vectors $u_1,\ldots,u_k \in \nsetR^n$ such that $(u_\ell)_i = 1$ if $i\in\setC_\ell$ and 0 otherwise.
We first claim that $u_1,\ldots,u_k$ lie in the null space of $\Lambda$.
To verify this, we compute the $\setC_j$-th block of $\Lambda u_\ell$. For $j\neq \ell$, 
\begin{align*}
(\Lambda u_\ell)_\j = \Lambda_\jl \1_\nl =  - \left(I_\nj -\frac{1}{n_j}\1_\nj\1_\nj^T\right) S_\jl \left(I_\nl -\frac{1}{\nl}\1_\nl\1_\nl^T\right) \1_\nl = 0,
\end{align*}
whereas for $j=\ell$, we have from \eqref{eqn_pf_sdp1} and \eqref{eqn_pf_sdp2},
\begin{align*}
(\Lambda u_\ell)_\l &= \Lambda_{\setC_\ell\setC_\ell} \1_\nl 
\\&=   -S_{\setC_\ell\setC_\ell} \1_\nl +  n_\ell \alpha_\l + \1_\nl\alpha_\l^T\1_\nl + \lambda \1_\nl
\\&= - S_{\setC_\ell\setC_\ell} \1_\nl +  S_{\setC_\ell\setC_\ell} \1_\nl - 2\left(\frac{\lambda}{2} + \frac{\1_\nl^T S_{\setC_\ell\setC_\ell} \1_\nl}{2\nl} \right)\1_\nl + \1_\nl \frac{\1_\nl^T S_{\setC_\ell\setC_\ell} \1_\nl}{\nl} + \lambda \1_\nl
\\&=0.
\end{align*}
Thus $\Lambda u_\ell = 0$ for $\ell = 1,\ldots,k$, and to prove that $\Lambda \succeq 0$, it suffices to show that $u^T\Lambda u \geq 0$ for all $u\in\nsetR^n$ that are orthogonal to $u_1,\ldots,u_k$. 
In other words, we consider only $u$ such that $u_\l^T \1_\nl = 0$ for every $\ell$.
For such a vector $u$, we have
\begin{align*}
u^T\Lambda u = \sum_{j,\ell=1}^k u_\j^T \Lambda_\jl u_\l
&= \sum_{j=1}^k u_\j^T \Lambda_\jj u_\j + \sum_{j\neq\ell} u_\j^T \Lambda_\jl u_\l
\\&= \sum_{j=1}^k u_\j^T (-S_\jj + \lambda I_\nj) u_\j - \sum_{j\neq\ell} u_\j^T S_\jl u_\l
\\&= \sum_{j=1}^k \lambda u_\j^T u_\j - \sum_{j,\ell} u_\j^T S_\jl u_\l
\\&= \lambda \Vert u\Vert^2 - u^TSu,
\end{align*} 
where $\Vert u\Vert$ is the Euclidean norm. 
The third equality follows from \eqref{eqn_pf_sdp2} and $u_\l^T \1_\nl = 0$ for every $\ell$.
In addition to above, recall that $\St = Z\Sigma Z^T$, where $Z = [u_1 \ldots u_k]$.
Hence, for $u$ orthogonal to $u_1,\ldots,u_k$, we have $u^T\St u = 0$, which, combined with above, gives 
\begin{align*}
u^T\Lambda u &=  \lambda \Vert u\Vert^2 - u^TSu
\\&= \lambda \Vert u\Vert^2 - u^T\left(S-\St\right)u
\\&\geq \left( \lambda - \left\Vert S-\St\right\Vert_2\right) \Vert u\Vert^2
\\&> 0
\end{align*}
for all $u$ if $\lambda > \left\Vert S-\St\right\Vert_2$, which is the condition stated in Proposition \ref{prop_sdp}.
Thus, for the specified range of $\lambda$, the KKT conditions are satisfied and $X^*$ is the optimal solution for \ref{eqn_sdp}.

\subsection{Uniqueness of the optimal solution $X^*$} 

The uniqueness of the solution can be shown by proving that any other optimal solution $X'$  for \ref{eqn_sdp} must satisfy $X'=X^*$.
This is shown in two steps. First, we show that any optimal solution $X'$ must have the same block structure as $X^*$ and $X^* - X' \succeq 0$. We use this fact to show that the objective value for $X^*$ is strictly greater than that for any such $X'$.

Note that the previously constructed Lagrange parameters in~\eqref{eqn_pf_sdp1}--\eqref{eqn_pf_sdp3} need not correspond to the optimal solution associated with $X'$. 
However, for the previously defined $\alpha,\Lambda,\Gamma$, we can still use the condition for stationarity to write
\begin{align*}
\langle \Lambda + \Gamma, X'\rangle
&=\langle - S + \1 \alpha^T + \alpha \1^T + \lambda I, X' \rangle
\\&= - \langle S,X'\rangle + \sum_{i,j=1}^n (\alpha_i+\alpha_j) X'_{ij} + \lambda \tr{X}
\\&= - \tr{SX'} + \lambda\tr{X'} + 2\sum_{i=1}^n \alpha_i 
\end{align*}
where the simplification happens noting that $X'$ is primal feasible and hence $\sum_j X_{ij} = 1$.
Due to optimality of $X'$ and $X^*$, we have $\tr{SX'} - \lambda\tr{X'} = \tr{SX^*} - \lambda\tr{X^*}$, and so,
\begin{align*}
\langle \Lambda + \Gamma, X'\rangle
&= - \tr{SX^*} + \lambda \tr{X^*} + 2\1^T\alpha
\\&= \lambda + \sum_{j=1}^k \left(- \frac{\1_\nj^T S_\jj \1_\nj}{\nj} + 2\1_\nj^T \alpha_\j\right) 
= 0,
\end{align*}
where the final step follows by substituting $\alpha_\j$ from \eqref{eqn_pf_sdp1}. 
From above, we argue that both $\langle \Lambda, X'\rangle$ and $\langle \Gamma,X'\rangle$ are zero.
To verify this, note that $\Gamma$ and $X'$ are both non-negative and hence, $\langle \Gamma,X'\rangle \geq0$.
On the other hand, from the definition of Frobenius (or Hilbert-Schmidt) norm, we have $\langle \Lambda, X'\rangle = \left\Vert \Lambda^{1/2}X'^{1/2} \right\Vert_F^2 \geq 0$, where the matrices square roots exist since $\Lambda,X'$ are both positive semi-definite.
Since both inner products, $\langle \Lambda, X'\rangle$ and $\langle \Gamma, X'\rangle$, are non-negative and yet their sum is zero, we can conclude that each of them equals zero.

Note that $\langle \Lambda, X'\rangle = \left\Vert \Lambda^{1/2}X'^{1/2} \right\Vert_F^2 = 0$ implies $\Lambda X' = 0$, or the range space of $X'$ lies in the null space of $\Lambda$.
Recall, from the proof of positive semi-definiteness of $\Lambda$, that, for $\lambda > \Vert S- \St\Vert_2$, the null space of $\Lambda$ is exactly spanned by $Z = [u_1 \ldots u_k]$.
Thus, the range space of $X'$ is spanned by the columns of $Z$, or in other words $X' = ZAZ^T$ for some $A\in\nsetR^{k\times k}$ that is symmetric, non-negative, and positive semi-definite (to ensure that $X'$ is primal feasible), and $\sum_j A_{ij} n_j = 1$ (to satisfy the row sum constraint).
Recall that $X^* = ZN^{-1}Z^T$ where $N=\text{diag}(n_1,\ldots,n_k)$.
Thus, $X'$ has the same block structure as $X^*$. However, this result does not imply that we can recover $k$ planted clusters from $X'$ since it is possible that $A$ has less than $k$ distinct rows.

We now argue that $X^* - X'$ must be positive semi-definite, a property that we use later. To see this, note that
\[
X^* - X'= ZN^{-1/2} \left(I_k - N^{1/2} A N^{1/2} \right) N^{-1/2} Z^T,
\]
where $ZN^{-1/2}$ is a matrix with orthonormal columns. Hence, to prove that $X^* -X' \succeq 0$, it suffices to show that 
$I_k - N^{1/2} A N^{1/2} \succeq 0$ or, equivalently, that the largest eigenvalue of $N^{1/2} A N^{1/2}$ is smaller than 1.
This can verified as
\begin{align*}
\left\Vert N^{1/2} A N^{1/2} \right\Vert_2 = \max_{u~:~\Vert u\Vert = 1} u^T N^{1/2} A N^{1/2} u
= \max_{u~:~\Vert u\Vert = 1} \sum_{i,j=1}^k A_{ij} \sqrt{n_i n_j} u_i u_j.
\end{align*}
From the AM-GM inequality, we have $\sqrt{n_i n_j} u_i u_j \leq \frac12\left(n_iu_j^2 + n_j u_i^2\right)$. Hence,
\begin{align*}
\left\Vert N^{1/2} A N^{1/2} \right\Vert_2 \leq \max_{u~:~\Vert u\Vert = 1} \frac12 \sum_{i,j=1}^k A_{ij}\left(n_iu_j^2 + n_j u_i^2\right) = \sum_{i=1}^k u_i^2 =1,
\end{align*}
where we use the fact that $\sum_j A_{ij} n_j = \sum_i A_{ij} n_i = 1$.
From this discussion, we have $X^*-X' \succeq 0$.
We now claim that 
\begin{equation}
\label{eqn_pf_sdp4}
\left| \tr{(S-\St)(X^* - X)}\right|  \leq \left\Vert S-\St\right\Vert_2 \tr{X^*-X'},
\end{equation}
which follows from von Neumann's trace inequality and the fact that $X^*-X'$ is positive semi-definite.

We now prove that for any $X' = ZAZ^T \neq X^*$, with $A$ satisfying the above mentioned conditions, and for $\Vert S-\St\Vert_2 <\lambda < \frac12 \Delta_1 n_{\min}$,
\begin{equation}
\label{eqn_pf_sdp5}
\tr{SX^*} - \lambda \tr{X^*} > \tr{SX'} - \lambda \tr{X'},
\end{equation}
which shows that $X^*$ is the unique optimal solution.
We compute
\begin{align*}
\tr{SX^*} - &\lambda \tr{X^*} - \tr{SX'} + \lambda \tr{X'}
\\&= \tr{S(X^*-X')} - \lambda \tr{X^* - X'}
\\&= \tr{\St(X^*-X')} + \tr{(S-\St)(X^*-X')} - \lambda \tr{X^* - X'}
\\&> \tr{\St(X^*-X')} -  \left\Vert S-\St\right\Vert_2 \tr{X^*-X'} - \lambda \tr{X^* - X'}
\\&> \tr{\St(X^*-X')} - n_{\min} \Delta_1 \tr{X^* - X'}.
\end{align*}

In the last step, we use $\left\Vert S-\St\right\Vert_2 + \lambda < 2\lambda < n_{\min}\Delta_1$.
We later prove that
\begin{equation}
\label{eqn_pf_sdp6}
\tr{\St(X^*-X')} \geq \sum_{\ell=1}^k \nl(1-A_{\ell\ell} \nl) \Delta_1 \geq  n_{\min} \Delta_1 \tr{X^* - X'}.
\end{equation}
Using \eqref{eqn_pf_sdp6} in the previous derivation proves \eqref{eqn_pf_sdp5} or the fact that $X^*$ is the unique optimal solution, provided that $\tr{X^*-X'} >0$ for all $X' \neq X^*$.
Hence, we need to verify the strict positivity of the trace.
Assume that $\tr{X^*-X'} = 0$.
Due to the row sum  constraint for $X'$, we have $\sum_j A_{\ell j} n_j = 1$, which implies $A_{\ell \ell} \nl \leq 1$.
On the other hand $\tr{X'} = \tr{X^*} = k$ holds if $\sum_\ell A_{\ell\ell} \nl = k$, which is only possible if $A_{\ell \ell} \nl = 1$ for every $\ell$, and hence $A_{\ell j} = 0$ for $j\neq\ell$. Thus, $\tr{X^*-X'} = 0$ if and only if $X' = X^*$.
For every $X'\neq X^*$,  we have $\tr{X^*-X'} = \sum_\ell (1- A_{\ell\ell} \nl) > 0$.
We conclude the proof by proving \eqref{eqn_pf_sdp6}.
We compute
\begin{align*}
\tr{\St (X^* - X')} &= \tr{Z\Sigma Z^T Z(N^{-1}-A)Z^T}
\\&=  \tr{\Sigma Z^T Z(N^{-1}-A)Z^TZ}
\\&=  \tr{\Sigma N(N^{-1}-A)N}
\\&= \sum_{\ell=1}^k \Sigma_{\ell\ell} \nl(1-A_{\ell\ell}\nl) - \sum_{\ell=1}^k \sum_{j\neq \ell} A_{\ell j} \nj\nl \Sigma_{\ell j},
\end{align*}
where the third equality follows from the fact $Z^TZ = N$.
Recall from the definition of $\Delta_1$ that $\Sigma_{\ell j} \leq \frac12(\Sigma_{jj} + \Sigma_{\ell\ell}) - \Delta_1$. 
Using this, we can write
\begin{align*}
&\tr{\St (X^* - X')} 
\\&\geq \sum_{\ell=1}^k \Sigma_{\ell\ell} \nl(1-A_{\ell\ell}\nl) + \sum_{\ell=1}^k \sum_{j\neq \ell} A_{\ell j} \nj\nl \Delta_1 -  \frac12 \sum_{\ell=1}^k \sum_{j\neq \ell} A_{\ell j} \nj\nl (\Sigma_{\ell\ell} + \Sigma_{jj})
\\&= \sum_{\ell=1}^k \Sigma_{\ell\ell} \nl(1-A_{\ell\ell}\nl) + \sum_{\ell=1}^k \sum_{j\neq \ell} A_{\ell j} \nj\nl \Delta_1 -   \sum_{\ell=1}^k \sum_{j\neq \ell} A_{\ell j} \nj\nl \Sigma_{\ell\ell}
\\&=  \sum_{\ell=1}^k \Sigma_{\ell\ell} \nl(1-A_{\ell\ell}\nl) + \sum_{\ell=1}^k \nl \Delta_1 (1-A_{\ell\ell}\nl) -   \sum_{\ell=1}^k \nl \Sigma_{\ell\ell}  (1-A_{\ell\ell}\nl). 
\end{align*}
In the first equality, we exploit the symmetry of the third summation, while the second equality uses the row sum constraint to write $\sum_{j\neq \ell} A_{\ell j} \nj = 1 - A_{\ell\ell} \nl$.
Cancelling first and third terms, we get
\begin{align*}
\tr{\St (X^* - X')} &\geq \Delta_1 \sum_{\ell=1}^k n_\ell (1- A_{\ell \ell} \nl)
\\&\geq \Delta_1 n_{\min}  \sum_{\ell=1}^k (1- A_{\ell \ell} \nl) = n_{\min} \Delta_1 \tr{X^* - X'},
\end{align*}
which proves \eqref{eqn_pf_sdp6}, and completes the proof.

\subsection{Unique optimality of $X^*$ when $S=\St$}

We now prove that $X^*$ is the unique optimal solution when $S=\St = Z\Sigma Z^T$ and $0 < \lambda < n_{\min}\Delta_1$.
This claim does not immediately follow from Proposition \ref{prop_sdp}, but can be derived from the proof.

We first prove the optimality of $X^*$ in this case.
Recall, from the proof of Proposition \ref{prop_sdp}, that the claim hinges on showing that $\Gamma \geq 0$ and $\Lambda \succeq 0$.
From the previous proof, it suffices to show that $\Gamma_\jl \geq 0$ and $u^T\Lambda u \geq0$ for any $u$ that is orthogonal to the columns of $Z$.
To show that the latter holds, recall that
\[
u^T\Lambda u = \lambda \Vert u\Vert^2 - u^TSu.
\]
Since $S=\St = Z\Sigma Z^T$ and $Z^Tu =0$, we get $u^T\Lambda u = \lambda \Vert u\Vert^2 \geq0$, which in turn shows that $\Lambda \succeq0$ for all $\lambda>0$.
To verify the non-negativity of $\Gamma_\jl$, we observe that, in this case, it can be computed as
\begin{align*}
\Gamma_\jl 
&= -\frac{1}{\nj} \1_\nj \1_\nj^T\St_\jl - \frac{1}{\nl} \St_\jl\1_\nl \1_\nl^T + \frac{1}{n_j} \St_\jj \1_\nj\1_\nl^T +  \frac{1}{\nl} \1_\nj\1_\nl^T \St_{\setC_\ell\setC_\ell}
\\&\qquad\qquad~ + \left(\frac{\1_\nj^T \St_\jl \1_\nl}{\nj\nl} -  \frac{\lambda}{2n_j} - \frac{\1_\nj^T\St_\jj \1_\nj}{2n_j^2}  -  \frac{\lambda}{2\nl} - \frac{ \1_\nl^T\St_{\setC_\ell\setC_\ell} \1_\nl}{2n_\ell^2}\right) \1_\nj\1_\nl^T \;
\\
&= \left( - 2\Sigma_{j\ell} + \Sigma_{jj} + \Sigma_{\ell\ell} + \Sigma_{j\ell} - \frac{\lambda}{2}\left(\frac{1}{n_j} + \frac{1}{\nl}\right) - \frac{\Sigma_{jj} + \Sigma_{\ell\ell}}{2} \right)\1_\nj\1_\nl^T
\\
&\geq \left(\Delta_1 - \frac{\lambda}{n_{\min}}\right) \1_\nj\1_\nl^T
\end{align*}
So for $\lambda\leq n_{\min}\Delta_1$, $\Gamma_\jl \geq0$, and hence $\Gamma$ is non-negative.
Combining this with the previous proof of optimality, we derive that $X^*$ is an optimal solution in this case for $0< \lambda\leq n_{\min}\Delta_1$.

The proof of uniqueness is similar to the more general case in Proposition \ref{prop_sdp}. We use the previously derived claim that any optimal solution $X'$ must be of the form $X' = ZAZ^T$ for some $A\in\nsetR^{k\times k}$, and $X^* - X \succeq 0$. We have also previously shown that $\tr{\St(X^*-X')} \geq n_{\min}\Delta_1 \tr{X^*-X'}$.
Hence, we have
\[
\tr{\St X^*} - \lambda\tr{X^*} - \left(\tr{\St X'} - \lambda\tr{X'}\right)
\geq (n_{\min}\Delta_1 - \lambda)\tr{X^*-X'},
\]
which is strictly positive for $\lambda < n_{\min}\Delta_1$, and hence, $X^*$ is the unique optimal solution in this case.

\section{Proof of Theorem \ref{thm_adis}}
\label{app_thm_adis}

We prove the result for triplets and quadruplets in separate sections. 
While the proof structure is the same in both cases, the computations are quite different.
Before presenting the proofs, we list the key steps.

We first compute the expectation of the similarity matrix $S$ computed using \ref{eqn_adis3} or \ref{eqn_adis4}, and derive appropriate ideal matrices $\St$ in each case. In our proofs, $\St = \Eb[S]$, except differences in the diagonal entries since $S_{ii} = 0$ for all $i$.
From the block structure of $\St$, we can compute $\Delta_1$.

Subsequently, concentration inequalities are used to derive upper bounds on $\Vert S-\St\Vert_2$ and $\Delta_2$ in terms of the model parameters.
In this context, note that though the pairwise similarities $\{w_{ij} ~:~ i<j\}$ are independent, the entries of the matrix $S$ are highly dependent since each $w_{ij}$ appears in multiple entries of $S$.
Hence, to decouple such dependencies, we use a technique by \citet{janson2002infamous}, which considers the dependency graphs of the random variables and finds an equitable colouring to find independent sets of comparable sizes.
To the best of our knowledge, the present work is the first study which uses the equitable colouring approach of \citet{janson2002infamous} to derive spectral norm bounds. 
\citet{ghoshdastidar2019foundations} use this technique only to bound matrix entries.

Finally, we use concentration to show that for a sampling rate $p$ large enough, the number of comparisons  ($|\setQ|$ or $|\setT|$) is close to its expected value. Hence, we can replace the sampling rate $p$ in the previously derived bounds by the number of comparisons, leading to differences in constants only.

\textbf{Notation.} For the sake of simplicity, we will ignore absolute constants in the inequalities stated below, and use the notations $\lesssim$ and $\gtrsim$ to write inequalities that hold up to some multiplicative absolute constant.

\subsection{Quadruplet setting}

We first present the proof for the quadruplet setting using the aforementioned steps.

\textbf{Computation of $\Delta_1$.}
We first derive the expectation of the \ref{eqn_adis4} similarity matrix $S$, where for $i\neq j$,
\begin{equation}
\Eb[S_{ij}] = \sum_{r\neq s}  \Pb\big((i,j,r,s)\in\setQ\big) - \Pb\big((r,s,i,j)\in\setQ\big).
\label{eqn_expadis4}
\end{equation}

Note that the summation in \ref{eqn_adis4} is a sum over all distinct pairs $r,s$, noting that we do not count both $(s,r)$ and $(r,s)$ since they refer to the same comparison.
To compute the expectation of each term  in the summation, recall that the items belong to the planted clusters $\setC_1,\ldots,\setC_k$ and, for each item $i$, use $\psi_i \in [k]$ to denote the cluster index in which $i$ belongs, that is, $i \in \setC_{\psi_i}$. The expected values of the terms are given in Table \ref{tab_expadis4}.
\begin{table}[ht]
  \caption{Value of each term in the summation in \eqref{eqn_expadis4}, assuming $i\neq j$, $r\neq s$, $(i,j)\neq (r,s)$.}
  \label{tab_expadis4}
  \centering
  \begin{tabular}{cccc}
    \toprule
    Case     & $\Pb\big((i,j,r,s)\in\setQ\big)$   & $\Pb\big((r,s,i,j)\in\setQ\big)$ & Difference \\
    \midrule
    $\psi_i = \psi_j$; $\psi_r = \psi_s$ & $p/2$ & $p/2$ & 0 \\
    $\psi_i = \psi_j$; $\psi_r \neq \psi_s$ & $p(1+\epsilon\delta)/2$ & $p(1-\epsilon\delta)/2$ &$p\epsilon\delta$ \\
    $\psi_i \neq \psi_j$; $\psi_r = \psi_s$ & $p(1-\epsilon\delta)/2$ & $p(1+\epsilon\delta)/2$ &$-p\epsilon\delta$\\
    $\psi_i \neq \psi_j$; $\psi_r \neq \psi_s$ & $p/2$ & $p/2$ & 0 \\   
    \bottomrule
  \end{tabular}
\end{table}

We only explain the derivation of $\Pb\big((i,j,r,s)\in\setQ\big)$ for the case $\psi_i=\psi_j$ and $\psi_r\neq\psi_s$ as the other values are computed similarly. In this case,
\begin{align*}
\Pb\big((i,j,r,s)\in\setQ\big) &= \Pb\big((i,j,r,s)\in\setQ \big| (i,j),(r,s) \text{ compared}\big)\Pb\big((i,j),(r,s) \text{ compared}\big)
\\&= p \Pb\big((i,j,r,s)\in\setQ \big| (i,j),(r,s) \text{ compared}\big)
\\&= p\Big[  \Pb\big((i,j,r,s)\in\setQ \big| w_{ij} > w_{rs};~(i,j),(r,s) \text{ compared}\big) \Pb(w_{ij} > w_{rs})
\\&\qquad\quad + \Pb\big((i,j,r,s)\in\setQ \big| w_{ij} < w_{rs};~(i,j),(r,s) \text{ compared}\big) \Pb(w_{ij} < w_{rs}) \Big]
\\&= p \left[ \frac{(1+\epsilon)}{2}\frac{(1+\delta)}{2} + \frac{(1-\epsilon)}{2}\frac{(1-\delta)}{2} \right]
= p\frac{(1+\epsilon\delta)}{2} \;,
\end{align*}
where in each product, the term $\Pb(w_{ij}>w_{rs})$ is computed from \eqref{eqn_Fcondn}, and the other term, denoting flipped answers, follows from the quadruplet variant of \eqref{eqn_Errcond}.
Based on Table \ref{tab_expadis4}, we have for $i, j$ such that $\psi_i =  \psi_j$
\begin{align*}
\Eb[S_{ij}] &= \sum_{(r,s): \psi_r\neq\psi_s} \hskip-2ex p\epsilon\delta = p\epsilon\delta \sum_{\ell=1}^k \frac{\nl(n-\nl)}{2}
\end{align*} 
if $i\neq j$. For $i=j$, obviously $\Eb[S_{ij}] = 0$.  For $i, j$ such that $\psi_i\neq\psi_j$, we have
\begin{align*}
\Eb[S_{ij}] &= \sum_{(r,s): \psi_r=\psi_s}\hskip-2ex -p\epsilon\delta = -p\epsilon\delta \sum_{\ell=1}^k \binom{\nl}{2}.
\end{align*} 
Hence, we define the ideal similarity matrix as $\St_{ij} = \Eb[S_{ij}]$ for $i\neq j$, and $\St_{ii} = p\epsilon\delta \sum\limits_{\ell=1}^k \frac{\nl(n-\nl)}{2}$.
Observe that $\St = Z\Sigma Z^T$, where $Z \in \{0,1\}^{n\times k}$ is the assignment matrix for the planted clusters, and $\Sigma\in\nsetR^{k\times k}$ such that $\Sigma_{\ell\ell} = p\epsilon\delta \sum_{\ell} \frac{\nl(n-\nl)}{2}$ and $\Sigma_{\ell\ell'} = -p\epsilon\delta \sum_{\ell} \binom{\nl}{2}$ for $\ell\neq\ell'$.
Hence, in this case, we have 
\begin{equation}
\Delta_1 = p\epsilon\delta\binom{n}{2}.
\label{eqn_pf_adis4_Del1}
\end{equation}

\textbf{Preliminary computations and definitions for concentration.}
As noted earlier, $\St$ and $\Eb[S]$ are identical, except in the diagonal entries. 
Hence, we mainly have to obtain concentration of $f(S-\Eb[S])$, where $f$ is a non-negative scalar function.
In the case of $\Delta_2$, $f$ denotes the maximum partial row sum, whereas $f$ is the spectral norm in the bound for $\Vert S-\St\Vert_2$.
Define $\setW = \{w_{ij} : i<j\}$ as the collection of random pairwise similarities. We write
\[
S - \Eb[S] = \big( S - \Eb[S|\setW] \big) + \big( \Eb[S|\setW] - \Eb[S] \big),
\]
where the first difference accounts for randomness in sampling and crowd noise, while the second difference accounts for the inherent noise in $\setW$.
This helps in separately concentrating both terms, which have different dependence structures.
Formally, we perform the concentration of $f(S-\Eb[S])$ in the following way, assuming $f$ satisfies triangle inequality (which holds in the cases that we later consider).
\begin{align*}
\Pb\big( f(S-\Eb[S]) > t \big) 
&\leq  \Pb\big( f(S-\Eb[S|\setW]) + f(\Eb[S|\setW] - \Eb[S]) > t \big) 
\\&\leq  \Pb\big( f(S-\Eb[S|\setW])  > t/2 \big) +  \Pb\big( f(\Eb[S|\setW]-\Eb[S])  > t/2 \big)
\\&\leq \Eb_\setW \left[ \Pb_{\cdot|\setW}\big( f(S-\Eb[S|\setW]) > t/2 \big) \right] +  \Pb\big( f(\Eb[S|\setW]-\Eb[S])  > t/2 \big),
\end{align*}
where $\Pb_{\cdot|\setW}$ denotes the probability over sampling and crowd noise, but conditioned on $\setW$. In fact, we derive an uniform upper bound on the conditional probability, irrespective of $\setW$, and hence the expectation is trivially bounded.
To separately deal with the randomness in $\setW$ and the randomness due to sampling and crowd noise,
we write
\begin{align}
S_{ij} &= \sum_{r\neq s}  \big(\ind{(i,j,r,s)\in\setQ} - \ind{(r,s,i,j)\in\setQ}\big)
= \sum_{r<s} \xi_{ijrs} \big(\ind{w_{ij}>w_{rs}} - \ind{w_{ij}<w_{rs}}\big)
\label{eqn_adis4_formal}
\end{align}
where $\xi_{ijrs} \in \{-1,0,+1\}$ denotes whether the comparison between $(i,j)$ and $(r,s)$ is observed ($\xi_{ijrs}=0$ if not observed), and whether the crowd response was correct $(\xi_{ijrs}=+1)$ or flipped $(\xi_{ijrs}=-1)$.
Under our sampling and noise model,
\[
\Pb(\xi_{ijrs} = 0) = 1-p, \qquad   \Pb(\xi_{ijrs} = 1) =\frac{p(1+\epsilon)}{2}\;,  \qquad   \Pb(\xi_{ijrs} = -1) =\frac{p(1-\epsilon)}{2} 
\]
and so, $\Eb[\xi_{ijrs}] = p\epsilon$ and $\var{\xi_{ijrs}} \leq p$.
Note that the set  $\Xi=\{ \xi_{ijrs} ~:~ i<j, r<s, (i,j)<(r,s)\}$ is a collection of independent random variables. Here, $(i,j)<(r,s)$ denotes a lexicographic ordering of tuples since we do not care about the ordering between $(i,j)$ and $(r,s)$.


In addition, recall that $F_{in},F_{out}$ are continuous, and hence, with probability 1, any two pairwise similarities are distinct. 
Hence, we can write $\ind{w_{ij}>w_{rs}} - \ind{w_{ij}<w_{rs}} = 2\ind{w_{ij}>w_{rs}} - 1$.
It is noted that $\xi_{ijrs}$ is independent of $(2\ind{w_{ij}>w_{rs}} - 1)$, and furthermore, the latter variable is deterministic conditioned on $\setW$.
Based on this and using the notation of $\xi_{ijrs}$, we write
\begin{equation}
\begin{aligned}
&S_{ij} - \Eb[S_{ij}|\setW] = \sum_{r<s} B_{ijrs}, 
&&\text{where} \qquad
B_{ijrs} = (\xi_{ijrs} - p\epsilon)\big(2\ind{w_{ij}>w_{rs}} - 1\big)
\\
&\Eb[S_{ij}|\setW] - \Eb[S_{ij}] = \sum_{r<s} B'_{ijrs}, 
&&\text{where} \qquad
B'_{ijrs} = 2p\epsilon\big(\ind{w_{ij}>w_{rs}}- \Pb(w_{ij}>w_{rs}) \big).
\end{aligned}
\label{eqn_pf_adis4_B}
\end{equation}

We make the following observations about the collection of random variables $B_{ijrs},B'_{ijrs}$, which are crucial to the subsequent concentration results.
It is easy to see that $|B_{ijrs}| \leq 2$, $|B'_{ijrs}| \leq 2p\epsilon$ with probability 1, and $\Eb[B_{ijrs}] = \Eb[B'_{ijrs}] = 0$, $\var{B_{ijrs}} \leq p$ and $\var{B'_{ijrs}} \leq 4p^2\epsilon^2$.
Define the sets 
\begin{equation}
\begin{aligned}
\setB &= \{ B_{ijrs} ~:~ i<j, r<s, (i,j)\neq(r,s) \}, \\
\setB' &= \{ B'_{ijrs} ~:~ i<j, r<s, (i,j)\neq(r,s) \},\\
\setB_{i\ell} &= \{ B_{ijrs} ~:~ j\in\setC_\ell, j\neq i, r<s, (i,j)\neq(r,s) \} \qquad\text{for every } i\in[n],\ell\in[k], \\
\text{and} \qquad \setB'_{i\ell} &= \{ B'_{ijrs} ~:~ j\in\setC_\ell, j\neq i, r<s, (i,j)\neq(r,s) \} \qquad\text{for every } i\in[n],\ell\in[k].
\end{aligned}
\label{eqn_pf_adis4_setB}
\end{equation}
Each of $\setB$ and $\setB'$ have $\binom{n}{2}\left(\binom{n}{2}-1\right)$ random variables. 
$B_{ijrs} = - B_{rsij}$, but conditioned on $\setW$, $B_{ijrs}$ is independent of all other variables in $\setB$. 
Thus, a dependency graph on $\setB$, conditioned on $\setW$, has a maximum degree of 1.
On the other hand, $B'_{ijrs}$ depends on all the random variables of the form $B'_{ijr's'}, B'_{i'j'rs}, B'_{r's'ij}$ and $B'_{rsi'j'}$, and so, the dependence graph for $\setB'$ has degree smaller than $4\binom{n}{2} - 7$.
Similarly, $\setB_{i\ell}, \setB'_{i\ell}$ have at most $\nl\left(\binom{n}{2}-1\right)$ random variables.
While $\setB_{i\ell}$ has a dependency graph with degree at most 1, the dependency graph of $\setB'_{i\ell}$ has degree at most $\nl + \binom{n}{2} - 3$.
We now use the above discussion to derive upper bounds on $\Delta_2$ and $\Vert S-\St\Vert_2$.

\textbf{Upper bound for $\Delta_2$.}
To derive a bound on $\Delta_2$, we first note that
\begin{align*}
\Delta_2 \leq \max_{i\in[n]} \max_{\ell\in[k]} \left| \frac{1}{\nl} \sum_{j\in\setC_\ell} S_{ij} - \Eb[S_{ij}] \right| +  \frac{a_0}{n_{\min}}
\end{align*}
where $a_0 = \St_{ii}$ takes into account the fact that $\St$ and $\Eb[S]$ differ only in diagonal terms.
In the subsequent steps, we bound the first term.
For any $t>0$, the union bound leads to
\begin{align}
&\Pb\left( \max_{i\in[n]} \max_{\ell\in[k]} \left| \frac{1}{\nl} \sum_{j\in\setC_\ell} S_{ij} - \Eb[S_{ij}] \right| > t\right)
\nonumber
\\&\leq \sum_{i\in[n]} \sum_{\ell \in [k]} \Pb\left( \left| \sum_{j\in\setC_\ell} S_{ij} - \Eb[S_{ij}|\setW] + \Eb[S_{ij}|\setW] - \Eb[S_{ij}] \right| > \nl t\right).
\nonumber
\\&\leq \sum_{i\in[n]} \sum_{\ell \in [k]} \Eb_{\setW} \hskip-1ex\left[\Pb_{\cdot|\setW}\hskip-1ex\left( \left| \sum_{j\in\setC_\ell} S_{ij} - \Eb[S_{ij}|\setW] \right| > \frac{\nl t}{2}\right) \right] + \Pb\hskip-1ex\left( \left| \sum_{j\in\setC_\ell} \Eb[S_{ij}|\setW] - \Eb[S_{ij}]\right| > \nl t\right)
\nonumber
\\&\leq \sum_{i\in[n]} \sum_{\ell \in [k]} \Eb_{\setW} \left[\Pb_{\cdot|\setW}\left( \left| \sum_{j\in\setC_\ell} \sum_{r<s} B_{ijrs} \right| > \frac{\nl t}{2}\right) \right] + \Pb\left( \left| \sum_{j\in\setC_\ell} \sum_{r<s} B'_{ijrs}\right| > \frac{\nl t}{2}\right)
\label{eqn_pf_adis4_1}
\end{align}

For the probability conditioned on $\setW$, the summation involves terms in the set $\setB_{i\ell}$ in~\eqref{eqn_pf_adis4_setB}. 
The discussion on $\setB_{i\ell}$ shows that, conditioned on $\setW$, the summation is a sum of independent random variables $B_{ijrs}$ whose properties are stated after \eqref{eqn_pf_adis4_B}. 
Hence, we can apply Bernstein's inequality to bound the conditional probability as
\begin{align*}
\Pb_{\cdot|\setW}\left( \left| \sum_{j\in\setC_\ell} \sum_{r<s} B_{ijrs} \right| > \frac{\nl t}{2}\right)
&\leq 2\exp\left(-\frac{(\frac{\nl t}{2})^2}{2p\nl\binom{n}{2} + \frac23 2 \frac{\nl t}{2}}\right)
\\& \lesssim 2\exp\left(- \min\left\{ \frac{\nl t^2}{pn^2} \,, {\nl t} \right\} \right)
 \lesssim \frac{1}{n^3}
\end{align*}
for $t \gtrsim \displaystyle\max\left\{ \sqrt{\frac{p n^2 \ln n}{n_{\min}}} \,, \frac{\ln n}{n_{\min}} \right\}$.
Since the $\bigO{\frac{1}{n^3}}$ bound on the probability holds uniformly for all $\setW$, it also bounds the first term in \eqref{eqn_pf_adis4_1}.

For the second probability in \eqref{eqn_pf_adis4_1}, note that the $B'_{ijrs}$ in the summation are not independent, and we cannot directly apply Bernstein inequality. 
Hence, we apply the technique in \citet[Theorem 5]{janson2002infamous} which bounds the probability by partitioning the random variables in $\setB'_{i\ell}$ into independent sets. 
Since the dependency graph on $\setB'_{i\ell}$ has maximum degree $d = \nl + \binom{n}{2}-3$, we can obtain an equitable $(d+1)$-colouring, with each independent set of size $\lfloor |\setB'_{i\ell}|/(d+1) \rfloor$ or $\lceil |\setB'_{i\ell}|/(d+1) \rceil$, which are both smaller than $\nl$.
Denote the independent sets by $\setB'_{i\ell,(1)}, \ldots,\setB'_{i\ell,(d+1)}$, and we can apply Bernstein's inequality to bound the summation over each independent set. 
Hence, we  bound the second probability in \eqref{eqn_pf_adis4_1}, for every $i,\ell$, as
\begin{align*}
 \Pb&\left( \left| \sum_{j\in\setC_\ell} \sum_{r<s} B'_{ijrs} \right| > \frac{ \nl t}{2}\right)
 \\&\leq  \Pb\left( \max_{r \in \{1,\ldots, d+1\}} \left| \sum_{B'\in \setB'_{i\ell,(r)}} B' \right| > \frac{\nl t}{2(d+1)}\right)
 \\&\leq  \sum_{r=1}^{d+1} \Pb\left( \left| \sum_{B'\in \setB'_{i\ell,(r)}} B' \right| > \frac{\nl t}{2(d+1)}\right)
 & \text{(union bound)}
 \\&\leq  2(d+1) \exp\left( - \frac{(\frac{\nl t}{2(d+1)})^2}{2\sum\limits_{B'\in \setB'_{i\ell,(r)}} \var{B'} + \frac23 2p\epsilon \frac{\nl t}{2(d+1)}} \right)
 & \text{(Bernstein bound)}
 \\&\leq 2(d+1) \exp\left( - \frac{(\frac{\nl t}{d+1})^2}{8p^2\epsilon^2 \nl + \frac23 p\epsilon \frac{\nl t}{d+1}} \right)
\\&\lesssim n^2 \exp\left( - \min\left\{ \frac{\nl t^2}{p^2\epsilon^2 n^4}\,,  \frac{\nl t}{p\epsilon n^2} \right\} \right),
\end{align*}
which is $\bigO{\frac{1}{n^3}}$ for $t\gtrsim \displaystyle  p\epsilon n^2 \cdot  \max\left\{ \sqrt{\frac{\ln n}{n_{\min}}},\frac{\ln n}{n_{\min}} \right\}$.
The first term dominates since, under the condition on $\delta$, we have $n_{\min} \gtrsim \ln n$.
Thus, we conclude that, with probability $1-\frac{1}{4n}$,
\begin{align}
\Delta_2 &\leq \max_{i\in[n]} \max_{\ell\in[k]} \left| \frac{1}{\nl} \sum_{j\in\setC_\ell} S_{ij} - \Eb[S_{ij}] \right| +  \frac{a_0}{n_{\min}}
\nonumber
\\&\lesssim   \max\left\{ \sqrt{\frac{p n^2 \ln n}{n_{\min}}} \,, \frac{\ln n}{n_{\min}}\,, p\epsilon n^2  \sqrt{\frac{\ln n}{n_{\min}}} \,, \frac{p\epsilon \delta n^2}{n_{\min}} \right\} \;,
\label{eqn_pf_adis4_Del2}
\end{align}
 where the last term is obviously dominated by the third term.
 
\textbf{Upper bound for $\Vert S-\St\Vert_2$.}
Similar to the case of $\Delta_2$, we bound the spectral norm as
\[
\Vert S - \St \Vert_2 \leq \Vert S - \Eb[S|\setW] \Vert_2 + \Vert \Eb[S|\setW] - \Eb[S] \Vert_2 + \Vert \Eb[S] - \St \Vert_2,
\]
where the last term equals $a_0 = \St_{ii}$ since $\Eb[S] - \St$ is a diagonal matrix. 
For the first term, we derive a bound conditioned on $\setW$.
Recall from \eqref{eqn_pf_adis4_B}--\eqref{eqn_pf_adis4_setB} that, conditioned on $\setW$, the matrix $S-\Eb[S|\setW]$ comprises of variables in $\setB$, which has a dependence graph with degree 1.
We partition $\setB$ into two independent sets via equitable colouring, and write $S-\Eb[S|\setW] = A + A'$, where $A$ and $A'$ are the symmetric matrices corresponding to each of the independent sets.
We derive a spectral norm for each of $A$ and $A'$.
For this, we first claim that, conditioned on $\setW$,  
the event $\mathcal{E} = \left\{\max_{i,j} \big\{ |A_{ij}|, |A'_{ij}| \big\} \lesssim \max \left\{ \sqrt{pn^2 \ln n}, \ln n\right\}\right\}$ occurs with probability $1 - \bigO{\frac1n}$.
To see this, observe that $A_{ij}$ (or $A'_{ij}$) is a sum of at most $\binom{n}{2}$ independent random variables $B_{ijrs}$. 
By Bernstein inequality,
\[
\Pb_{\cdot|\setW} ( |A_{ij}| > \tau ) \leq 2\exp\left(-\frac{\tau^2}{2p\binom{n}{2} + \frac43\tau}\right)
\lesssim \exp\left(-\min\left\{\frac{\tau^2}{pn^2}, \tau\right\}\right)
\]
which is $\bigO{\frac{1}{n^3}}$ for $\tau \gtrsim \max \left\{ \sqrt{pn^2 \ln n}, \ln n\right\}$.
Applying the union bound gives $\Pb(\mathcal{E}^c) = \bigO{\frac1n}$.

Conditioned on $\setW$ and $\mathcal{E}$, the matrices $A,A'$ have independent zero mean entries, with each entry bounded by $ \bigO{\max \left\{ \sqrt{pn^2 \ln n}, \ln n\right\}}$.
Furthermore, from the variance of $B_{ijrs}$, we have $\max_i \sum_j \var{A_{ij}} < pn^3$, and same for $A'$.
Hence, by matrix Bernstein inequality \citep{tropp2012userfriendly}, 
\begin{align*}
\Pb_{\cdot| \setW,\mathcal{E}} \left(\Vert S - \Eb[S |\setW] \Vert_2 > t \right)
& \leq \Pb_{\cdot| \setW,\mathcal{E}} \left(\Vert A \Vert_2 > t/2 \right) + \Pb_{\cdot| \setW,\mathcal{E}} \left(\Vert A' \Vert_2 > t/2 \right)
\\&\leq 2n\exp\left(-\frac{t^2/4}{pn^3 + \frac13t\cdot\max \left\{ \sqrt{pn^2 \ln n}, \ln n\right\}}\right)
\\&\lesssim n \exp\left(-\min\left\{\frac{t^2}{pn^3},\frac{t}{\sqrt{pn^2 \ln n}},\frac{t}{\ln n}\right\}\right)
\lesssim \frac{1}{n}
\end{align*}
for $t\gtrsim \displaystyle \left\{ \sqrt{p n^3\ln n} \,, \sqrt{pn^2 (\ln n)^3} \,, (\ln n)^2 \right\}$, where the second term is smaller than the first for $n$ large enough.
Denote the complement of $\mathcal{E}$ by $\mathcal{E}^c$. For $t$ satisfying the stated condition,
\begin{align*}
\Pb&\left( \Vert S - \Eb[S|\setW] \Vert_2 > t \right)
\\&= \Eb_\setW \left[ \Pb_{\cdot|\setW}\left( \Vert S - \Eb[S|\setW] \Vert_2 > t \right) \right]
\\&= \Eb_\setW \left[ \Pb_{\cdot|\setW,\mathcal{E}}\left( \Vert S - \Eb[S|\setW] \Vert_2 > t \right) \Pb_{\cdot|\setW}(\mathcal{E}) + \Pb_{\cdot|\setW,\mathcal{E}^c}\left( \Vert S - \Eb[S|\setW] \Vert_2 > t \right) \Pb_{\cdot|\setW}(\mathcal{E}^c)\right]
\\&\lesssim  \Eb_\setW \left[ \Pb_{\cdot|\setW,\mathcal{E}}\left( \Vert S - \Eb[S|\setW] \Vert_2 > t \right)  +  \Pb_{\cdot|\setW}(\mathcal{E}^c)\right]
\\&\lesssim  \frac1n
\end{align*}
as each term in the expectation is $\bigO{\frac1n}$.
Thus, we have $\Vert S- \Eb[S|\setW]\Vert_2 \lesssim  \displaystyle \left\{ \sqrt{p n^3\ln n} \,,  (\ln n)^2 \right\}$.

To bound $\Vert \Eb[S|\setW] - \Eb[S]\Vert_2$, we note that the entries of the matrix comprises of mutually dependent variables in the set $\setB'$ defined in \eqref{eqn_pf_adis4_setB}.
We need to partition the entries into independent sets.
Since the dependency graph for $\setB'$ has maximum degree $d=4\left(\binom{n}{2}-1\right)$, we can partition $\setB'$ into $d+1$ independent sets of nearly identical sizes (equitable colouring).
Let $\Eb[S|\setW] - \Eb[S] = A^{(1)} + \ldots + A^{(d+1)}$ denote the corresponding partition of the matrix, where $A^{(\ell)}\in\nsetR^{n\times n}$ is a symmetric matrix, consisting of the variables in the $\ell$-th independent set.
Due to independence of variables, we have $A^{(\ell)}_{ij} = B'_{ijrs}$ for some $r,s$, and hence, we can conclude that each $A^{(\ell)}$ is a symmetric matrix with independent zero-mean entries, bounded by $2p\epsilon$ and variance at most $4p^2\epsilon^2$ (follows from properties of $B'_{ijrs}$).  
Thus, by matrix Bernstein inequality \citep{tropp2012userfriendly}, we have
\[
\Pb\left( \Vert A^{(\ell)}\Vert_2  > \tau \right) \leq n \exp\left(-\frac{t^2}{8p^2\epsilon^2 n + \frac23 p\epsilon t}\right),
\]
and combining with the union bound,
\begin{align*}
\Pb\left( \Vert \Eb[S|\setW] - \Eb[S] \Vert_2  > t \right) &\leq \Pb\left( \max_{\ell \in [d+1]} \Vert A^{(\ell)}\Vert_2  > \frac{t}{d+1} \right) 
\\&\leq n(d+1) \exp\left(-\frac{(\frac{t}{d+1})^2}{8p^2\epsilon^2 n + \frac23 p\epsilon\frac{t}{d+1}}\right)
\\&\lesssim n^3 \exp\left(- \min\left\{ \frac{t^2}{p^2\epsilon^2 n^5} \,, \frac{t}{p\epsilon n^2} \right\} \right),
\end{align*}
which is $\bigO{\frac1n}$ for $t\gtrsim \displaystyle p\epsilon n^2 \cdot \max\left\{ \sqrt{n\ln n},\ln n \right\}$, where the first term obviously dominates.
Combining the above derivations, we have with probability $1-\frac{1}{4n}$,
\begin{align}
\Vert S-\St\Vert_2 
\lesssim  \max\left\{ \sqrt{p n^3\ln n} \,,  (\ln n)^2\,, p\epsilon n^2 \sqrt{n \ln n} \,, p \epsilon \delta n^2\right\}
\label{eqn_pf_adis4_spec}
\end{align}
where the last term (arising due to $a_0$) is dominated by the third.

\textbf{Deriving interval for $\lambda$ in terms of $|\setQ|$.}
We now use \eqref{eqn_pf_adis4_Del1}, \eqref{eqn_pf_adis4_Del2} and \eqref{eqn_pf_adis4_spec} to complete the proof for the quadruplet setting.
To this end, our main objective is to verify the conditions in Proposition \ref{prop_sdp}:
\[
\frac{\Delta_1}{2} < \Delta_1 - 6\Delta_2, \text{ that is, } \Delta_2 < \frac{\Delta_1}{12}
\qquad \text{and} \qquad
\Vert S - \St\Vert_2 < \frac{n_{\min} \Delta_1}{2}.
\]
Using \eqref{eqn_pf_adis4_Del1} and the bound in \eqref{eqn_pf_adis4_Del2}, we observe that $\Delta_2 \lesssim \Delta_1$ if
\[
p \gtrsim \max\left\{\frac{\ln n}{\epsilon^2\delta^2 n^2 n_{\min}} \,, \frac{\ln n}{\epsilon \delta n^2 n_{\min}} \right\} \qquad \text{and} \qquad
\delta \gtrsim \sqrt{\frac{\ln n}{n_{\min}}},
\]
where the condition on $\delta$ arises due to the third term in the bound in \eqref{eqn_pf_adis4_Del2}. 
Similarly, comparing the bound in \eqref{eqn_pf_adis4_spec} to \eqref{eqn_pf_adis4_Del1}, we get that $\Vert S-\St\Vert_2 \lesssim n_{\min} \Delta_1$ if 
\[
p \gtrsim \max\left\{\frac{\ln n}{\epsilon^2\delta^2 n n_{\min}^2} \,, \frac{(\ln n)^2}{\epsilon \delta n^2 n_{\min}} \right\} \qquad \text{and} \qquad
\delta \gtrsim \frac{\sqrt{n\ln n}}{n_{\min}},
\]
where the condition on $\delta$ arises from the third bound in \eqref{eqn_pf_adis4_spec}.
Combining the above cases, we conclude that if 
\begin{equation}
\label{eqn_pf_adis4_2}
\delta \gtrsim \frac{\sqrt{n \ln n}}{n_{\min}} 
\qquad \text{and} \qquad 
p \gtrsim \frac{(\ln n)^2}{\epsilon^2 \delta^2 n n_{\min}^2} \;,
\end{equation}
then the criteria for $\Delta_2$ and $\Vert S-\St\Vert_2$ are satisfied, and by Proposition \ref{prop_sdp}, $X^*$ is the unique optimal solution for \ref{eqn_sdp} with the range of $\lambda$ given by
\begin{equation}
\label{eqn_pf_adis4_3}
\Vert S-\St \Vert_2 ~\lesssim~ \max\left\{\sqrt{p n^3 \ln n}\,, p\epsilon \sqrt{n^5\ln n}\,,(\ln n)^2\right\} ~\lesssim~ \lambda ~<~ \frac{p\epsilon \delta n_{\min}}{2}\binom{n}{2} ~=~ \frac{\Delta_1}{2} \;.
\end{equation}

We finally show that the condition on $p$ holds under the stated condition of $|\setQ|  \gtrsim \displaystyle \frac{n^3 (\ln n)^2}{\epsilon^2\delta^2 n_{\min}^2}$, and state the above interval for $\lambda$ in terms of $|\setQ|$. Under the assumption that each quadruplet is observed independently with probability $p$, we have that $\Eb[|\setQ|] = p \binom{\binom{n}{2}}{2} = O(pn^4)$.
By Bernstein inequality, it is easy to verify that for $p \gtrsim \frac{\ln n}{n^4}$ or equivalently $|\setQ| \gtrsim \ln n$, we have $|\setQ| \in \left(\frac12 \Eb[|\setQ|], \frac32\Eb[|\setQ|] \right)$ with probability $1-\bigO{\frac1n}$.
Hence, we can replace $p$ by $\frac{|\setQ|}{n^4}$ in \eqref{eqn_pf_adis4_2}--\eqref{eqn_pf_adis4_3} up to difference in constants,
which leads to the statement of Theorem \ref{thm_adis} in the quadruplet setting.

\subsection{Triplet setting}

The proof structure in the triplet case is similar to that of the quadruplet setting.
We derive an appropriate ideal matrix $\St$, where $\St = \Eb[S]$, except for some differences in the diagonal entries since $S_{ii} = 0$ for all $i$.
From the block structure of $\St$, we can compute $\Delta_1$.
Subsequently, concentration inequalities are used to derive upper bounds on $\Vert S-\St\Vert_2$ and $\Delta_2$ in terms of the model parameters.
As done in the analysis of \ref{eqn_adis4}, we let $\setW$ denote the collection of random pairwise similarities, and decompose 
\[
S-\St = \big(S - \Eb[S|\setW]\big) + \big(\Eb[S|\setW] - \Eb[S] \big) + \big(\Eb[S] - \St\big).
\]
The last term is easy to tackle, and we use separate concentration for the first two terms both in the context of $\Delta_2$ and the spectral norm.
Bounds on these terms, combined with Proposition \ref{prop_sdp}, provide sufficient conditions on $\delta$ and sampling rate $p$ such that exact recovery occurs.
Finally, we show that for $p$ large enough, the number of triplets  $|\setT|$ is close to its expected value $pn\binom{n-1}{2}$, and state the conditions in terms of $|\setT|$.

\textbf{Computation of $\Delta_1$.}
The expectation of the \ref{eqn_adis3} similarity $S_{ij}$, for $i\neq j$, is given by
\begin{equation}
\Eb[S_{ij}] = \sum_{r\neq i,j}  \Pb\big((i,j,r)\in\setT\big) - \Pb\big((i,r,j)\in\setT\big) + \Pb\big((j,i,r)\in\setT\big) - \Pb\big((j,r,i)\in\setT\big).
\label{eqn_expadis3}
\end{equation}
We now compute each term in the summation using the notation $\psi_i \in [k]$ to indicate $i \in \setC_{\psi_i}$. The expected values of the terms are given in Table \ref{tab_expadis3}, where the last column represents the overall term for each $r\neq i,j$ in \eqref{eqn_expadis3}.
The derivation for these values is identical to the one in the quadruplet setting.

\begin{table}[ht]
  \caption{Value of each term in the summation in \eqref{eqn_expadis3}, assuming $i,j,r$ are distinct.}
  \label{tab_expadis3}
  \centering
  \begin{tabular}{@{}c@{}c@{}c@{}c@{}c@{}c@{}}
    \toprule
    Case     & $\;\Pb\big((i,j,r)\in\setT\big)\;$ & $\;\Pb\big((i,r,j)\in\setT\big)\;$ & $\;\Pb\big((j,i,r)\in\setT\big)\;$ & $\;\Pb\big((j,r,i)\in\setT\big)\;$ & \;Aggregate\; \\
    \midrule
    $\psi_i = \psi_j = \psi_r$ & $p/2$ & $p/2$ & $p/2$ & $p/2$ & 0 \\
    $\psi_i = \psi_j \neq \psi_r$ & $p(1+\epsilon\delta)/2$ & $p(1-\epsilon\delta)/2$ & $p(1+\epsilon\delta)/2$ & $p(1-\epsilon\delta)/2$ & $2p\epsilon\delta$ \\
    $\psi_i \neq \psi_j = \psi_r$ & $p/2$ & $p/2$ & $p(1-\epsilon\delta)/2$ & $p(1+\epsilon\delta)/2$ & $-p\epsilon\delta$\\
    $\psi_i = \psi_r \neq \psi_j$ & $p(1-\epsilon\delta)/2$ & $p(1+\epsilon\delta)/2$ & $p/2$ & $p/2$ & $-p\epsilon\delta$\\
    $\psi_i \neq \psi_j \neq \psi_r$ & $p/2$ & $p/2$ & $p/2$ & $p/2$ & 0 \\
    \bottomrule
  \end{tabular}
\end{table}

Based on Table \ref{tab_expadis3}, we can infer that for $i\neq j$ such that $\psi_i =  \psi_j$, 
\begin{align*}
\Eb[S_{ij}] &= \sum_{r \notin \setC_{\psi_i}}  2p\epsilon\delta = 2p\epsilon\delta (n - n_{\psi_i})
\end{align*} 
For $i, j$ such that $\psi_i\neq\psi_j$, we have
\begin{align*}
\Eb[S_{ij}] &= \sum_{r \in \setC_{\psi_i}, r\neq i}\hskip-2ex (-p\epsilon\delta) +   \sum_{r \in \setC_{\psi_j}, r\neq j}\hskip-2ex (-p\epsilon\delta) = -p\epsilon\delta (n_{\psi_i} + n_{\psi_j} - 2).
\end{align*} 
Hence, we define the ideal similarity matrix as $\St_{ij} = Z\Sigma Z^T$, where $\Sigma_{\ell\ell} = 2p\epsilon\delta (n-\nl)$ and $\Sigma_{\ell\ell'} = -p\epsilon\delta (\nl + n_{\ell'} - 2)$ for $\ell\neq\ell'$, and we can compute 
\begin{equation}
\Delta_1 
=  p\epsilon\delta (n-2).
\label{eqn_pf_adis3_Del1}
\end{equation}

\textbf{Preliminary computations and definitions for concentration.}
We define $\setW = \{w_{ij} : i<j\}$ as the collection of random pairwise similarities, and split the concentration of $\Delta_2$ and $\Vert S-\St\Vert_2$ into terms involving $S-\Eb[S|\setW]$ and $\Eb[S|\setW] - \Eb[S]$.
The basic idea is discussed in the corresponding part of the quadruplet setting, and here, we introduce the key random variables.
We first write
\begin{align}
S_{ij} &= \sum_{r\neq i,j}  \big(\ind{(i,j,r)\in\setT} - \ind{(i,r,j)\in\setT}\big) +  \big(\ind{(j,i,r)\in\setT} - \ind{(j,r,i)\in\setT}\big)
\nonumber
\\&= \sum_{r\neq i,j} \xi_{ijr} \big(\ind{w_{ij}>w_{ir}} - \ind{w_{ij}<w_{ir}}\big) +  \xi_{jir} \big(\ind{w_{ji}>w_{jr}} - \ind{w_{ji}<w_{jr}}\big)
\label{eqn_adis3_formal}
\end{align}
where $\xi_{ijr} \in \{-1,0,+1\}$ denotes whether the comparison between $(i,j)$ and $(i,r)$ is observed ($\xi_{ijr}=0$ if not observed), and whether the crowd response was correct $(\xi_{ijr}=+1)$ or flipped $(\xi_{ijr}=-1)$.
Under our sampling and noise model,
\[
\Pb(\xi_{ijr} = 0) = 1-p, \qquad   \Pb(\xi_{ijr} = 1) =\frac{p(1+\epsilon)}{2}\;,  \qquad   \Pb(\xi_{ijr} = -1) =\frac{p(1-\epsilon)}{2} 
\]
and so, $\Eb[\xi_{ijr}] = p\epsilon$ and $\var{\xi_{ijr}} \leq p$.
The set  $\Xi=\{ \xi_{ijr} ~:~  j<r, i \neq j,r\}$ denotes the collection of such random variables, where we abuse notation by using $\xi_{ijr}$ and $\xi_{irj}$ to refer to the same variable. 
We note that the variables in $\Xi$ are mutually independent.

We use the continuous nature of $F_{in},F_{out}$ to write $\ind{w_{ij}>w_{ir}} - \ind{w_{ij}<w_{ir}} = 2\ind{w_{ij}>w_{ir}} - 1$, and further define
\begin{equation}
\begin{aligned}
&S_{ij} - \Eb[S_{ij}|\setW] = \sum_{r\neq i,j} B_{ijr} + B_{jir}
&&\text{with} \quad
B_{ijr} = (\xi_{ijr} - p\epsilon)\big(2\ind{w_{ij}>w_{ir}} - 1\big),
\\
&\Eb[S_{ij}|\setW] - \Eb[S_{ij}] = \sum_{r\neq i,j} B'_{ijr} + B'_{jir}
&&\text{with} \quad
B'_{ijr} = 2p\epsilon\big(\ind{w_{ij}>w_{ir}}- \Pb(w_{ij}>w_{ir}) \big).
\end{aligned}
\label{eqn_pf_adis3_B}
\end{equation}
The random variables $B_{ijr},B'_{ijr}$ have the following properties: $|B_{ijr}| \leq 2$, $|B'_{ijr}| \leq 2p\epsilon$ with probability 1, $\Eb[B_{ijr}] = \Eb[B'_{ijr}] = 0$, $\var{B_{ijr}} \leq p$ and $\var{B'_{ijr}} \leq 4p^2\epsilon^2$.
We define the sets 
\begin{equation}
\begin{aligned}
\setB &= \{ B_{ijr} ~:~ j\neq i, r\neq i,j \}, \\
\setB' &= \{ B'_{ijr} ~:~ j\neq i, r\neq i,j  \},\\
\setB_{i\ell} &= \{ B_{ijr}, B_{jir} ~:~ j\in\setC_\ell, r\neq i,j \} \qquad\text{for every } i\in[n],\ell\in[k], \\
\text{and} \qquad \setB'_{i\ell} &= \{ B'_{ijr}, B'_{jir} ~:~ j\in\setC_\ell, r\neq i,j  \} \qquad\text{for every } i\in[n],\ell\in[k].
\end{aligned}
\label{eqn_pf_adis3_setB}
\end{equation}
Each of $\setB$ and $\setB'$ have $n(n-1)(n-2)$ random variables, whereas $\setB_{i\ell}, \setB'_{i\ell}$ have at most $2\nl\left(n-2\right)$ random variables.
Note that $B_{ijr} = - B_{irj}$, but conditioned on $\setW$, $B_{ijr}$ is independent of all other variables in $\setB$. 
Thus, a dependency graph on $\setB$, conditioned on $\setW$, has a maximum degree of 1.
The same is also true for $\setB'_{i\ell}$.
On the other hand, $B'_{ijr}$ depends on the random variables that involve either $w_{ij}$ or $w_{ir}$, that is $B'_{irj}, B'_{jir}, B'_{jri}, B'_{rij}, B'_{rji}$, as well as all six variants of variables $B'_{ijr'}$ and $B'_{ij'r}$, $j',r' \notin\{i,j,r\}$. Thus each $B'_{ijr}$ depends on at most $5 + 12(n-3) = \bigO{n}$ variables in $\setB'$.
The same holds when we restrict the set to $\setB'_{i\ell}$.
Thus, the dependency graph for $\setB'$ and $\setB'_{i\ell}$ have dependency graph with $\bigO{n}$ maximum degree.
We now use the above defined random variables and their properties to derive upper bounds on $\Delta_2$ and $\Vert S-\St\Vert_2$.

\textbf{Upper bound for $\Delta_2$.}
To derive a bound on $\Delta_2$, we first note that
\begin{align*}
\Delta_2 \leq \max_{i\in[n]} \max_{\ell\in[k]} \left| \frac{1}{\nl} \sum_{j\in\setC_\ell} S_{ij} - \Eb[S_{ij} | \setW] \right| +  \max_{i\in[n]} \max_{\ell\in[k]} \left| \frac{1}{\nl} \sum_{j\in\setC_\ell} \Eb[S_{ij} | \setW] - \Eb[S_{ij}]\right| + \max_{i\in [n]} \frac{\St_{ii}}{n_{\min}}.
\end{align*}
In the subsequent steps, we bound the first term.
For any $t>0$, the union bound leads to
\begin{align*}
\Pb&\left( \max_{i\in[n]} \max_{\ell\in[k]} \left| \frac{1}{\nl} \sum_{j\in\setC_\ell} S_{ij} - \Eb[S_{ij}|\setW] \right| > t\right)
\nonumber
\\&\leq \sum_{i\in[n]} \sum_{\ell \in [k]} \Pb\left( \left| \sum_{j\in\setC_\ell} S_{ij} - \Eb[S_{ij}|\setW] \right| > \nl t\right)
\nonumber
\\&= \sum_{i\in[n]} \sum_{\ell \in [k]} \Eb_{\setW} \left[\Pb_{\cdot|\setW}\left( \left| \sum_{j\in\setC_\ell} \sum_{r\neq i,j} B_{ijr} + B_{jir} \right| >\nl t\right) \right] 
\end{align*}
The summation inside the conditional probability involves terms in $\setB_{i\ell}$ defined in~\eqref{eqn_pf_adis3_setB}, and the previous discussions show that the dependency graph of $\setB_{i\ell}$ has maximum degree of 1.
Hence, we can split the $2\nl(n-2)$ variables in $\setB_{i\ell}$ into two independent sets, say $\setB_{i\ell,(1)}$ and $\setB_{i\ell,(2)}$, and derive concentration for each of them separately using Bernstein inequality in the following way.
\begin{align*}
\Pb&\left( \max_{i\in[n]} \max_{\ell\in[k]} \left| \frac{1}{\nl} \sum_{j\in\setC_\ell} S_{ij} - \Eb[S_{ij}|\setW] \right| > t\right)
\\&\leq \sum_{i\in[n]} \sum_{\ell \in [k]} \Eb_{\setW} \left[\Pb_{\cdot|\setW}\left( \left| \sum_{B\in \setB_{i\ell,(1)}} B \right| >\frac{\nl t}{2}\right) + [\Pb_{\cdot|\setW}\left( \left| \sum_{B\in \setB_{i\ell,(2)}} B \right| >\frac{\nl t}{2}\right) \right] 
\\&\leq \sum_{i\in[n]} \sum_{\ell \in [k]}   \Eb_{\setW} \left[ 2\exp\left(-\frac{ (\nl t/2)^2}{2p|\setB_{i\ell,(1)}| + 2\nl t / 3}\right) + 2\exp\left(-\frac{ (\nl t/2)^2}{2p|\setB_{i\ell,(2)}| + 2\nl t / 3}\right) \right]
\\&\lesssim n^2 \exp\left(- \min\left\{ \frac{n_{\min} t^2}{pn} \,, n_{\min} t \right\}\right),
\end{align*}
where the last step follows by noting that each of the two independent sets have $\bigOmega{n\nl}$ variables, and the bounds are independent of $\setW$.
The above probability is $\bigO{\frac1n}$ for $t\gtrsim \displaystyle \max\left\{ \sqrt{\frac{pn\ln n}{n_{\min}}}\,, \frac{\ln n}{n_{\min}}\right\}$.

For the second term in the upper bound for $\Delta_2$, we have
\begin{align*}
\Pb\hskip-0.5ex\left( \max_{i\in[n]} \max_{\ell\in[k]} \left| \frac{1}{\nl} \sum_{j\in\setC_\ell} \Eb[S_{ij}|\setW] - \Eb[S_{ij}] \right| > t\right)
\nonumber
\leq \sum_{i\in[n]} \sum_{\ell \in [k]} \Pb\hskip-0.5ex\left( \left| \sum_{j\in\setC_\ell} \sum_{r\neq i,j} B'_{ijr} + B'_{jir} \right| > \nl t\right)
\end{align*}
where the tail bound is for the sum of all random variables in $\setB'_{i\ell}$.
Since the dependency graph on $\setB'_{i\ell}$ has maximum degree $d = \bigO{n}$, we can obtain an equitable $(d+1)$-colouring with each independent set of size $\lfloor |\setB'_{i\ell}|/(d+1) \rfloor$ or $\lceil |\setB'_{i\ell}|/(d+1) \rceil$, which are smaller than $\nl$.
We denote the independent sets by $\setB'_{i\ell,(1)}, \ldots,\setB'_{i\ell,(d+1)}$, and use Bernstein inequality to bound the summation over each independent set. 
Hence, we bound the probability for every $i,\ell$, as
\begin{align*}
 \Pb\left( \left| \sum_{j\in\setC_\ell} \sum_{r\neq i,j} B'_{ijr} + B'_{jir} \right| > \nl t\right)
 &\leq  \Pb\left( \max_{r \in \{1,\ldots, d+1\}} \left| \sum_{B'\in \setB'_{i\ell,(r)}} B' \right| > \frac{\nl t}{(d+1)}\right)
 \\&\leq  \sum_{r=1}^{d+1} \Pb\left( \left| \sum_{B'\in \setB'_{i\ell,(r)}} B' \right| > \frac{\nl t}{(d+1)}\right)
 \\&\leq  2(d+1) \exp\left( - \frac{(\frac{\nl t}{(d+1)})^2}{8p^2\epsilon^2|\setB'_{i\ell,(r)}| + \frac43 p\epsilon \frac{\nl t}{(d+1)}} \right)
\\&\lesssim n \exp\left( - \min\left\{ \frac{\nl t^2}{p^2\epsilon^2 n^2}\,,  \frac{\nl t}{p\epsilon n} \right\} \right),
\end{align*}
which is $\bigO{\frac{1}{n^3}}$ for $t\gtrsim \displaystyle  p\epsilon n \cdot  \max\left\{ \sqrt{\frac{\ln n}{n_{\min}}},\frac{\ln n}{n_{\min}} \right\}$.
The first term dominates for $n_{\min} \gtrsim \ln n$, which arises due to the condition on $\delta$.
Combining the above discussions we claim that, with probability $1-\frac{1}{4n}$,
\begin{align}
\Delta_2 
&\lesssim   \max\left\{ \sqrt{\frac{p n \ln n}{n_{\min}}} \,, \frac{\ln n}{n_{\min}}\,, p\epsilon n  \sqrt{\frac{\ln n}{n_{\min}}} \,, \frac{p\epsilon \delta n}{n_{\min}} \right\} \;,
\label{eqn_pf_adis3_Del2}
\end{align}
 where the last term is obviously dominated by the third term.
 
\textbf{Upper bound for $\Vert S-\St\Vert_2$.}
Similar to the case of $\Delta_2$, we bound
\[
\Vert S - \St \Vert_2 \leq \Vert S - \Eb[S|\setW] \Vert_2 + \Vert \Eb[S|\setW] - \Eb[S] \Vert_2 + \Vert \Eb[S] - \St \Vert_2,
\]
where the last term equals $\max_i \St_{ii}$. 
For the first term, we derive a bound conditioned on $\setW$.
Recall from \eqref{eqn_pf_adis3_B}--\eqref{eqn_pf_adis3_setB} that, conditioned on $\setW$, the matrix $S-\Eb[S|\setW]$ comprises of variables in $\setB$, which has a dependence graph with degree 1.
We partition $\setB$ into two independent sets via equitable colouring, and write $S-\Eb[S|\setW] = A + A'$, where $A$ and $A'$ are the matrices corresponding to each of the independent sets.
We derive a spectral norm for each of $A$ and $A'$.
For this, we first claim that, conditioned on $\setW$,  
the event $\mathcal{E} = \left\{\max_{i,j} \big\{ |A_{ij}|, |A'_{ij}| \big\} \lesssim \max \left\{ \sqrt{pn \ln n}, \ln n\right\}\right\}$ occurs with probability $1 - \bigO{\frac1n}$.
To see this, observe that $A_{ij}$ (or $A'_{ij}$) is a sum of $2(n-2)$ independent random variables $B_{ijr}, B_{jir}$. 
By Bernstein inequality,
\[
\Pb_{\cdot|\setW} ( |A_{ij}| > \tau ) \leq 2\exp\left(-\frac{\tau^2}{4p(n-2) + \frac43\tau}\right)
\lesssim \exp\left(-\min\left\{\frac{\tau^2}{pn}, \tau\right\}\right)
\]
which is $\bigO{\frac{1}{n^3}}$ for $\tau \gtrsim \max \left\{ \sqrt{pn \ln n}, \ln n\right\}$.
Applying the union bound gives $\Pb(\mathcal{E}^c) = \bigO{\frac1n}$.

Conditioned on $\setW$ and $\mathcal{E}$, the matrices $A,A'$ have independent zero mean entries, with each entry bounded by $ \bigO{\max \left\{ \sqrt{pn \ln n}, \ln n\right\}}$.
Furthermore, from the variance of $B_{ijr}$, we have $\max_i \sum_j \var{A_{ij}} < 2pn^2$, and the same holds for $A'$.
Hence, by matrix Bernstein inequality \citep{tropp2012userfriendly}, 
\begin{align*}
\Pb_{\cdot| \setW,\mathcal{E}} \left(\Vert S - \Eb[S |\setW] \Vert_2 > t \right)
& \leq \Pb_{\cdot| \setW,\mathcal{E}} \left(\Vert A \Vert_2 > t/2 \right) + \Pb_{\cdot| \setW,\mathcal{E}} \left(\Vert A' \Vert_2 > t/2 \right)
\\&\leq 2n\exp\left(-\frac{t^2/4}{pn^2 + \frac13t\cdot\max \left\{ \sqrt{pn \ln n}, \ln n\right\}}\right)
\\&\lesssim n \exp\left(-\min\left\{\frac{t^2}{pn^2},\frac{t}{\sqrt{pn \ln n}},\frac{t}{\ln n}\right\}\right)
\lesssim \frac{1}{n}
\end{align*}
for $t\gtrsim \displaystyle \left\{ \sqrt{p n^2 \ln n} \,, \sqrt{pn (\ln n)^3} \,, (\ln n)^2 \right\}$, where the second term is smaller than the first for $n$ large enough.
As in the quadruplet setting, we add the probability $\Pb(\mathcal{E}^c)$ and take expectation over $\setW$ to obtain $\Vert S- \Eb[S|\setW]\Vert_2 \lesssim  \displaystyle \left\{ \sqrt{p n^2\ln n} \,,  (\ln n)^2 \right\}$ with probability $1-\bigO{\frac1n}$.

To bound $\Vert \Eb[S|\setW] - \Eb[S]\Vert_2$, we note that the entries of the matrix comprises of mutually dependent variables in the set $\setB'$ defined in \eqref{eqn_pf_adis3_setB}.
Since the dependency graph for $\setB'$ has maximum degree $d=\bigO{n}$, we partition $\setB'$ into $d+1$ independent sets of nearly identical sizes (equitable colouring).
Let $\Eb[S|\setW] - \Eb[S] = A^{(1)} + \ldots + A^{(d+1)}$ denote the corresponding partition of the matrix, where $A^{(\ell)}\in\nsetR^{n\times n}$ is a symmetric matrix consisting of the variables in the $\ell$-th independent set.
Due to the independence of the variables, we have $A^{(\ell)}_{ij} = A^{(\ell)}_{ji} = B'_{ijr}$ or $B'_{jir}$ for some $r\neq i,j$.
Hence, each $A^{(\ell)}$ is a symmetric matrix with independent zero-mean entries, bounded by $2p\epsilon$ and variance at most $4p^2\epsilon^2$ (follows from properties of $B'_{ijr}$).  
Thus, by matrix Bernstein inequality \citep{tropp2012userfriendly}, we have
\[
\Pb\left( \Vert A^{(\ell)}\Vert_2  > \tau \right) \leq n \exp\left(-\frac{t^2}{8p^2\epsilon^2 n + \frac23 p\epsilon t}\right),
\]
and combining with the union bound,
\begin{align*}
\Pb\left( \Vert \Eb[S|\setW] - \Eb[S] \Vert_2  > t \right) &\leq \Pb\left( \max_{\ell \in [d+1]} \Vert A^{(\ell)}\Vert_2  > \frac{t}{d+1} \right) 
\\&\leq n(d+1) \exp\left(-\frac{(\frac{t}{d+1})^2}{8p^2\epsilon^2 n + \frac23 p\epsilon\frac{t}{d+1}}\right)
\\&\lesssim n^2 \exp\left(- \min\left\{ \frac{t^2}{p^2\epsilon^2 n^3} \,, \frac{t}{p\epsilon n} \right\} \right),
\end{align*}
which is $\bigO{\frac1n}$ for $t\gtrsim \displaystyle p\epsilon n \cdot \max\left\{ \sqrt{n\ln n},\ln n \right\}$, where the first term obviously dominates.
Combining the above derivations, we have with probability $1-\frac{1}{4n}$,
\begin{align}
\Vert S-\St\Vert_2 
\lesssim  \max\left\{ \sqrt{p n^2\ln n} \,,  (\ln n)^2\,, p\epsilon n \sqrt{n \ln n} \,, p \epsilon \delta n\right\}
\label{eqn_pf_adis3_spec}
\end{align}
where the last term (arising due to $\max_i \St_{ii}$) is dominated by the third.

\textbf{Deriving interval for $\lambda$ in terms of $|\setT|$.}
We now use \eqref{eqn_pf_adis3_Del1}, \eqref{eqn_pf_adis3_Del2} and \eqref{eqn_pf_adis3_spec} to complete the proof for the triplet setting.
We verify the conditions in Proposition \ref{prop_sdp} by deriving conditions under which $\Delta_2 < \frac{1}{12} \Delta_1$ and $\Vert S-\St\Vert_2 < \frac12 n_{\min}\Delta_1$.
Similar to the proof for the quadruplet setting, we compare the upper bounds in \eqref{eqn_pf_adis4_Del2} and \eqref{eqn_pf_adis4_spec} with $\Delta_1$ and $n_{\min}\Delta_1$, respectively. 
As in the previous setting, the first two bounds in \eqref{eqn_pf_adis4_Del2}--\eqref{eqn_pf_adis4_spec} lead to conditions on $p$, while the third term leads to a condition on $\delta$.
Combining the different cases, it follows that if 
\begin{equation}
\label{eqn_pf_adis3_2}
\delta \gtrsim \frac{\sqrt{n\ln n}}{n_{\min}} 
\qquad \text{and} \qquad 
p \gtrsim \frac{(\ln n)^2}{\epsilon^2 \delta^2  n_{\min}^2} \;,
\end{equation}
then the criteria for $\Delta_2$ and $\Vert S-\St\Vert_2$ are satisfied, and by Proposition \ref{prop_sdp}, $X^*$ is the unique optimal solution for \ref{eqn_sdp} with the range of $\lambda$ given by
\begin{equation}
\label{eqn_pf_adis3_3}
\Vert S-\St \Vert_2 ~\lesssim~ \max\left\{ \sqrt{p n^2 \ln n} \,, p\epsilon \sqrt{n^3\ln n} \,, (\ln n)^2\right\}   ~\lesssim~ \lambda ~<~ \frac{p\epsilon \delta n_{\min}(n-2)}{2} ~=~ \frac{\Delta_1}{2} \;.
\end{equation}

We finally show that the condition on $p$ holds under the stated condition of $|\setT|  \gtrsim \displaystyle \frac{n^3 (\ln n)^2}{\epsilon^2\delta^2 n_{\min}^2}$, and state the above interval for $\lambda$ in terms of $|\setT|$. Under the assumption that each triplet is observed independently with probability $p$, we $\Eb[|\setT|] = pn\binom{n-1}{2} = \bigO{pn^3}$.
By Bernstein inequality, it is easy to verify that for $p \gtrsim \frac{\ln n}{n^3}$ or equivalently $|\setT| \gtrsim \ln n$, we have $|\setT| \in \left(\frac12 \Eb[|\setT|], \frac32\Eb[|\setT|] \right)$ with probability $1-\bigO{\frac1n}$.
Hence, we can replace $p$ by $\frac{|\setT|}{n^3}$ in \eqref{eqn_pf_adis3_2}--\eqref{eqn_pf_adis3_3} up to differences in constants,
which leads to the statement of Theorem \ref{thm_adis} in the triplet setting.

\section{Algorithmic details}\label{app: sec: Algorithm}

In this section, we provide details on the modified SPUR algorithm that we use to tune the parameter $\lambda$, and to select the number of clusters.


SPUR, acronym for Semidefinite Program with Unknown $r$ ($r$ denoting the number of clusters), was proposed by \citet{yan2018provable} to tune the parameter $\lambda$ of \ref{eqn_sdp} in the context of graph clustering (see Algorithm \ref{alg: SPUR}).
The underlying idea of this approach is to search for the optimal $\lambda$ using a grid search over the range $0<\lambda<\lambda_{\max}$, where $\lambda_{\max}$ is derived from an exact recovery result under stochastic block model. 

\begin{algorithm}
\caption{Semidefinite Program with Unknown $k$ (SPUR).}
\label{alg: SPUR}
\SetKwInOut{Input}{input}\SetKwInOut{Output}{output}
\DontPrintSemicolon
\Input{~graph $A$, number of candidates $T$.}
\Begin{
\For{$t = 1$ to $T$}{
$\lambda_t = \exp\left(\frac{t}{T}\ln\left(1+\lambda_{\max}\right)\right)-1$. \hfill (\citet{yan2018provable} set $\lambda_{\max}=\norm{A}_{op}$)\;
Solve \ref{eqn_sdp} with $\lambda =\lambda_t$ to obtain  $X_t$.\;
Estimate $k_t =$ integer approximation of $\tr{X_t}$.
}
Choose $\hat{t} = \displaystyle{\argmax_t} \frac{\sum_{i\leq k_t} \sigma_i(X_t)}{\tr{X_t}}$, where $\sigma_i(X_t)$ denotes $i$-th largest eigenvalue of $X_t$.\;
}
\Output{~Number of clusters $k_{\hat{t}}$, $X_{\hat{t}}$.}
\end{algorithm}

In the present setting, Theorem \ref{thm_adis} shows that the planted clusters can be exactly recovered given a sufficient number of comparisons and an appropriate choice of $\lambda$.
From Theorem \ref{thm_adis}, a candidate for $\lambda_{\max}$ can be chosen as $\frac{|\setT|}{n}$ (for triplets) or $\frac{|\setQ|}{n}$ (for quadruplets), which is a loose upper bound for the theoretical interval for $\lambda$, obtained by noting that $\epsilon \delta n_{\min} \leq n$. 
Thus, following \citet{yan2018provable}, we could use Algorithm~\ref{alg: SPUR} with our choice of $\lambda_{\max}$.

Unfortunately, this approach has two main drawbacks. First, it ignores the lower bound in Theorem~\ref{thm_adis} and, second, setting $T$, the number of $\lambda$ values that should be considered in Algorithm~\ref{alg: SPUR}, is difficult. To address the former issue, we propose to consider Theorem~\ref{thm_adis} once more and to use $\lambda_{\text{min}} = \sqrt{ c (\ln n)/n}$ as a lower bound for $\lambda$ instead of $0$, as used in \citet{yan2018provable}. To address the latter issue, we use the fact that the estimated number of clusters $k$ monotonically decreases with $\lambda$ as shown in the next Lemma.
\begin{mylem}[\textbf{The estimated number of clusters decreases monotonically with increasing $\lambda$}]
For any $\lambda>0$, let $X_\lambda$ denote the solution of \ref{eqn_sdp} and $k_\lambda = \lfloor\tr{X_\lambda}\rceil$ be the integer approximation of $\tr{X_\lambda}$, which is an estimate of the number of clusters. Then, $k_\lambda$ is a non-increasing function of $\lambda$, that is
\begin{align*}
  \lambda' \geq \lambda \Rightarrow k_{\lambda'} \leq k_\lambda.
\end{align*}
\end{mylem}
\begin{proof}
We start this proof by noting that since $k_\lambda$ is the integer approximation of $\tr{X_\lambda}$, it suffices to show that $\tr{X_\lambda}$ is a non-increasing function of $\lambda$. 
Then, consider distinct $\lambda',\lambda$ and let $X_{\lambda'}$, $X_\lambda$ be the solutions of \ref{eqn_sdp} with parameters $\lambda',\lambda$, respectively. We have
\begin{align*}
    \tr{SX_\lambda} - \lambda\tr{X_\lambda} \geq{}& \tr{SX_{\lambda'}} - \lambda\tr{X_{\lambda'}} \;, \\
    \tr{SX_\lambda} - \lambda'\tr{X_\lambda} \leq{}& \tr{SX_{\lambda'}} - \lambda'\tr{X_{\lambda'}} \;. 
\end{align*}
Subtracting the second inequality from the first inequality implies
\begin{align*}
\tr{SX_\lambda} - \lambda\tr{X_\lambda} -{}& \left(\tr{SX_\lambda} - \lambda'\tr{X_\lambda}\right) \\
\geq{}& \tr{SX_{\lambda'}} - \lambda\tr{X_{\lambda'}} - \tr{SX_{\lambda'}} + \lambda'\tr{X_{\lambda'}} 
\end{align*}
which implies
\[
(\lambda' - \lambda)\tr{X_\lambda} \geq (\lambda' - \lambda)\tr{X_{\lambda'}} 
\]
or equivalently, 
$(\lambda' - \lambda)(\tr{X_{\lambda'}}-\tr{X_\lambda}) \leq 0$.
Thus, for $\lambda'>\lambda$, we can conclude that $\tr{X_{\lambda'}} \leq \tr{X_\lambda}$, which shows that $\tr{X_\lambda}$ and  $k_\lambda$ are non-increasing functions of $\lambda$.
\end{proof}
Following this, using $\lambda_{\min}$ and $\lambda_{\max}$, we get two estimates of the number of clusters, $k_{\lambda_{\text{min}}}$ and $k_{\lambda_{\text{max}}}$. Then, we search over $k \in [k_{\lambda_{\max}},k_{\lambda_{\min}}]$ instead of searching over $\lambda$---in practice, it helps to search over the values $\max\{2,k_{\lambda_{\max}}\} \leq k \leq k_{\lambda_{\min}}+2$. 
We select $k$ that maximises the above SPUR objective, where $X_k$ is computed using the simpler \ref{eqn_sdp_PW} \citep{yan2018provable}. This approach is summarized in Algorithm~\ref{alg: modified SPUR}.

\begin{algorithm}[ht]
\caption{Comparison-based SPUR}
\label{alg: modified SPUR}
\SetKwInOut{Input}{input}\SetKwInOut{Output}{output}
\DontPrintSemicolon
\Input{~$n$ and $\mathcal{T}$ or $\mathcal{Q}$}
\Begin{
Define $c=|\mathcal{T}|$ or $|\mathcal{Q}|$\;
Let $S$ be obtained with \ref{eqn_adis3} or \ref{eqn_adis4}\;
Define $\lambda_{\min}=\sqrt{\frac{c(\ln c)}{n}}$ and $\lambda_{\max}=\frac{c}{n}$ \; 
$X_{\lambda_{\min}},X_{\lambda_{max}}\leftarrow$ SDP-$\lambda_{\min}$, SDP-$\lambda_{\max}$ on $S$\;
$k_{\lambda_{\min}},k_{\lambda_{\max}}\leftarrow\lfloor\tr{X_{\lambda_{\min}}}\rceil,\lfloor\tr{X_{\lambda_{\max}}}\rceil$\;
\For{$k = \max\{2,k_{\lambda_{\max}}\}$ to $k_{\lambda_{min}}+2$}{
Solve \ref{eqn_sdp_PW} with $k$ to obtain $X_k$.\;
}
Choose $\hat{k} = \underset{k}{\textup{argmax}} \frac{\sum_{i\leq k} \sigma_i(X_k)}{\tr{X_k}}$, where $\sigma_i(X_k)$ denotes $i$-th largest eigenvalue of $X_k$.
}
\Output{~Number of clusters $\hat{k}$, $X_{\hat{k}}$.}
\end{algorithm}





\section{Additional results for the planted model}
\label{app: sec: planted}

In this section, we provide additional experiments on our planted model. We show that changing the clustering method used in the last step of our approach to cluster the matrix $X$ learned by \ref{eqn_sdp} or \ref{eqn_sdp_PW} does not affect the results. We demonstrate that, given a sufficient number of comparisons, SPUR correctly estimates the number of clusters. We give details on the distributions used in Figure~\ref{fig: planted distributions}. Finally, we consider several additional experiments where we vary the planted model parameters that were ignored in Section~\ref{sec:experiments} in the main paper.

\subsection{Clustering method in the last step}
In the last step of our approach, we use $k$-means to cluster the learned matrix $X_k$. We experimentally demonstrate here that the partition obtained is, in fact, independent of the clustering algorithm used in this step. 
Hence, in Figure~\ref{app: tab: plots SPUR vs known k}, we compare spectral clustering with k-means. As in the main paper, we here consider varying the number of observations, $|\mathcal{T}|,|\mathcal{Q}|$ and varying the crowd noise $\epsilon$ for both the setting where $k$ is estimated by SPUR and where we consider $k$ to be known.
There is no differences between the ARI obtained when using k-means or spectral clustering.

\begin{figure}
    \centering
\begin{minipage}{\textwidth}
\centering
\begin{tabular}{@{}*{4}{c@{}}}
\toprule
SPUR with k-means & 
SPUR with spectral& 
known $k$  with k-means &
known $k$ with spectral
\\ 
\cmidrule(r){1-1}\cmidrule(lr){2-2}\cmidrule(lr){3-3}\cmidrule(l){4-4}

\includegraphics[width=0.24\textwidth]{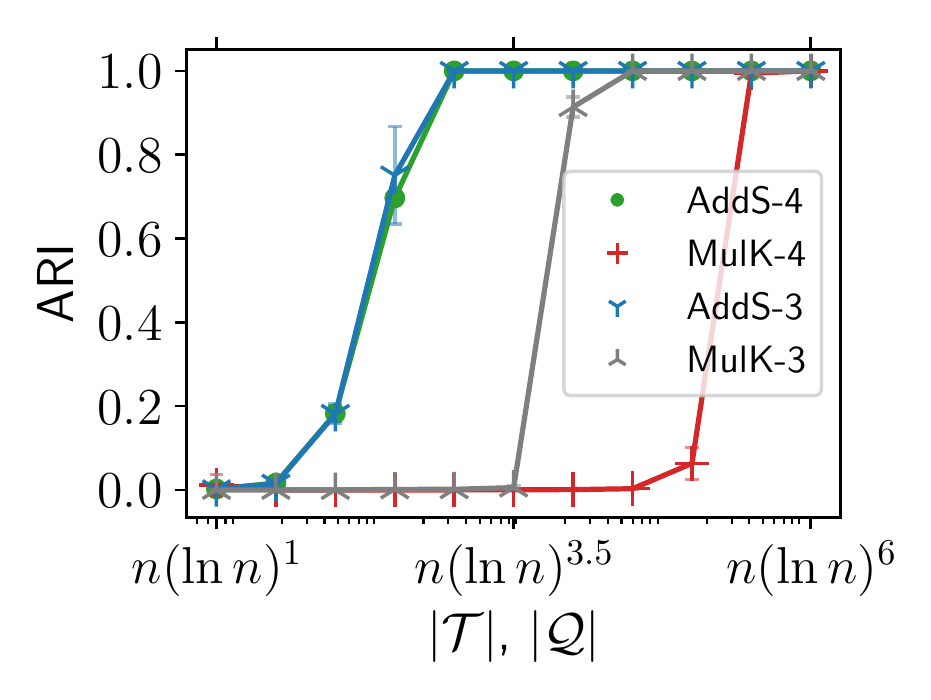}&
\includegraphics[width=0.24\textwidth]{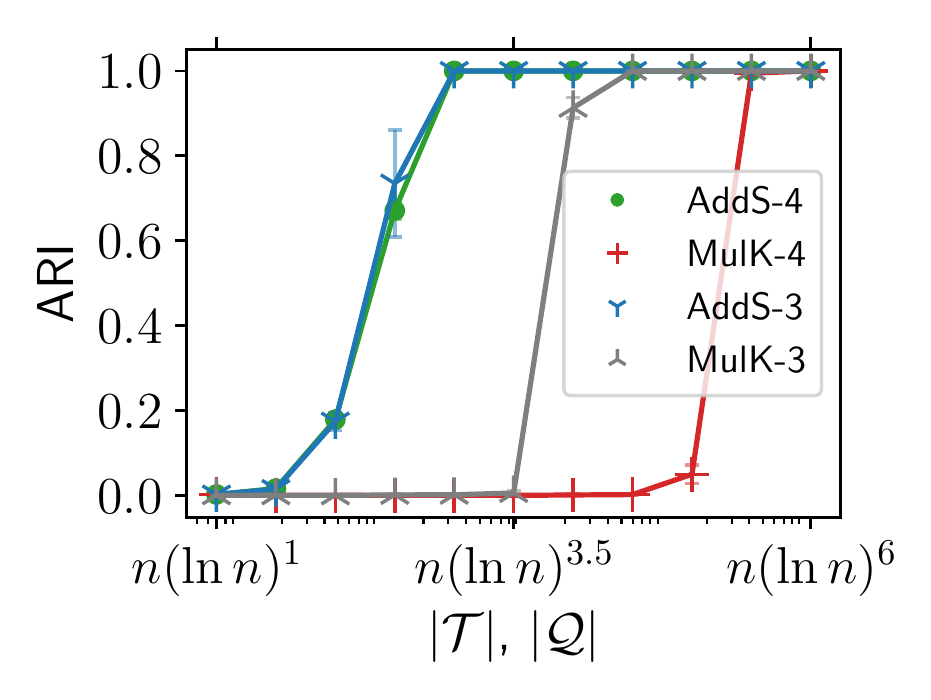}&
\includegraphics[width=0.24\textwidth]{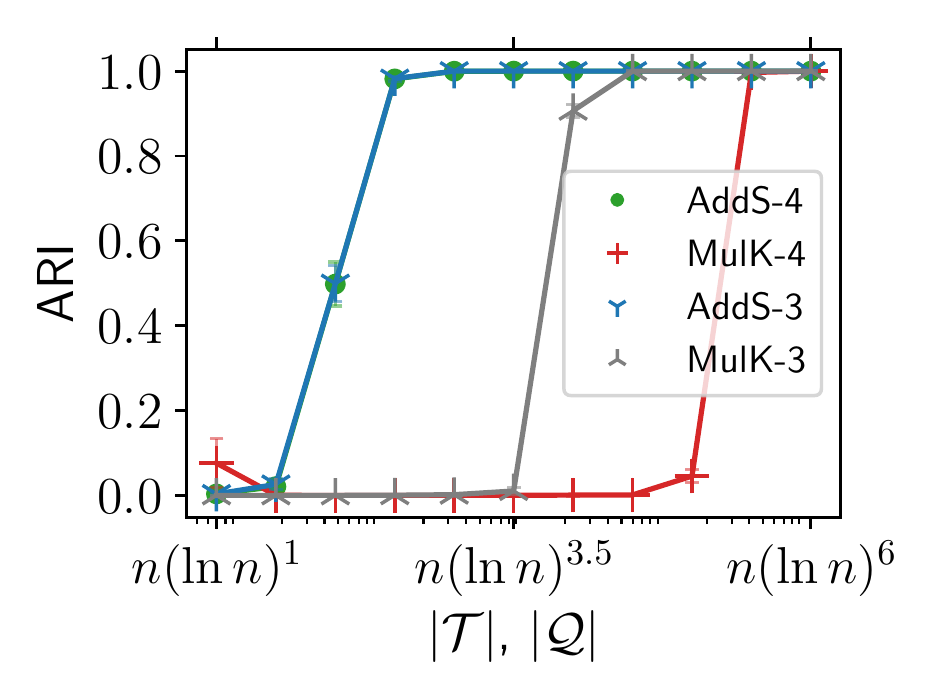}&
\includegraphics[width=0.24\textwidth]{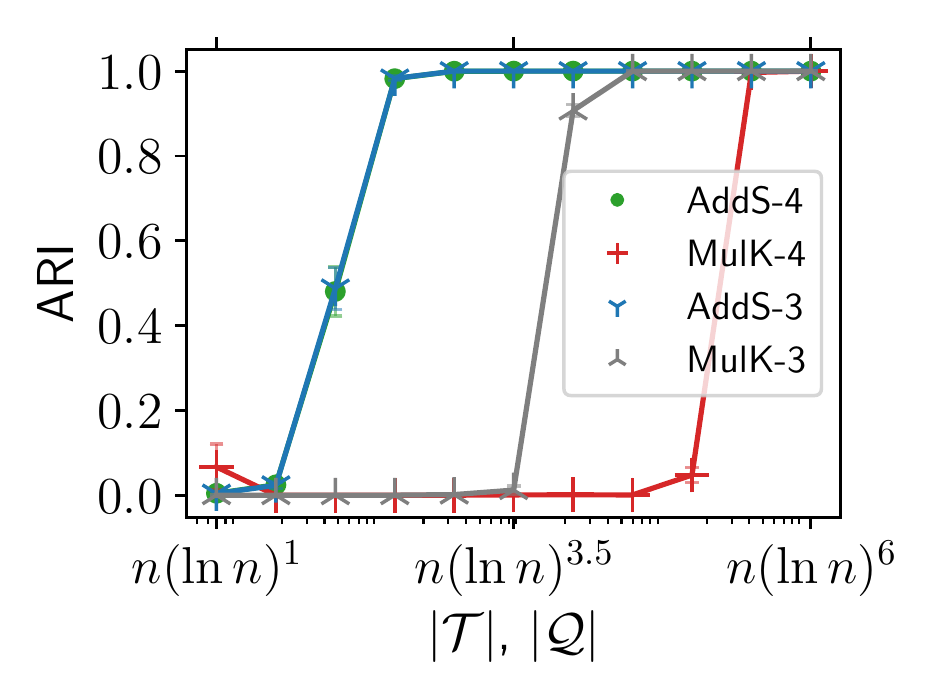}\\

\includegraphics[width=0.24\textwidth]{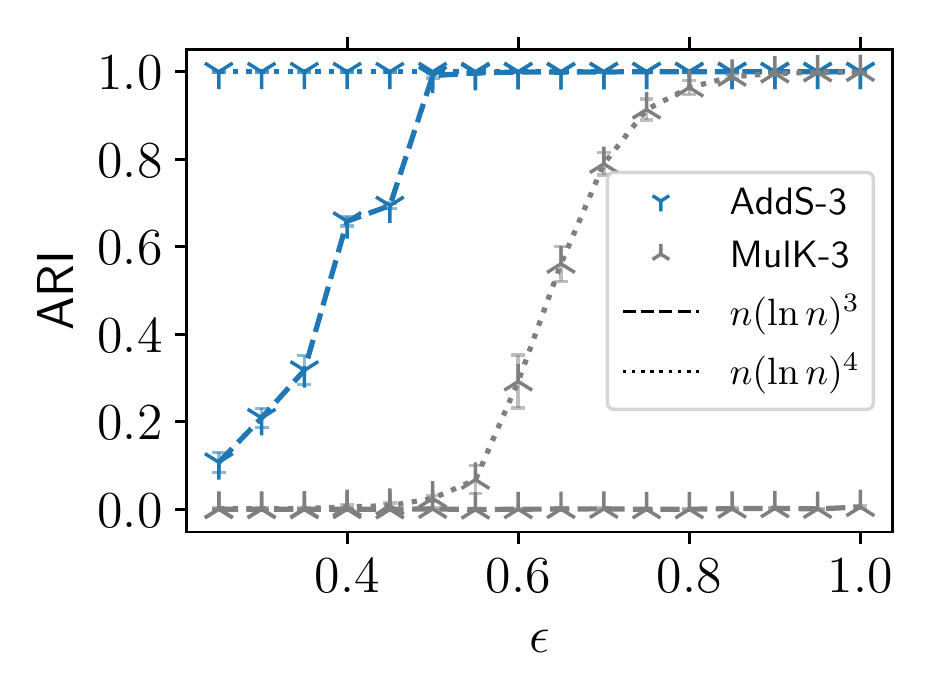}&
\includegraphics[width=0.24\textwidth]{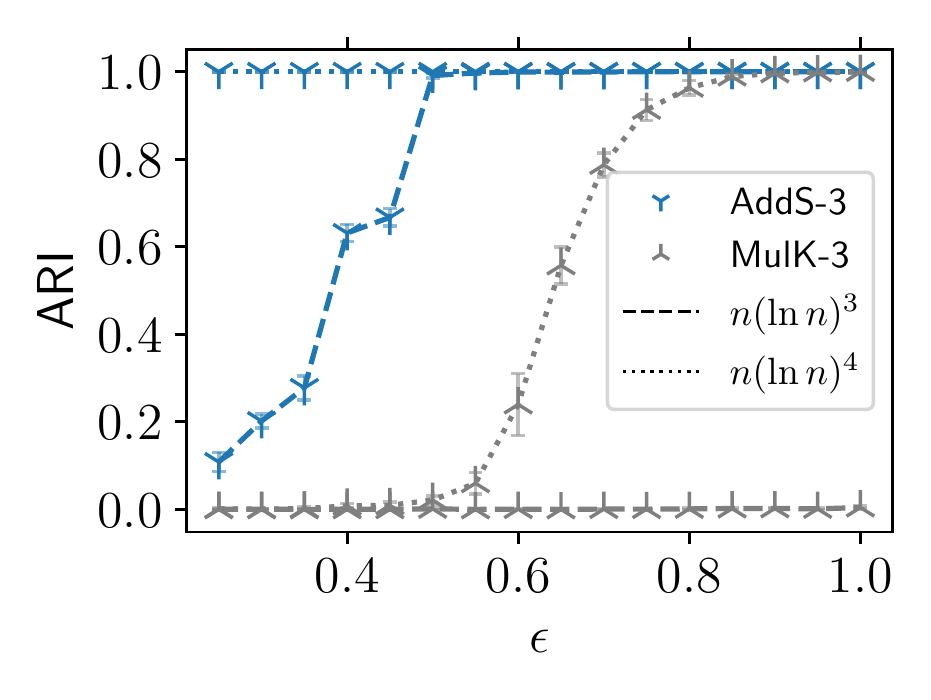}&
\includegraphics[width=0.24\textwidth]{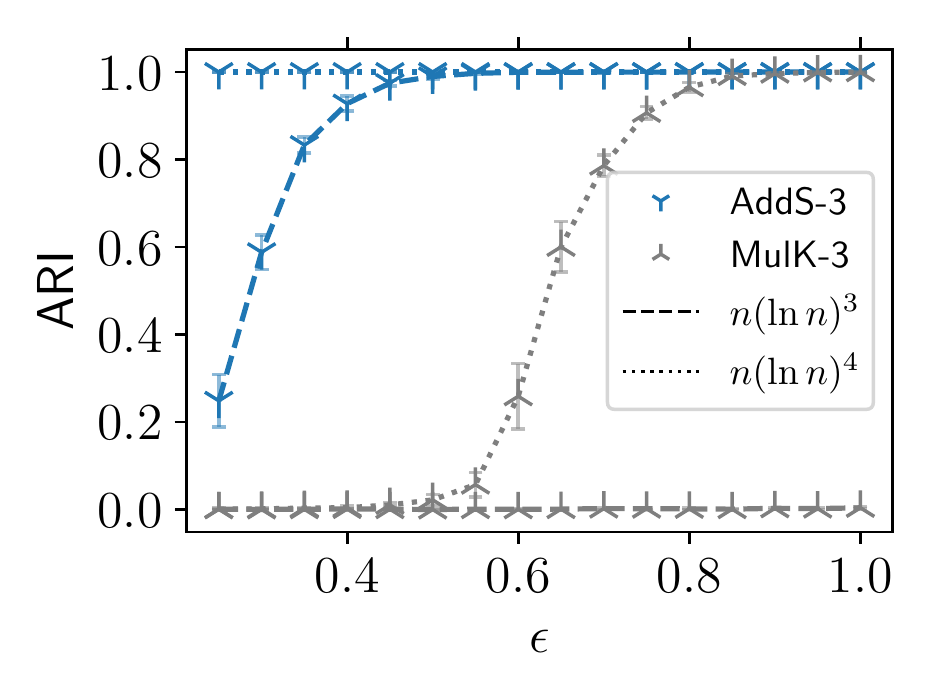}&
\includegraphics[width=0.24\textwidth]{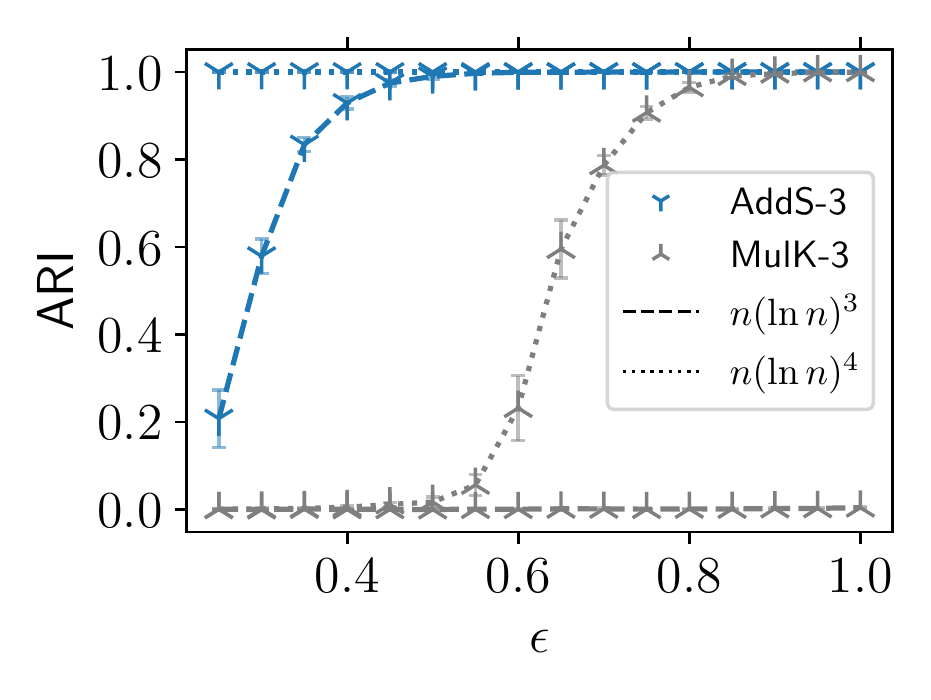}\\

\hline
\end{tabular}
\end{minipage}
\caption{Comparing clustering algorithms to partition $X$ in the last step. Using k-means or spectral clustering does not affect the output of our approach.}
    \label{app: tab: plots SPUR vs known k}
\end{figure}

\subsection{Compare SPUR with known $k$}\label{app: sec: compare SPUR to known k}

An important question is how good is SPUR at estimating the true number of clusters. We illustrate this in Figure~\ref{app: tab: plots SPUR vs known k 1}. We start by comparing the first two columns, showing how the ARI changes for various parameters of the planted model. In the setting of $|\mathcal{Q}|,|\mathcal{T}|=n(\ln n)^3$  we see that using a known number of clusters outperforms SPUR, especially in parameter ranges that are harder to cluster (e.g. small $\delta,\epsilon$ or for a larger number of clusters).
If we consider $|\mathcal{Q}|,|\mathcal{T}|=n(\ln n)^4$, SPUR correctly estimates the number of clusters and thus we omit the plots with known $k$. 

\subsection{Experimental details for changing $F_{in},F_{out}$ in the planted model}
In this section, we give implementation details on the different distributions considered in Figure~\ref{fig: planted distributions}.
In the following let $\phi$ be the normal pdf and $\Phi$ the normal cdf. Recall that, in all the experiments, we fix $\delta = 0.5$ as the default.

\textbf{Parameters for $F_{in}$ and $F_{out}$ normal distributions.} Let $F_{in} = \mathcal{N}(\mu_{in}, \sigma)$ and $F_{out} = \mathcal{N}(\mu_{out}, \sigma)$.
We fix $\sigma = 0.1$ and $\mu_{out} = 0$. Using $\delta$ we can compute $\mu_{in}$.
Indeed, in this case, the cumulative distribution function is known and, thus, by setting it equal to $\Pb_{w\sim F_{in},w'\sim F_{out}}(w>w') = \frac{1+\delta}{2} \quad \text{for some } \delta \in (0,1]$ (as given in Equation~\eqref{eqn_Fcondn}) we directly get the $\delta$ defined in Section~\ref{sec_background}:
$\delta=2 \Phi\left({\left(\mu_{in}-\mu_{out}\right)}/{  (\sqrt{2} \sigma)}\right)-1$. Then, assuming that $\mu_{out} = 0$, we get
$\mu_{in} = \sqrt{2}\sigma\Phi^{-1}\left( \frac{1+\delta}{2} \right)$.



\textbf{Parameters for $F_{in}$ and $F_{out}$ Beta distributions.} Let $F_{in} = \betadist(\alpha, \beta), F_{out} = \betadist(\alpha', \beta')$.
We set $\alpha' = \beta' = 1$ such that $F_{out} = \betadist(1, 1) = \uniform(0,1)$. We can then compute
\begin{align*}
     \Pb_{w\sim \betadist(\alpha, \beta),w'\sim \betadist(1, 1)}(w>w') &= {\textstyle\expect_w} \left[ \int_0^w dw'\right]\\
      &= {\textstyle\expect_w}\left[ w\right]\\
      &=\frac{\alpha}{\alpha+\beta}
\end{align*}
where the last line follows from the mean of the Beta distribution. Setting this equal to $\frac{1+\delta}{2} $ and solving for $\alpha$ gives:
$ \alpha = \beta\left(\frac{1+\delta}{1-\delta}\right)$. In our experiments, we fix $\beta=2$.

\textbf{Parameters for $F_{in}$ Normal and $F_{out}$ Uniform.} Let $F_{in} = \mathcal{N}(\mu, 0), F_{out} = \uniform(0,1)$. To set $\mu$, we compute: 
\begin{align*}
     \Pb_{w\sim\mathcal{N}(\mu, 0),w'\sim \uniform(0,1)}(w>w') =& \int_0^\infty \phi(w-\mu)dw\left[ \int_0^{\min(w, 1)} dw'\right]\\
      =& \int_0^1 w \phi(w -\mu) dw + \int_1^\infty \phi(w -\mu) dw + \mu \left(\Phi(1-\mu) - \Phi(-\mu)\right)\\
     =& 1 + \phi(-\mu) - \phi(1-\mu) + (\mu-1)\Phi(1-\mu) -  \mu\Phi(-\mu)
\end{align*}{}
Solving numerically for $\mu$ gives
$    \mu = \frac{1+\delta}{2}$.

\subsection{Influence of different planted model parameters}
In this section we present additional experiments where we vary various parameters of the planted model.
Recall that we consider the following parameters as default: $n=1000$, $k=4$, $\epsilon=0.75$, $|\mathcal{T}| = |\mathcal{Q}|= n(\ln n)^4$ and $F_{in}=\mathcal{N}\left(\sqrt{2}\sigma\Phi^{-1}\left( \frac{1+\delta}{2} \right),\sigma^2\right),F_{out}=\mathcal{N}\left(0,\sigma^2\right)$ with  $\sigma = 0.1$ and $\delta = 0.5$.

\textbf{Number of samples $n$, first row in Figure~\ref{app: tab: plots SPUR vs known k 1}.}
We can first note that for $|\mathcal{Q}|,|\mathcal{T}|=n(\ln n)^3$ there is no difference in the behaviour between SPUR and known $k$. Both \ref{eqn_adis3} and \ref{eqn_adis4} achieve full recovery while \ref{eqn_3k} and \ref{eqn_4k} predictions are random. To learn somewhat meaningful partitions with \ref{eqn_3k}, one needs to increase the number of observations to $n(\ln n)^4$. However, even with this many comparisons, \ref{eqn_4k} still learns random clusters. 

\textbf{Intrinsic noise $\delta$, second row in Figure~\ref{app: tab: plots SPUR vs known k 1}.}
Using $|\mathcal{Q}|,|\mathcal{T}|=n(\ln n)^3$, we see that, for both SPUR and known $k$, \ref{eqn_adis3} and \ref{eqn_adis4} exactly recover the clusters even when the intrinsic noise is high, that is $\delta = 0.4$.  \ref{eqn_3k} and \ref{eqn_4k} can only make random predictions in this case.
When the number of observations increases to $n(\ln n)^4$, \ref{eqn_adis3} and \ref{eqn_adis4} exactly recover the clusters even for values of $\delta$ that are as small as $0.25$. In this case, \ref{eqn_4k} still predicts random clusters, while \ref{eqn_3k} is able to recover the clusters when the intrinsic noise is sufficiently small, that is $\delta \geq 0.6$.

\textbf{Crowd noise $\epsilon$, third row in Figure~\ref{app: tab: plots SPUR vs known k 1}.}
This parameter was already analyzed in the main paper. The plots are recalled here for the sake of completeness.

\textbf{Number of clusters $k$, fourth row in Figure~\ref{app: tab: plots SPUR vs known k 1}.}
Finally, we vary the number of planted clusters. Here, we observe the most noticeable difference between SPUR and known $k$. For  $|\mathcal{Q}|,|\mathcal{T}|=n(\ln n)^3$, \ref{eqn_adis3} and \ref{eqn_adis4} with SPUR achieve perfect recovery for up to five clusters. While we notice a similar behaviour for \ref{eqn_adis3} and \ref{eqn_adis4} with known $k$, the drop in ARI only starts for $k > 7$ and is far less important than with SPUR.
For  $n(\ln n)^4$ observations \ref{eqn_adis3} and \ref{eqn_adis4} consistently recover all the clusters. On the other hand, \ref{eqn_3k} only recovers clusters up to $k=3$ (here, \ref{eqn_3k} uses the number of clusters estimated by \ref{eqn_adis3} with SPUR, that is $k=3$). Once again, \ref{eqn_4k} can only make random predictions.

\begin{figure}
    \centering
\begin{minipage}{\textwidth}
\centering
\begin{tabular}{@{}*{3}{c@{}}}
\toprule
$~~n(\ln n)^3$ with SPUR & 
$~~n(\ln n)^3$ with known $k$ & 
$~~n(\ln n)^4$ with SPUR \\ 
\cmidrule(r){1-1}\cmidrule(lr){2-2}\cmidrule(l){3-3}

\includegraphics[width=0.33\textwidth]{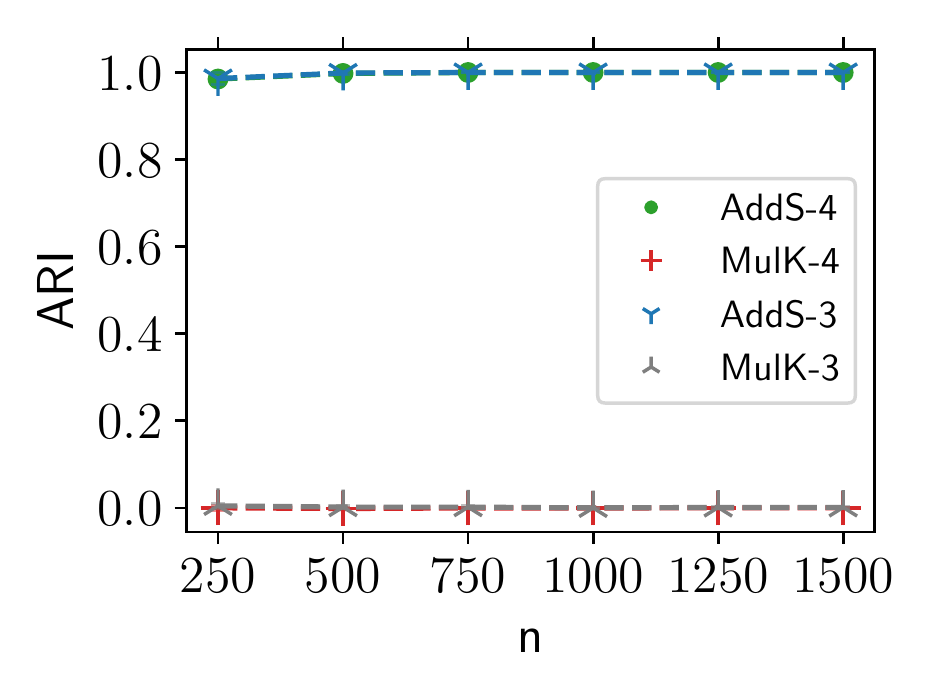}&
\includegraphics[width=0.33\textwidth]{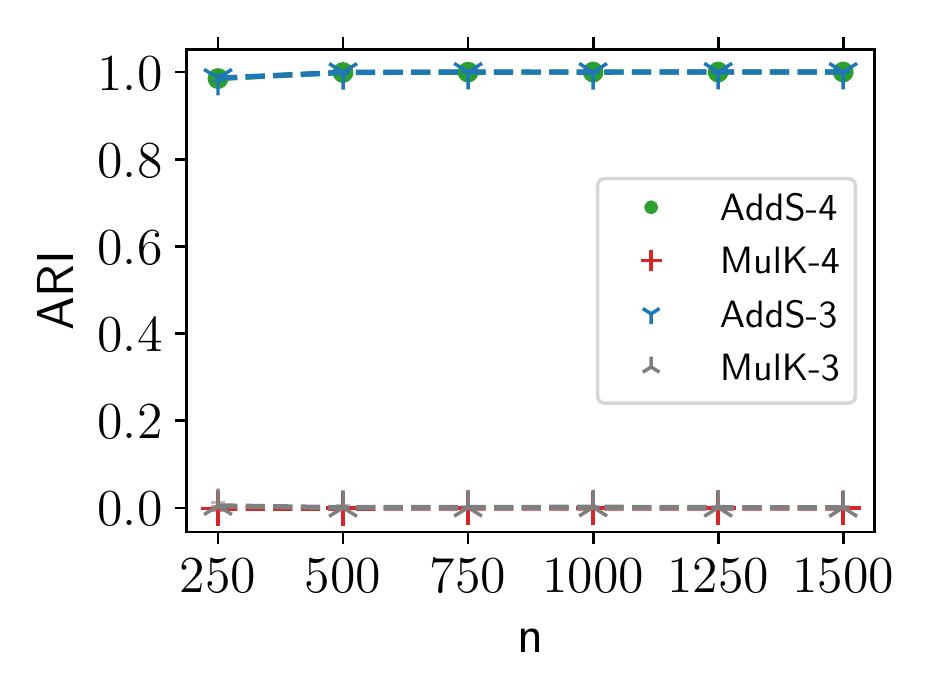}&
\includegraphics[width=0.33\textwidth]{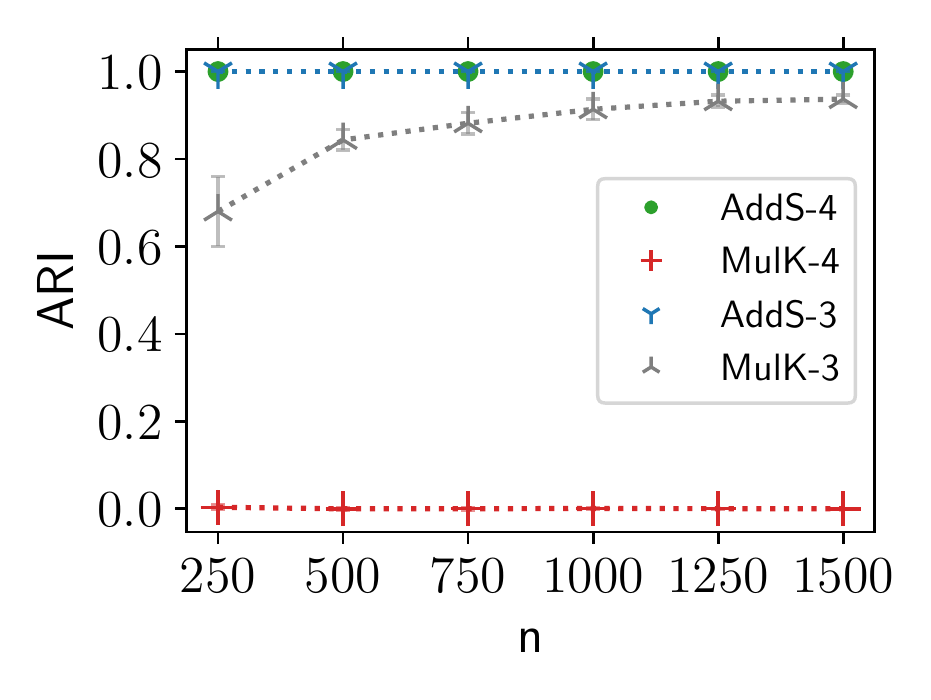}\\

\includegraphics[width=0.33\textwidth]{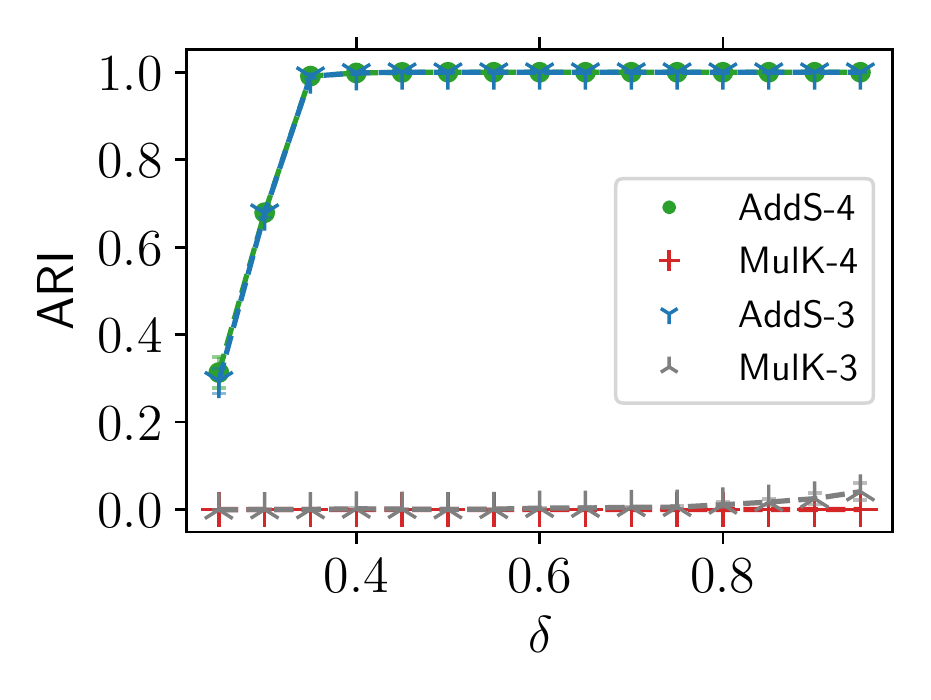}&
\includegraphics[width=0.33\textwidth]{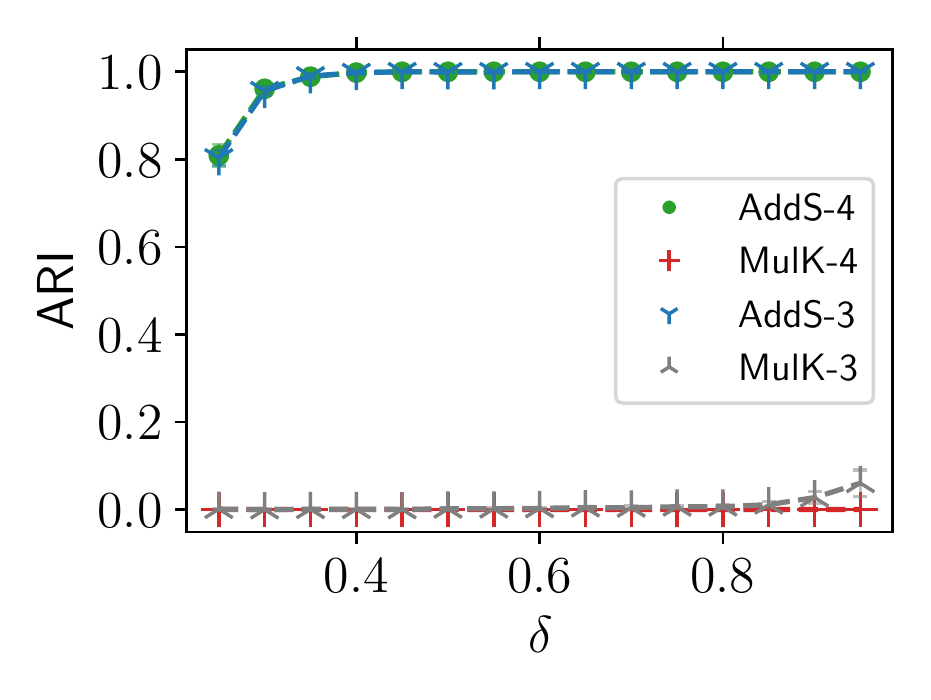}&
\includegraphics[width=0.33\textwidth]{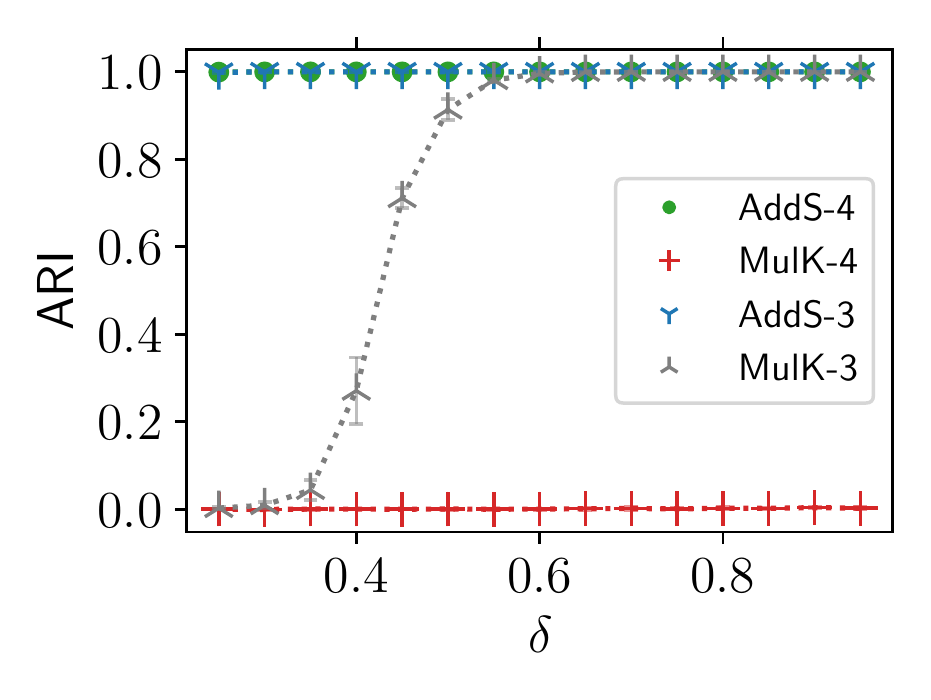}\\

\includegraphics[width=0.33\textwidth]{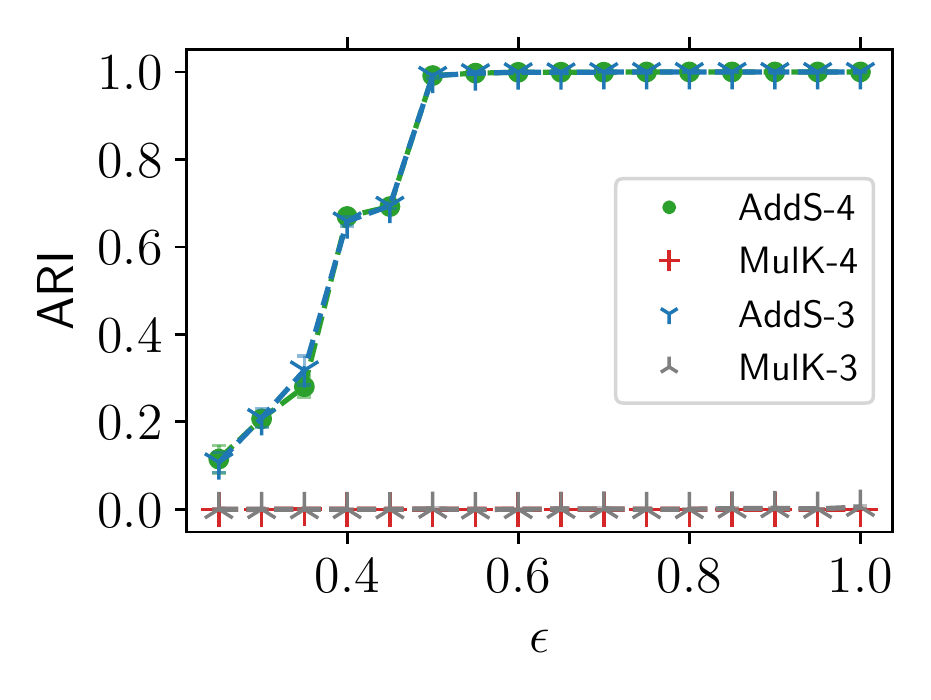}&
\includegraphics[width=0.33\textwidth]{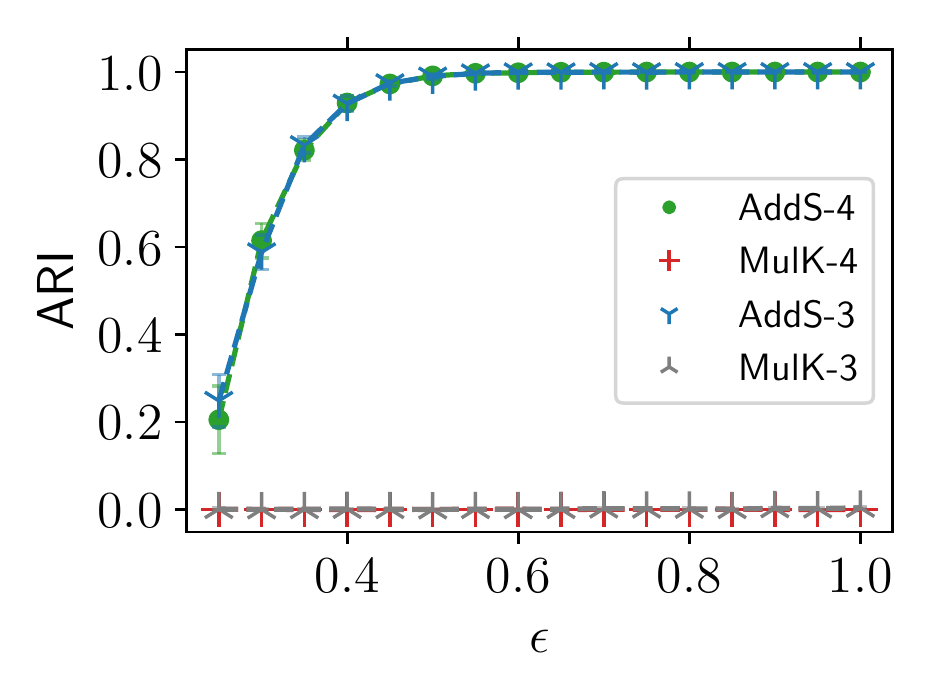}&
\includegraphics[width=0.33\textwidth]{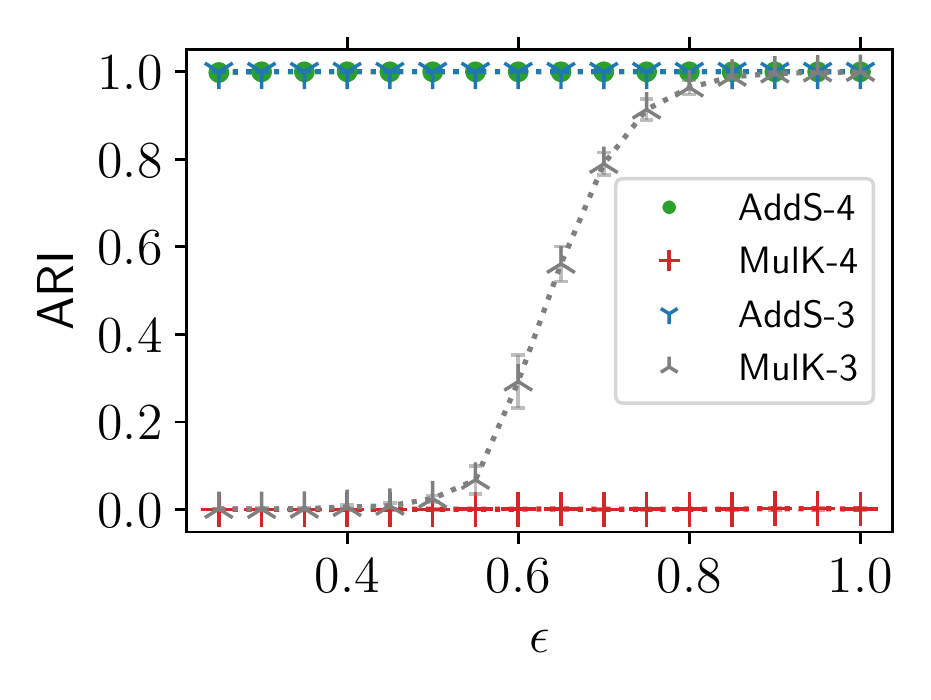}\\

\includegraphics[width=0.33\textwidth]{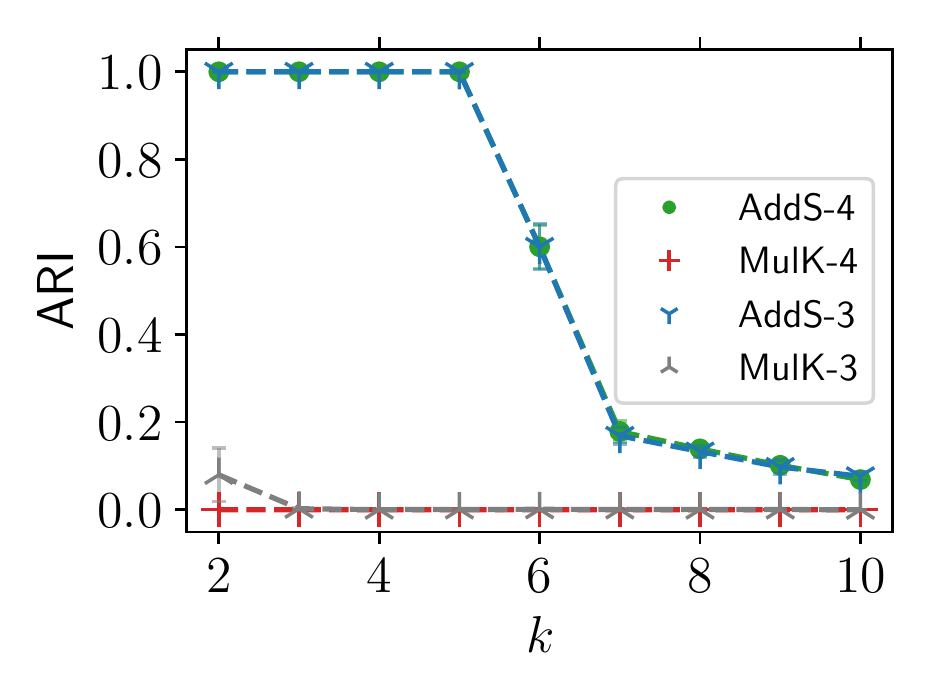}&
\includegraphics[width=0.33\textwidth]{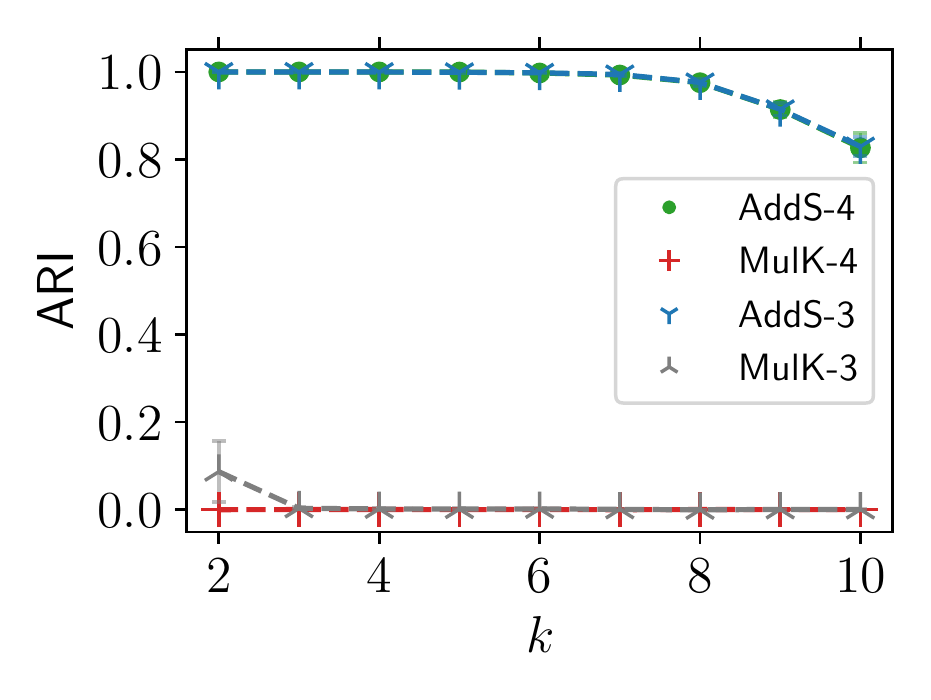}&
\includegraphics[width=0.33\textwidth]{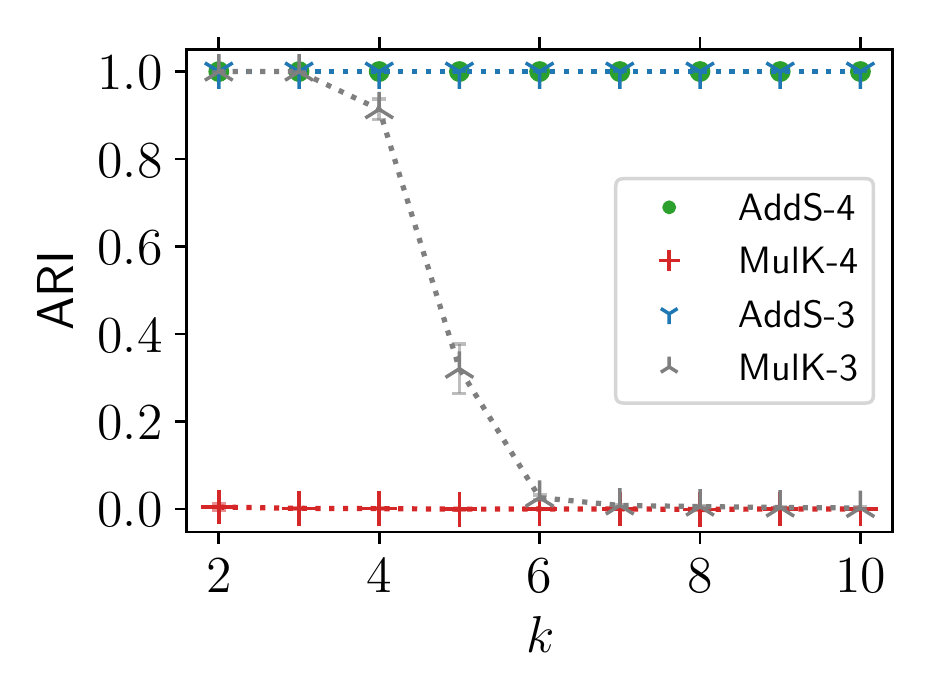}\\


\hline
\end{tabular}
\end{minipage}
\caption{Further experiments on the planted model. On the one hand, SPUR needs sufficiently many comparisons to correctly estimate the number of underlying clusters. On the other hand, our approaches are not overly sensitive to changes in the planted model parameters and are able to exactly recover the planted clusters with $n(\ln n)^3$ comparisons even in fairly difficult cases (small $\delta$, high $k$, \ldots). Furthermore, given $n(\ln n)^4$ comparisons, our approaches are able to exactly recover the planted clusters in all the considered cases.}
    \label{app: tab: plots SPUR vs known k 1}
\end{figure}




\section{Further results for experiments on real comparison based data}
\label{app: sec: realdata}
In this final section we present supporting results for the real data experiments presented in Section~\ref{sec:experiments}.

\subsection{Details on the Car dataset}
The Car dataset \citep{kleindessner2016cardataset} is a comparison based dataset that contains 60 examples grouped into 3 classes (SUV, city cars, sport cars) with 4 outliers. This dataset originally comes
with a set of 6056 comparisons of the form ``$x_i$ is most central in the triple $x_i$, $x_j$, $x_k$.'' Each of these comparisons corresponds to two triplets: ``$x_j$ is more similar to $x_i$ than to $x_k$'' and ``$x_k$ is more similar to $x_i$ than to $x_j$.'' Hence, we have access to 12112 triplet comparisons.
%


\subsection{Food Dataset}\label{app: sec: food}
In addition to the Car dataset we now look at a second comparison based dataset called Food \citep{wilber2014hcomp}. It contains 100 food images and comes with 190376 triplet comparisons. Since there are no ground truth labels for the food dataset, we use the number of clusters estimated by SPUR for all methods and plot, in Figure~\ref{app: fig: food comparison}, the similarity matrix between the different clustering approaches considered. Here, there is a high degree of agreement between all the clustering methods. Thus, most approaches predict the same clusters up to minor differences for a few data points. In Figure~\ref{app: tab: food samples}, we plot the clusters obtained by \ref{eqn_adis3} with SPUR (estimated $k$ is $2$). The two clusters seem to separate \textit{Sweet foods} from \textit{Savoury foods}. Intuitively, it seems indeed natural for humans to judge that two sweet foods are more similar to each other than to a third savoury food.
\begin{figure}
    \centering
    \includegraphics[width=0.3\textwidth]{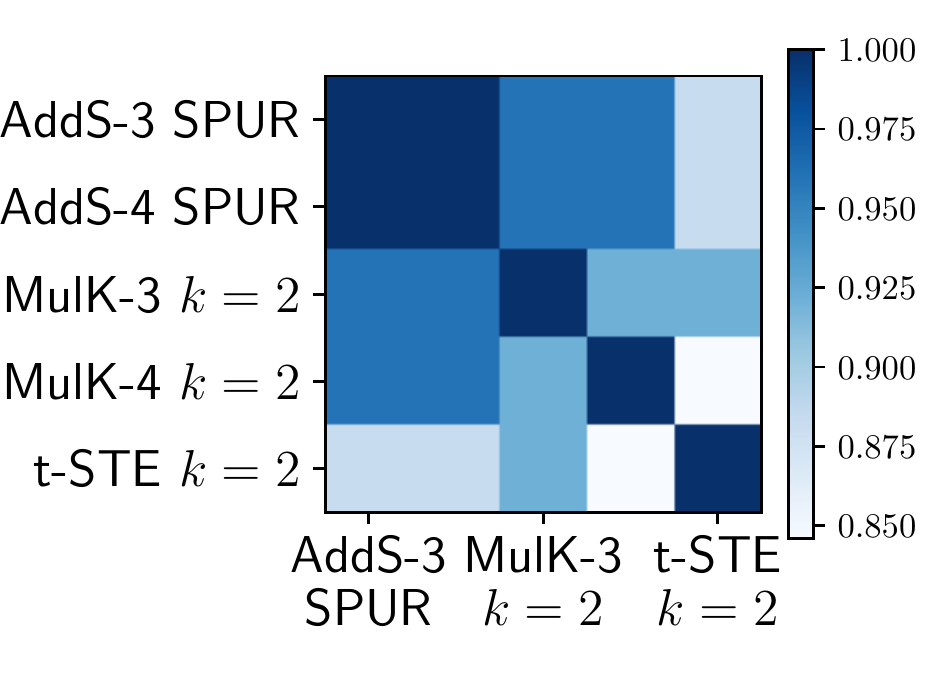}
    \caption{ARI between the clustering obtained by the different baselines. \ref{eqn_adis3} and \ref{eqn_adis4} with SPUR both estimate that the number of cluster is $k=2$. There is a high degree of agreement between the different approaches.}
    \label{app: fig: food comparison}
\end{figure}

\subsection{MNIST} \label{app: sec: mnist}

In this section, we consider additional experiments on the MNIST dataset. First, we consider a second similarity measure to generate the triplets. Then, we illustrate the partitions obtained with \ref{eqn_adis3} with known $k$ and SPUR respectively.

\textbf{Gaussian similarity.} In the main paper, we use the Gaussian similarity to generate the comparisons. More precisely, we compute the similarity between two examples $x_i$ and $x_j$ as
\begin{align*}
    w_{ij} = \exp\left(\frac{\normTwo{x_i - x_j}^2}{\gamma^2}\right) \text{ with } \gamma = 1.
\end{align*}

\textbf{Cosine similarity.}
Instead of the Gaussian similarity, we could consider alternatives to generate the comparisons. For example, the Cosine similarity:
\begin{align*}
    w_{ij} = \frac{\langle x_i,x_j\rangle}{\normTwo{x_i}\normTwo{x_j}}.
\end{align*}
In Figure~\ref{app: fig: mnist}, we show that using this alternative similarity affects the absolute results of the considered approaches. However, it does not change the overall trend, that is, as the number of comparisons increases, \ref{eqn_adis3} converges to the baseline of $k$-means with access to the original similarities.

\begin{figure}
     \centering
     \begin{subfigure}[b]{0.32\textwidth}
         \centering
        \includegraphics[height=0.65\textwidth,clip=true,trim=1mm 1mm 90mm 3mm]{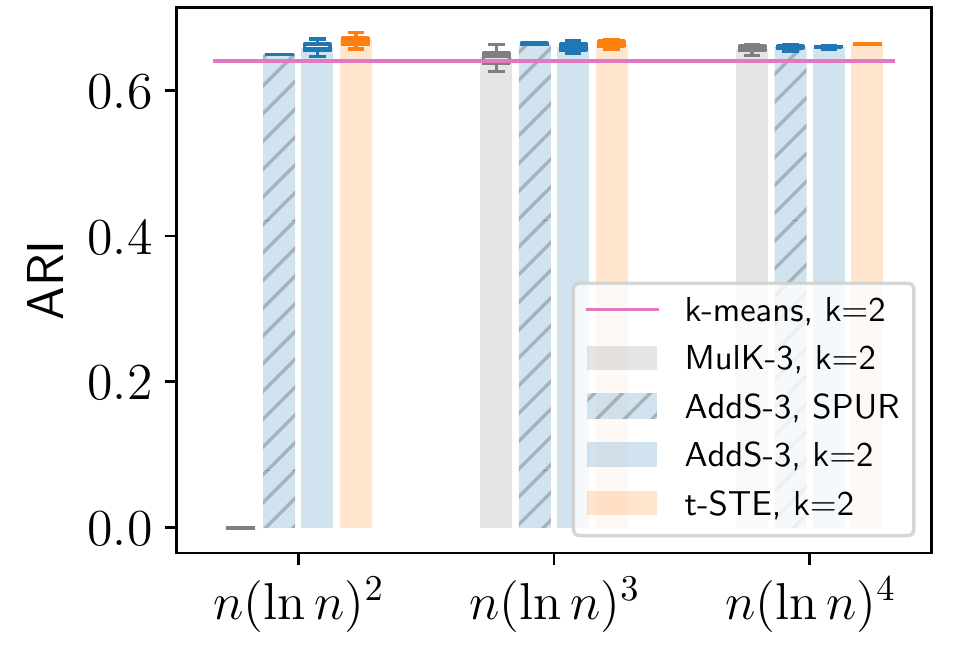}
         \caption{MNIST 1vs.7, $n=2163$}
         \label{app: fig: mnist 1v7}
     \end{subfigure}
     ~
     \begin{subfigure}[b]{0.32\textwidth}
         \centering
        \includegraphics[height=0.65\textwidth,clip=true,trim=97mm 1mm 1mm 3mm]{{figures_appendix/real_data_mnist}.pdf}
         \caption{MNIST 10, $n=2000$}
         \label{app: fig: mnist 10}
     \end{subfigure}
        \caption{Experiments on MNIST using the cosine similarity. The absolute ARI performances are different from the Gaussian similarity. However, the overall trend is preserved and, given sufficiently many comparisons, all the ordinal baselines reach the performance of $k$-means on the original data.
        }
        \label{app: fig: mnist}
\end{figure}

\textbf{Clustering using known $k$.}
Figure~\ref{app: fig: MNIST true labels} shows the t-SNE embedding of 2000 MNIST samples of all ten classes, where we see a clear separation between some classes (for example, 0 and 1) and very close embedding between others (for example, 1 and 9). Note that the classes obtained by \ref{eqn_adis3} are shown up to permutations and may not reflect the majority label in the different clusters. Further note that the data presented here corresponds to a single repetition out of the 10 repetitions used to compute the mean ARI (with standard deviation) in the main paper and this appendix.
In Figure~\ref{app: fig: MNIST AddS 2}, we see that, for $|\mathcal{T}|=n(\ln n)^2$, the learned partition is not very representative of the original labels.
Figure~\ref{app: fig: MNIST AddS 3} shows that, when the number of comparisons increases to  $|\mathcal{T}|=n(\ln n)^3$, the recovery ability of \ref{eqn_adis3} is greatly improved. However, the obtained partitions are not entirely satisfactory.  
Finally, Figure~\ref{app: fig: MNIST AddS 4} shows that, when the number of comparisons further increases to $|\mathcal{T}|=n(\ln n)^4$, the clustering obtained is close to the true labeling and most clusters are correctly identified.

\textbf{Clustering using SPUR.}
In this second set of experiments, we extend our observations from the previous paragraph to the labeling obtained by \ref{eqn_adis3} using SPUR. One can note that SPUR always underestimates the number of clusters.
Hence, in Figure~\ref{app: fig: MNIST SPUR AddS 3}, with $|\mathcal{T}|=n(\ln n)^3$, the number of predicted clusters is $k=6$ while, in Figure~\ref{app: fig: MNIST SPUR AddS 4}, with $|\mathcal{T}|=n(\ln n)^4$, the number of predicted clusters is $k=8$.
This explain the slightly worse behaviour of SPUR compared to known $k$ in Figure~\ref{fig: mnist 10} in the main paper. Nevertheless, the difference in average ARI is not so significant when $|\mathcal{T}|=n(\ln n)^4$, suggesting that $8$ clusters is, in fact, a good estimate of the number of clusters that can reliably be distinguished by the different methods.

\begin{figure}
    \centering
\begin{minipage}{\textwidth}
\centering
\centering
\begin{tabular}{*{1}{m{\textwidth}}}
\toprule\\
\begin{center}Class A with 36 images (\textit{Sweet?})\end{center}\\
\begin{center}\includegraphics[width=\textwidth]{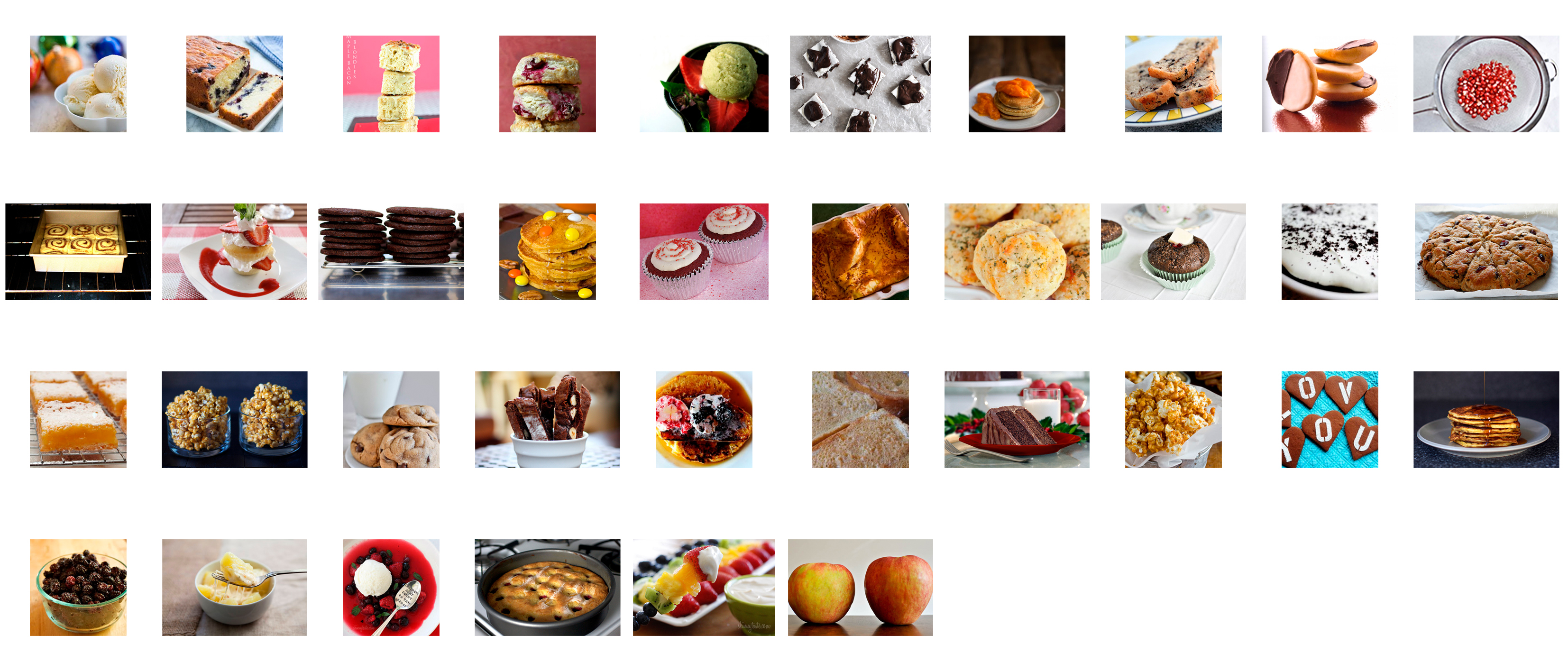}\end{center}\\
\midrule
\\
\begin{center}Class B with 64 images (\textit{Savoury?})\end{center}\\
\begin{center}\includegraphics[width=\textwidth]{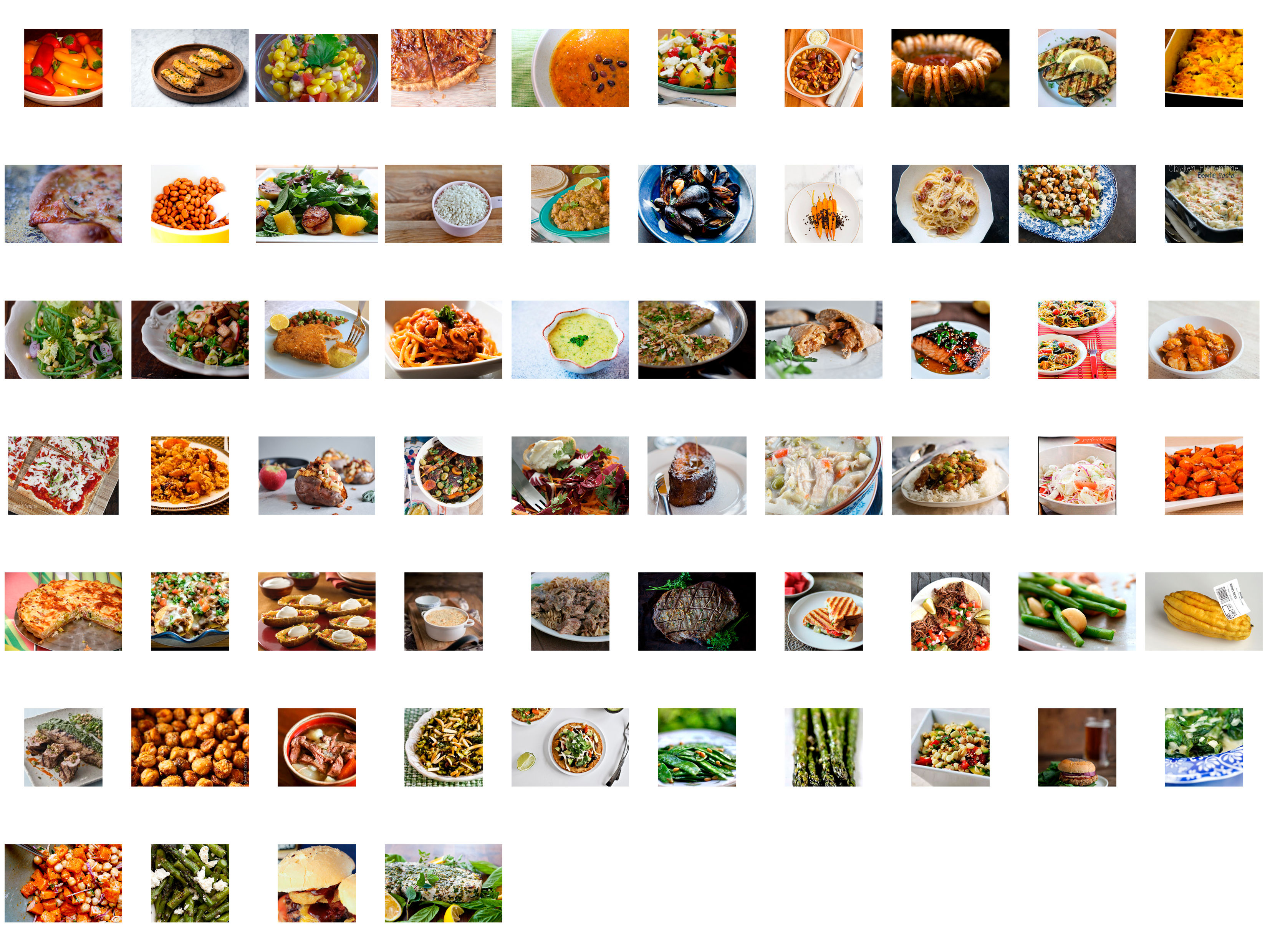}\end{center}\\
\hline
\end{tabular}
\end{minipage}
\caption{Clusters obtained by \ref{eqn_adis3} on the food dataset. It seems that the Sweet foods are separated from the Savoury ones.}
    \label{app: tab: food samples}
\end{figure}

\begin{figure*}[ht]
     \centering
     \begin{subfigure}[b]{0.49\textwidth}
         \centering
        \includegraphics[width=\textwidth]{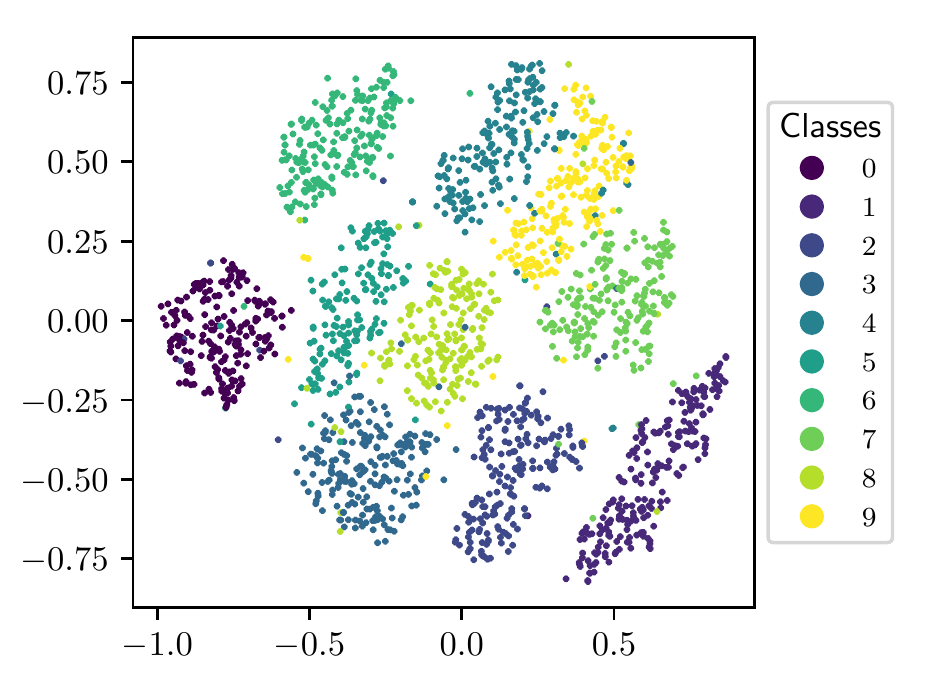}
         \caption{MNIST embedding with true labels}
         \label{app: fig: MNIST true labels}
     \end{subfigure}
     \hfill
     \begin{subfigure}[b]{0.49\textwidth}
         \centering
        \includegraphics[width=\textwidth]{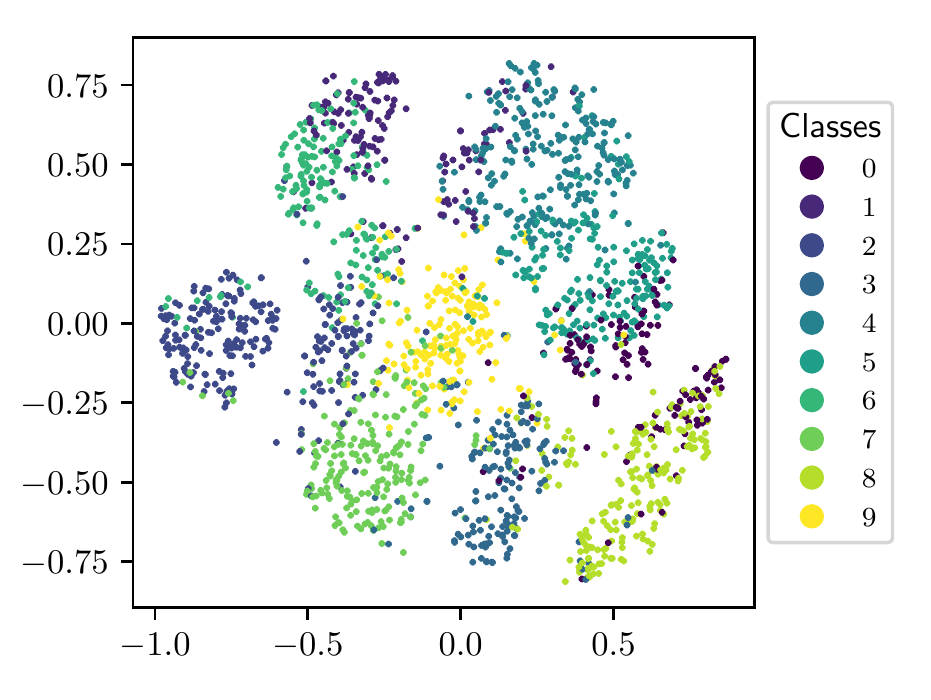}
         \caption{\ref{eqn_adis3} $k=10$, $|\mathcal{T}|=n(\ln n)^2$}
        \label{app: fig: MNIST AddS 4}
     \end{subfigure}
     
     \begin{subfigure}[b]{0.49\textwidth}
         \centering
        \includegraphics[width=\textwidth]{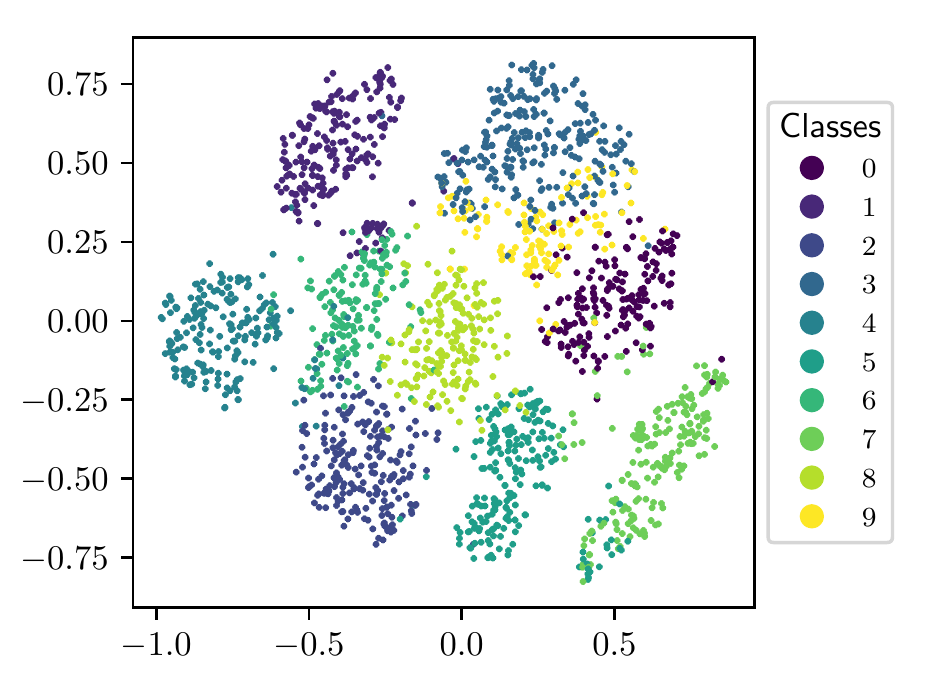}
         \caption{\ref{eqn_adis3} $k=10$, $|\mathcal{T}|=n(\ln n)^3$}
        \label{app: fig: MNIST AddS 3}
     \end{subfigure}
     \hfill
     \begin{subfigure}[b]{0.49\textwidth}
         \centering
        \includegraphics[width=\textwidth]{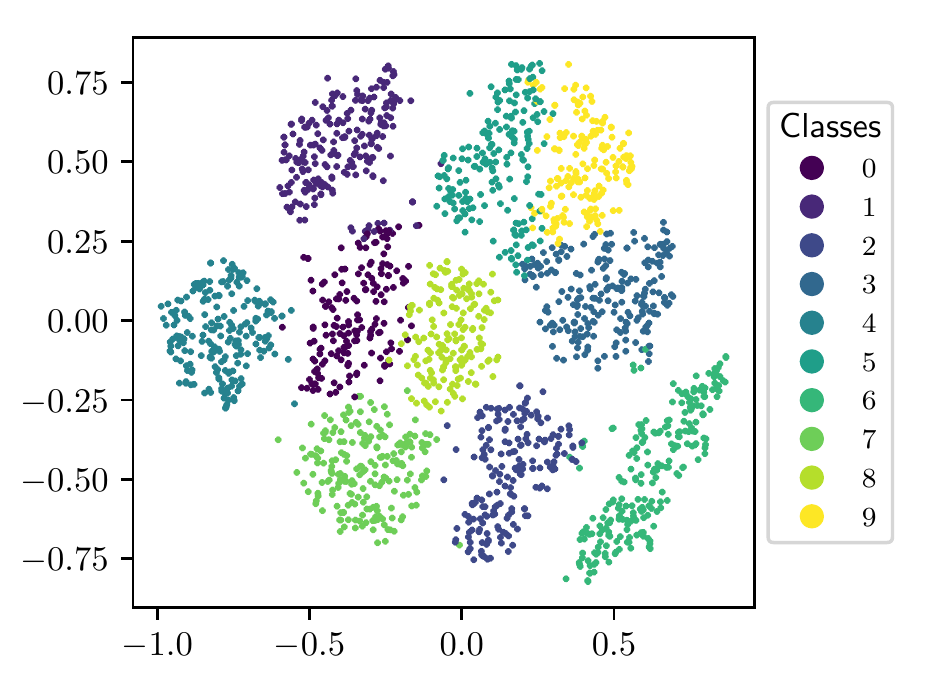}
         \caption{\ref{eqn_adis3} $k=10$, $|\mathcal{T}|=n(\ln n)^4$}
        \label{app: fig: MNIST AddS 2}
     \end{subfigure}
        \caption{t-SNE embedding of 2000 MNIST samples with \eqref{app: fig: MNIST true labels} true labeling and \eqref{app: fig: MNIST AddS 2}--\eqref{app: fig: MNIST AddS 4} clusters obtained by \ref{eqn_adis3} with known $k=10$ and varying number of observations. The classes are given up to permutations and may not reflect the majority label in each cluster.}
        \label{app: fig: scatter MNIST}
\end{figure*}

\begin{figure*}[ht]
     \centering
     \begin{subfigure}[b]{0.49\textwidth}
         \centering
        \includegraphics[width=\textwidth]{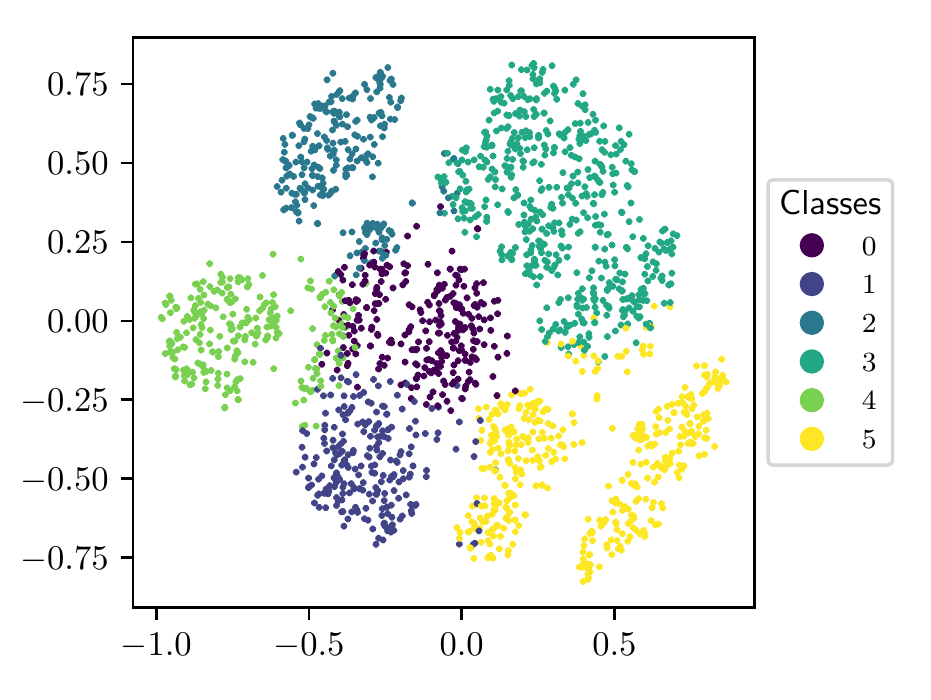}
         \caption{\ref{eqn_adis3} SPUR, $|\mathcal{T}|=n(\ln n)^3$}
         \label{app: fig: MNIST SPUR AddS 3}
     \end{subfigure}
     \hfill
     \begin{subfigure}[b]{0.49\textwidth}
         \centering
        \includegraphics[width=\textwidth]{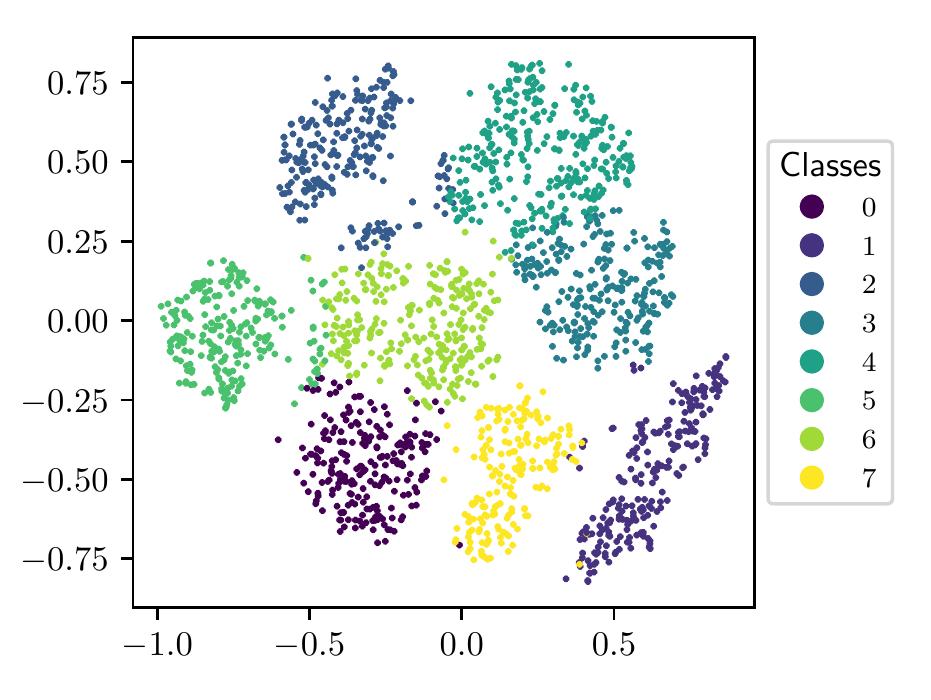}
         \caption{\ref{eqn_adis3} SPUR, $|\mathcal{T}|=n(\ln n)^4$}
        \label{app: fig: MNIST SPUR AddS 4}
     \end{subfigure}
     
        \caption{t-SNE embedding of 2000 MNIST samples with the clusters predicted by \ref{eqn_adis3} using SPUR and varying number of comparisons. The classes are given up to permutations and may not reflect the majority label in each cluster.}
        \label{app: fig: scatter MNIST SPUR}
\end{figure*}

\end{document}